%% file: main.tex
\documentclass[twoside,11pt]{article}

\usepackage[preprint]{jmlr2e}

\input{preamble}
\input{macros}

\usepackage{lastpage}
\jmlrheading{0}{2026}{1-\pageref{LastPage}}{7/26}{}{}{Vishal Rajput}

\ShortHeadings{Matching Principle: Covering Deployment Shift}{Rajput}
\firstpageno{1}

\begin{document}

\title{The Matching Principle:\\
When Does a Training Penalty Cover Deployment Shift?}

\author{\name Vishal Rajput \email vishal.stark42@gmail.com \\
       \addr Department of Computer Science \\
       KU Leuven \\
       Celestijnenlaan 200A, 3001 Leuven, Belgium}

\editor{}

\maketitle

\begin{abstract}%
\input{sections/00_abstract_jmlr}
\end{abstract}

\begin{keywords}
  domain adaptation, robustness, representation learning, Jacobian regularization,
  label-preserving nuisance
\end{keywords}

\input{sections/01_introduction}
\input{sections/05b_related_work}
\input{sections/02_setup}
\input{sections/04_core_matching}
\input{sections/05_estimators}
\input{sections/06_recipe}
\input{sections/08_empirical}
\input{sections/09_discussion}
\input{sections/10_conclusion}

\acks{The author thanks colleagues at KU Leuven for feedback on early drafts of this framework.
\textbf{Disclosure of funding:} No third-party funding supported the writing of this manuscript.
\textbf{Competing interests:} None declared.
This preprint is a single-PDF merge of the JMLR manuscript and its Online Appendix~1
(protocols and supporting proofs). A journal version is under submission at JMLR.
Companion note: arXiv:2604.21395.}

\appendix
\input{appendix/proofs_main}
\input{appendix/conditional_licenses}

\input{appendix/task_supplement}

\input{appendix/proofs}

\bibliography{bibliography}

\end{document}

%% file: preamble.tex
\def\ManuscriptRoot{.}%
\usepackage[utf8]{inputenc}
\usepackage{cmap}
\usepackage{times}
\usepackage{amsmath,amssymb,amsfonts}
\usepackage{booktabs}
\usepackage{array}
\usepackage{tabularx}
\usepackage{graphicx}
\graphicspath{{./}{figures/main/}{figures/tasks/}{figures/tasks/real/}}
\usepackage{placeins}
\usepackage{float}
\usepackage[protrusion=true,expansion=true]{microtype}
\usepackage{enumitem}
\usepackage{xcolor}
\usepackage{subcaption}
\usepackage{multirow}
\usepackage{pgfplots}
\pgfplotsset{compat=1.18}
\usepackage{tikz}
\usetikzlibrary{arrows.meta,positioning,fit,backgrounds,decorations.pathreplacing,calc}
\usepackage{tcolorbox}
\tcbuselibrary{breakable,skins}
\usepackage{listings}
\lstdefinestyle{pythonsketch}{
  language=Python,
  basicstyle=\ttfamily\small,
  keywordstyle=\bfseries\color{blue!60!black},
  stringstyle=\color{red!55!black},
  commentstyle=\itshape\color{black!55},
  showstringspaces=false,
  columns=fullflexible,
  keepspaces=true,
  breaklines=true,
  xleftmargin=0.4em,
  aboveskip=2pt, belowskip=2pt,
}
\input{pmh_style}

\tcbset{
  recipebox/.style={
    colback=blue!3, colframe=blue!55!black, boxrule=0.6pt,
    arc=2pt, left=8pt, right=8pt, top=6pt, bottom=6pt,
    title={#1}, fonttitle=\bfseries,
    coltitle=white, colbacktitle=blue!55!black,
    sharp corners=northeast, sharp corners=northwest,
  },
  keytake/.style={
    colback=orange!4, colframe=orange!60!black, boxrule=0.5pt,
    arc=2pt, left=8pt, right=8pt, top=5pt, bottom=5pt,
    fonttitle=\bfseries,
  },
  codebox/.style={
    colback=blue!2, colframe=blue!55!black, boxrule=0.5pt, arc=2pt,
    left=6pt, right=6pt, top=4pt, bottom=4pt,
    fonttitle=\bfseries, coltitle=white, colbacktitle=blue!55!black,
  },
}

\newcommand{\SubmissionFig}[2][width=\linewidth]{%
  \centering
  \includegraphics[#1,keepaspectratio,interpolate=false]{#2}%
}
\newcommand{\SubmissionFigCompact}[2][width=\linewidth]{%
  \centering
  \includegraphics[#1,keepaspectratio,interpolate=false]{#2}%
}

\usepackage{xurl}
\hypersetup{hidelinks}

%% file: pmh_style.tex
\definecolor{pmhMatched}{RGB}{46,125,50}
\definecolor{pmhIso}{RGB}{25,118,210}
\definecolor{pmhWrong}{RGB}{198,40,40}
\definecolor{pmhSignal}{RGB}{245,124,0}
\definecolor{pmhGray}{RGB}{120,120,120}
\definecolor{pmhSpectrum}{RGB}{155,182,198}
\definecolor{pmhTheoryOpt}{RGB}{122,62,182}
\definecolor{pmhSignalBand}{RGB}{255,240,230}

\pgfplotsset{
  pmhbar/.style={
    ybar,
    bar width=8pt,
    enlarge x limits=0.15,
    ymin=0,
    legend style={font=\small,at={(0.5,1.05)},anchor=south,legend columns=-1},
    tick label style={font=\small},
    label style={font=\small},
  }
}

%% file: macros.tex
\newcommand{\Sigmatask}{\Sigma_{\mathrm{task}}}

\newcommand{\PMH}{\textnormal{\textsc{pmh}}}
\newcommand{\MatchingPmhRepoUrl}{https://github.com/vishalstark512/matching-pmh}
\newcommand{\MatchingPmhRepo}{\url{\MatchingPmhRepoUrl}}
\newcommand{\MatchingPmhPackage}{The \texttt{matching-pmh} package}
\newcommand{\MatchingPmhInstall}{; \texttt{pip install matching-pmh}}
\newcommand{\JsonPath}[1]{\path{#1}}

\newcommand{\TDI}{\textnormal{\textsc{tdi}}}

\newcommand{\rangeop}{\mathrm{range}}
\newcommand{\Tr}{\mathrm{Tr}}
\newcommand{\E}{\mathbb{E}}
\newcommand{\Cov}{\mathrm{Cov}}

\newcommand{\dd}{\,\mathrm{d}}

\newtheorem{assumption}[theorem]{Assumption}

%% file: sections/00_abstract_jmlr.tex
Ordinary training optimises the task loss and then stops.
It never pays for internal representation energy: Jacobians can stay large in
directions that never helped the label, so even small label-preserving noise
throws the model off---a design gap that classical noise-injection theory fixes
at second order, but only when applied as default regularisation, which current
practice does not do.
We make that precise with a Matching Principle: name deployment directions
($\Sigmatask$) and the training penalty $\Sigma'$, and ask whether the second
covers the first.
The no-thinking default is even-spread / isotropic penalty ($\Sigma'\propto I$)---classical
Gaussian / Tikhonov at second order: no axis estimate, no architecture change,
and---in a simple linear ridge model---strictly less deployment drift than
task-only training, with no coverage miss by construction.
When axes are known, matching is sharper; when they are missed, a residual floor
remains.
Across seven domains a named second-moment penalty beats unregularised training;
a controlled illustration recovers match $\succ$ even-spread $\succ$ wrong-axis
when axes are forced.
The ridge theorems are proved; deep nets remain experiments under a specified
perturbation.
Design rule: fix internal energy by default (even-spread); match when axes are
known; treat losses that control representation sensitivity as first-class design.

%% file: sections/01_introduction.tex
\section{Introduction}
\label{sec:intro}

Ordinary supervised training cares about one thing: the task loss.
Once that loss is small, training is done.
Nothing in that objective pays for \emph{internal} representation energy---how far the
embedding moves when the input is nudged---so Jacobians often stay large in many directions
and mild label-preserving noise can move the head with them.
Models look ``broken'' under small shifts not because deep nets cannot be robust, but because
the loss never asked for it.
Robustness methods then bolt on an attack, an augmentation, or a domain moment along
\emph{some} directions, without asking whether those are the ones that will move at
deployment.

This paper names that gap---classical noise injection already fixes it at second order
\citep{bishop1995training}, but practice does not treat that fix as default
regularisation---and
starts from the easiest application of it.
Bishop / Tikhonov give the second-order expansion; what we treat as first-class is the
\emph{cover} question, the default when axes are unknown, and the residual floor when they
are missed.
Name deployment motion ($\Sigmatask$) and the training penalty ($\Sigma'$); ask whether the
second \emph{covers} the first (the Matching Principle).
Even-spread ($\Sigma'\propto I$)---isotropic training penalty; classical Gaussian / Tikhonov
at second order---needs no
axis estimate and no architecture change: in the linear--quadratic response it strictly
reduces deployment drift versus task-only training and drives the floor to zero for every
nonzero nuisance (Corollary~\ref{cor:even-spread-vs-erm}).
Matching known axes comes after; the larger invitation is to treat losses and architectures
that control representation energy as first-class design.

\paragraph{A picture.}
\label{sec:toy-2d}
Imagine points in the plane whose labels depend only on the vertical coordinate
(Figure~\ref{fig:pipeline}).
At deployment, noise is horizontal and the label stays fixed.
Task-only training never pays for horizontal sensitivity, so that axis stays live.
Even-spread shrinks every axis---including the horizontal one---and the residual floor goes
away.
Matched training spends the same budget on the true axis and shrinks it faster.
Wrong-axis training shrinks the vertical direction instead: no larger weight removes the
horizontal floor.
The same cartoon applies if the ``horizontal'' axis is a frequency band, a paired domain
shift, or an adversarial subspace found by an attack algorithm.

\begin{figure}[!ht]
\centering
\resizebox{0.96\linewidth}{!}{\input{figures/main/fig01_pipeline}}
\caption{\textbf{Task-only training leaves the nuisance live; even-spread / isotropic is the default fix.}
Same horizontal nuisance: no-penalty task-only training (left) stays sensitive along it; even-spread /
isotropic (centre) shrinks every axis including the nuisance; a wrong-axis penalty (right) shrinks an
orthogonal direction and leaves a floor.
Matched training (not drawn) spends the same budget on the true axis and is sharper when
those axes are known.}
\label{fig:pipeline}
\end{figure}

\paragraph{The object.}
Under \emph{label-preserving} shift we summarise the nuisance by
$\Sigmatask=\E[\Delta\Delta^\top]$, the average outer product of label-preserving input
changes $\Delta$.
Large eigenvalues of $\Sigmatask$ are directions where deployment moves a lot.
Domain adaptation, augmentation, virtual adversarial training, and adversarial training
all penalise related second-moment geometry
\citep{sun2016coral,madry2018towards,miyato2018virtual}
(Table~\ref{tab:methods-as-estimators})---typically along one law at a time.
The design question is whether that law covers $\rangeop(\Sigmatask)$
\citep{hendrycks2019benchmarking,tsipras2019robustness,koh2021wilds}; when it is unknown,
even-spread is the cover that does not require an answer.

\paragraph{Contributions.}
\begin{itemize}[leftmargin=1.4em,itemsep=2pt plus 0pt minus 0pt,parsep=0pt,topsep=2pt plus 0pt minus 0pt]
  \item \textbf{Internal energy / even-spread.}
    Task-only training never penalises representation energy; even-spread / isotropic does
    (Corollary~\ref{cor:even-spread-vs-erm}). Across seven domains \emph{some} named penalty
    beats none (\S\ref{sec:emp-tier1})---even-spread is the no-axis case, not every row.
  \item \textbf{Cover-or-floor.}
    Missing any deployment direction leaves a permanent floor as $\lambda\to\infty$;
    full-support penalties differ from matched only by a constant (\S\ref{sec:matching}).
  \item \textbf{When axes are known.}
    Oracle hard projection recovers match $\succ$ even-spread $\succ$ wrong-axis as a
    controlled illustration (\S\ref{sec:emp-tier2})---secondary to the default-regularisation
    claim.
  \item \textbf{Checklist.}
    Name the shift, the penalty, and the eigengap; report geometry beside the task score;
    default to even-spread when axes are unknown (\S\ref{sec:recipe}).
\end{itemize}
Coverage theorems are proved in a ridge/linear setting as $\lambda\to\infty$.
Deep-net blocks are experiments under a specified perturbation, not a completed nonlinear
proof
(item~(O) in Table~\ref{tab:open} remains open).
Even-spread does not wait on that proof; estimating real axes remains largely open
(\S\ref{sec:estimation-gap}).

\begin{tcolorbox}[keytake,title={Vocabulary used below},breakable]
\label{box:glossary}
\begin{description}[leftmargin=1.15em,style=nextline,itemsep=1pt plus 0pt minus 0pt,parsep=0pt,topsep=0pt,font=\small]
  \item[Task-only training]
    Optimise the labelled task loss alone---no second-moment / Jacobian regulariser on
    representation energy (the unregularised baseline B0 in experiments).
  \item[$\Sigmatask$ / $\Sigma'$ / cover]
    Deployment second moment vs.\ training penalty; cover means
    $\rangeop(\Sigma')\supseteq\rangeop(\Sigmatask)$.
  \item[Even-spread / isotropic / matched / wrong]
    \emph{Even-spread} / \emph{isotropic} $=$ training penalty $\Sigma'\propto cI$ (default
    when axes are unknown); \emph{matched} penalises the true axes; \emph{wrong} penalises an
    equal-dimension miss.
    (If deployment itself is isotropic, $\Sigmatask\propto I$, matched and even-spread
    coincide.)
  \item[Hard projection / law status]
    Oracle freeze of chosen axes (coverage test, not the trainer);
    \emph{identified} / \emph{algorithmic} / \emph{proxy-only} naming of $\Sigmatask$.
  \item[\PMH{}]
    Task loss plus $\lambda$ times Jacobian energy along $\Sigma'$ (\S\ref{sec:setup}).
\end{description}
\end{tcolorbox}

\paragraph{Roadmap.}
\S\ref{sec:related-work} situates the lineage;
\S\ref{sec:setup} defines the objects;
\S\ref{sec:matching} proves even-spread-vs-task-only and cover-or-floor;
\S\ref{sec:estimators} maps familiar methods;
\S\ref{sec:recipe} gives the checklist;
\S\ref{sec:empirical} reports default-regularisation breadth first, a short controlled
illustration, then mixed
field comparisons (not a single-metric contest).
Conditional licenses: Appendix~\ref{app:conditional-licenses}.

%% file: figures/main/fig01_pipeline.tex
\begin{tikzpicture}[
  every node/.style={font=\small},
  axisstyle/.style={-{Stealth[length=1.5mm]}, pmhGray, thin},
  perturb/.style={-{Stealth[length=1.8mm]}, pmhWrong, thick, line width=0.9pt},
  shadow/.style={dashed, pmhGray!80, thin},
  panelbox/.style={draw=pmhGray!35, rounded corners=1.5pt, line width=0.45pt}
]

  \foreach \xoff/\ea/\eb/\rot/\px/\py/\title/\col in {
    0/0.88/0.88/0/2.15/0/{Task-only (no penalty)}/pmhGray,
    5.1/0.42/0.42/0/1.55/0/{Even-spread / iso}/pmhIso,
    10.2/0.98/0.72/42/1.65/0/{Wrong-axis}/pmhWrong%
  }{
    \begin{scope}[xshift=\xoff cm]
      \node[panelbox, minimum width=4.6cm, minimum height=3.2cm, anchor=south west] at (-2.15,-1.75) {};
      \draw[axisstyle] (-1.7,0) -- (2.0,0) node[right, font=\scriptsize, pmhGray]{nuisance};
      \draw[axisstyle] (0,-1.35) -- (0,1.35) node[above, font=\scriptsize, pmhGray]{signal};
      \begin{scope}[rotate=\rot]
        \draw[draw=\col, fill=\col, thick, fill opacity=0.14, line width=0.9pt] (0,0) ellipse ({\ea} and {\eb});
      \end{scope}
      \filldraw[black] (0,0) circle (1.3pt);
      \draw[perturb] (0.05,0.12) -- (1.95,0.12);
      \filldraw[pmhWrong] (\px,\py) circle (1.3pt);
      \draw[shadow] (0,0) -- (\px,\py);
      \node[anchor=north, font=\footnotesize\bfseries] at (0.15,-1.95) {{\color{\col}\title}};
    \end{scope}
  }

  \node[font=\scriptsize, pmhGray, anchor=north] at (5.1,-2.35)
    {Same deployment perturbation; Jacobian sensitivity ellipses show where $\|Jv\|$ is large.};

\end{tikzpicture}

%% file: sections/05b_related_work.tex
\section{Related work}
\label{sec:related-work}
\label{sec:prior-unifications}
\label{sec:prior-chargers}
\label{sec:prior-methods}
\label{sec:classical-chargers}

Robustness and adaptation already attack the same geometric idea: match feature moments
\citep{sun2016coral,tzeng2014ddc}, confuse a domain discriminator~\citep{ganin2015domain,ben2010theory},
or shrink local sensitivity by augmentation, VAT, or adversarial training
\citep{hendrycks2019benchmarking,miyato2018virtual,madry2018towards,koh2021wilds}.
Surveys catalogue the recipes~\citep{csurka2017domain,wilson2020survey}; they rarely ask
whether the penalised object covers the directions that move at deployment
\citep{tsipras2019robustness}.
Table~\ref{tab:methods-as-estimators} is a diagnostic map of that question.

\paragraph{Classical second-order lineage.}
That these recipes are related at second order is not new---and the easiest penalty is already
classical: additive Gaussian noise is Tikhonov / ridge
\citep{bishop1995training,an1996effects,hoerl1970ridge}, i.e.\ even-spread / isotropic
penalty at second order,
with no need to name deployment axes.
Wager et al.\ give the dropout form~\citep{wager2013dropout}.
TangentProp~\citep{simard1991tangent} already penalises sensitivity along chosen
label-preserving tangents.
Later unifications connect adversarial gradient penalties and data augmentation to the same
second-order picture
\citep{lyu2015unified,simongabriel2019first,finlay2020scaleable,dao2019kernel}.
Dao et al.\ link that expansion to invariant kernels, tangent propagation, and robust
optimisation.
Across noise, dropout, known invariances, augmentation, and adversarial training, the field
already regularises related second-order objects---typically along one chosen law.

\paragraph{What remains first-class here.}
Prior work establishes the second-order expansion and regularising a known law; we do not
re-derive Tikhonov.
What we treat as first-class is the cover question, the no-axis default
(even-spread / isotropic; Corollary~\ref{cor:even-spread-vs-erm}), and cover-or-floor when
$\Sigma'\neq\Sigmatask$ (Theorem~\ref{thm:game}, Proposition~\ref{prop:local-obstruction},
Corollary~\ref{cor:nfl-noise}).
Default-regularisation breadth comes first; controlled illustrations and field head-to-heads
are secondary
(\S\ref{sec:empirical}).
Label / spurious change~\citep{arjovsky2019irm,sagawa2020groupdro} and preference / style
objectives~\citep{rafailov2023dpo} are out of scope.

%% file: sections/02_setup.tex
\section{Setup and definitions}
\label{sec:setup}

We study shifts where the \emph{label stays fixed} while the \emph{input moves}---a stain
change that should not alter a pathology call, or sensor noise that should not alter an
activity class.
Two matrices matter.
$\Sigmatask$ summarises the directions in which deployment actually moves the input.
$\Sigma'$ is the matrix training uses when it penalises the encoder's sensitivity
(the penalty matrix from \S\ref{sec:intro}).
The rest of this section names the predictor, defines those two matrices, and writes the
loss that connects them.

\paragraph{Predictor.}
An input $x\in\mathcal{X}=\mathbb{R}^{d_x}$ is mapped by an encoder
$\phi_\theta:\mathcal{X}\to\mathcal{Z}=\mathbb{R}^{d_\phi}$ (differentiable almost everywhere)
to a representation, then by a head $h_\theta:\mathcal{Z}\to\mathcal{Y}'$ to a prediction.
Write $f_\theta=h_\theta\circ\phi_\theta$ and let $\mathcal{L}_{\mathrm{task}}$ be the usual
task loss.
For the theory we take $h_\theta$ to be $L$-Lipschitz.%
\footnote{Rescaling $\phi\mapsto c\phi$ and $h\mapsto h/c$ leaves $f$ unchanged while scaling
Jacobian energy by $c^2$; we pin that gauge in Proposition~\ref{prop:gauge}.}

\paragraph{Deployment object.}
In the Figure~\ref{fig:pipeline} cartoon, deployment moves horizontally while the label
stays fixed: that horizontal axis is what $\Sigmatask$ will record.
At deployment the input can change while the label does not.
Formally, a coupling $(X,Y,X^+,Y^+)$ is \emph{structurally label-preserving} when
$Y^+=Y$ almost surely
(Appendix~\ref{app:proofs}, Definition~\ref{def:label-preserving}).
Write $\Delta:=X^+-X$ for that input change.
In the additive, input-independent specialisation used by the matrix theory,
$\Delta\perp X$ and
\begin{equation}
\Sigmatask := \E[\Delta\Delta^\top]
\;\in\;\mathbb{S}^{d_x}_{\geq 0}
\end{equation}
is finite (equals $\Cov(\Delta)$ when $\E[\Delta]=0$).
Large eigenvalues of $\Sigmatask$ are directions where deployment moves a lot; the
\emph{range} of $\Sigmatask$ is the span of those directions---the set that training must
cover.
Shifts that change the label with positive probability are out of scope.

\paragraph{Four roles (do not collapse them).}
This paper does not give a general recipe for discovering $\Sigmatask$ from data.
Beside each experiment we only report law status (glossary): \emph{identified},
\emph{algorithmic}, or \emph{proxy-only}.
Keep four objects distinct:
\begin{itemize}[leftmargin=1.4em,itemsep=1pt plus 0pt minus 0pt,parsep=0pt,topsep=2pt plus 0pt minus 0pt]
  \item $\Sigmatask$ --- where deployment truly moves;
  \item $\hat\Sigmatask$ --- whatever estimate or proxy an experiment uses;
  \item $\Sigma'$ --- what training actually penalises;
  \item $D_Q$ --- how much the representation moves under a named nudge (next paragraph).
\end{itemize}
Training is \emph{matched} only when the stated coupling identifies $\Sigmatask$ and
$\Sigma'=\hat\Sigmatask$.

\paragraph{Drift.}
Drift asks a simple question: if we nudge the input the way deployment does, how far does
the representation move?
\begin{equation}
D_Q(\phi) = \E_{x\sim P_X,\,n\sim Q_n}\!\left[\|\phi(x+n)-\phi(x)\|_2^2\right].
\end{equation}
For theory we use the first-order version
\begin{equation}
\tilde D_Q(\phi)
= \E_x\!\left[\Tr\!\left(J_\phi(x)^\top J_\phi(x)\,\Sigmatask\right)\right],
\label{eq:drift-linear}
\end{equation}
with remainder controlled by Lemma~\ref{lem:lindrift}.
The trace is Jacobian energy reweighted by $\Sigmatask$: directions with large deployment
mass cost more.
Experiments report $D_Q$ where stated; theorems use $\tilde D_Q$.

\paragraph{Why drift matters.}
If the representation moves under a label-preserving nudge, any Lipschitz head on top of
it can feel that move.
That is the reason we treat embedding drift as the object to control.

\begin{theorem}[Consumer-class reading of embedding drift]
\label{thm:consumer-class}
Fix an $L$-Lipschitz class of downstream maps $h:\mathcal{Z}\to\mathbb{R}$.
For any encoder $\phi$ and displacement law $Q_n$,
\[
\sup_{\mathrm{Lip}(h)\le L}
\E\bigl|h(\phi(x+n))-h(\phi(x))\bigr|^2
\;\le\;
L^2\,D_Q(\phi).
\]
\end{theorem}
\textit{Meaning.}
If the representation drifts, every Lipschitz head can feel it---so controlling $D_Q$ is
enough to control that consumer class.
(A matching lower bound of order $\lambda_{\max}(\E[J_\phi\Sigmatask J_\phi^\top])$ holds in
the linearised regime; see Appendix~\ref{app:proofs}.)

\paragraph{The matched Jacobian penalty.}
\label{par:pmh-loss}
We train with the task loss plus a penalty on Jacobian energy along a chosen matrix
$\Sigma'$---call this the matched Jacobian penalty when $\Sigma'=\Sigmatask$
(abbreviated \PMH{} in equations and tables):
\begin{equation}
\mathcal{L}_{\PMH(\Sigma')}(\theta)
= \mathcal{L}_{\mathrm{task}}(\theta)
+ \lambda\,\E_x\!\left[\Tr\!\left(J_\phi(x)^\top J_\phi(x)\,\Sigma'\right)\right],
\label{eq:pmh-family}
\end{equation}
for $\Sigma'\succeq 0$ and $\lambda>0$.
How the trace is estimated in code is deferred to \S\ref{sec:recipe}.
Coverage necessity is proved under a ridge/linear bridge
(Assumption~\ref{ass:bridge}, Appendix~\ref{app:conditional-licenses}); nonlinear
experiments are under a specified perturbation (named law status).
Proofs: Appendix~\ref{app:proofs}; estimator catalogue: Online Appendix~1.

\subsection{What ``nuisance law'' means in practice}

In each experiment we must say \emph{how the input is allowed to move} while the label
stays fixed---sensor noise, a paired domain change, an augmentation family, an adversarial
subspace, and so on.
That named displacement is the nuisance law: it fixes the nudge distribution $Q_n$ and
therefore the matrix $\Sigmatask=\E[\Delta\Delta^\top]$.
Online Appendix~1 spells out seven concrete families ($A_1$--$A_7$) under stated couplings.
In code we rarely know $\Sigmatask$ exactly, so we plug an estimate
$\hat\Sigma_{\mathrm{task}}^{(k)}$ into~\eqref{eq:pmh-family} as the penalty
$\Sigma'$---the matched template used across blocks.
If that estimate is wrong, or the spectrum has no gap so the top directions are not
identifiable (Office-31), calling the run ``matched'' does not guarantee a win: the failure
is in identifying $\Sigmatask$, not in the coverage claim itself.

%% file: sections/04_core_matching.tex
\section{Core matching principle}
\label{sec:matching}

This section turns the intro cartoon into claims in the design-rule order: matched penalty
tracks local drift; even-spread beats task-only training in the linear--quadratic response; a miss leaves
a floor; a healthy fixed probe does not prove coverage.
Proofs: Appendix~\ref{app:proofs}.
Licenses and the training-weight cap: Appendix~\ref{app:conditional-licenses}.

\subsection{Matched penalty $=$ local deployment drift}
\label{sec:deep-scope}

\textit{Intuition.}
Nudge the input by a small $\Delta$ (horizontal, in Figure~\ref{fig:pipeline}).
To leading order the representation moves by $J_\phi\Delta$.
Averaging the squared move recovers a weighted Jacobian energy---exactly the matched
Jacobian penalty~\eqref{eq:pmh-family} with $\Sigma'=\Sigmatask$.

\begin{theorem}[Deployment consistency of the matched Jacobian penalty]
\label{thm:deployment-consistency}
Let $\phi$ have $\beta$-Lipschitz Jacobian with $\sup_x\|J_\phi(x)\|_F\le M$, and let
$\Delta$ be a zero-mean deployment displacement with
$\Sigmatask=\E[\Delta\Delta^\top]$ and scale parameter $\sigma$ (e.g.\
$\Delta=\sigma Z$ with $\E\|Z\|_2=O(1)$ and $\E\|Z\|_2^3<\infty$), under the additive
specialisation of Lemma~\ref{lem:lindrift}.
Then
\[
\E\bigl\|\phi(x+\Delta)-\phi(x)\bigr\|_2^2
=\E_x\bigl[\Tr\!\bigl(J_\phi(x)^\top J_\phi(x)\,\Sigmatask\bigr)\bigr]
+O(\sigma^3),
\]
with the remainder controlled by that lemma ($O(\sigma^4)$ under its stronger symmetry
hypotheses).
Consequently, population matched~\eqref{eq:pmh-family} with
$\Sigma'=\Sigmatask$ is the second-order population shadow of \emph{deployment embedding
drift} under the true coupling~$Q$---not a proved expansion of task-only training risk
\citep{bishop1995training,an1996effects}.
\end{theorem}
\textit{Meaning.}
Penalising $\Tr(J^\top J\,\Sigmatask)$ is locally charging the true label-preserving
nudge---not an arbitrary Jacobian recipe (\S\ref{sec:related-work}).

\subsection{Cover the range, or keep a floor}
\label{sec:thmA}
\label{sec:thmG}

\textit{Intuition.}
Back to Figure~\ref{fig:pipeline}: if training never penalises the horizontal axis, no
amount of turning up the weight removes horizontal drift.
In symbols: turn $\lambda$ up without bound.
Directions that $\Sigma'$ penalises are crushed; directions in the kernel of $\Sigma'$ are
left alone.
If deployment moves along an untouched direction, drift along that direction never goes
away.
``Cover the range'' means every deployment direction is penalised.
Closed form below: linear encoder / ridge response (Jacobian energy along $\Sigma'$ shrinks
direction $v$); deep nets inherit cover-or-floor only under the bridge
(Appendix~\ref{app:conditional-licenses}).

\begin{theorem}[Coverage or a permanent floor]
\label{thm:game}
Let $\Sigma',\Sigmatask\succeq0$, $v\in\mathbb{R}^{d_x}$ with $\|v\|_2=1$, and $\lambda>0$.
Define the linear--quadratic response and its deployment drift by
\[
w_\lambda(\Sigma';v):=(I+2\lambda\Sigma')^{-1}v,\qquad
D_{\Sigmatask}(w):=w^\top\Sigmatask w.
\]
The designer chooses $\Sigma'\succeq0$ (optionally under $\Tr(\Sigma')\le B$);
nature chooses unit $v\in\rangeop(\Sigmatask)$; the payoff is
$\lim_{\lambda\to\infty}D_{\Sigmatask}\bigl(w_\lambda(\Sigma';v)\bigr)$.
The minimax value is $0$ if and only if
$\rangeop(\Sigma')\supseteq\rangeop(\Sigmatask)$; under that coverage,
$D_{\Sigmatask}\bigl(w_\lambda(\Sigma';v)\bigr)=O(\lambda^{-2})$ for each fixed unit $v$.
Equivalently (unit length only normalises the packaging),
\[
\lim_{\lambda\to\infty}w_\lambda(\Sigma';v)=\Pi_{\ker(\Sigma')}v,\qquad
\lim_{\lambda\to\infty}D_{\Sigmatask}\bigl(w_\lambda(\Sigma';v)\bigr)
=\bigl\|\Sigmatask^{1/2}\Pi_{\ker(\Sigma')}v\bigr\|_2^2,
\]
so the limit vanishes for every $v\in\rangeop(\Sigmatask)$ precisely under range coverage.
\end{theorem}
\textit{Meaning.}
Miss any deployment direction and no larger $\lambda$ removes the floor.
Finite training weight is a different regime: practical soft training need not yet show
the asymptotic floor that oracle hard projection forces
(Proposition~\ref{prop:pmh-cap}).

\begin{corollary}[Even-spread / isotropic penalty beats task-only training]
\label{cor:even-spread-vs-erm}
In the linear--quadratic response of Theorem~\ref{thm:game}, let $\Sigma'=cI$ with $c>0$
(even-spread $=$ isotropic penalty) and fix any $\Sigmatask\succeq 0$ together with a unit
$v$ satisfying
$v^\top\Sigmatask v>0$.
Write $D_0:=v^\top\Sigmatask v$ for the unregularised ($\lambda=0$) drift.
Then for every $\lambda>0$,
\[
D_{\Sigmatask}\bigl(w_\lambda(cI;v)\bigr)
=\frac{D_0}{(1+2\lambda c)^2}
<D_0,
\]
and $\lim_{\lambda\to\infty}D_{\Sigmatask}\bigl(w_\lambda(cI;v)\bigr)=0$.
\end{corollary}
\textit{Meaning.}
Even-spread (isotropic $\Sigma'\propto I$) needs no estimate of the axes: in the
linear--quadratic response it strictly
beats task-only training at every $\lambda>0$ and drives the floor to zero for every nonzero
nuisance---the reason to treat it as default regularisation when $\Sigmatask$ is unknown.
Matching known axes improves the constant; it is not a prerequisite for getting a cover
(Figure~\ref{fig:ridge-floor-thm4}).
Under a lazy Jacobian the same identity holds up to $O(\varepsilon_{\mathrm{jac}})$
(Proposition~\ref{prop:lazy-even-spread}), still for additive $\Delta$
(Lemma~\ref{lem:lindrift}).

\begin{figure}[!htbp]
\centering
\SubmissionFig[width=0.72\linewidth]{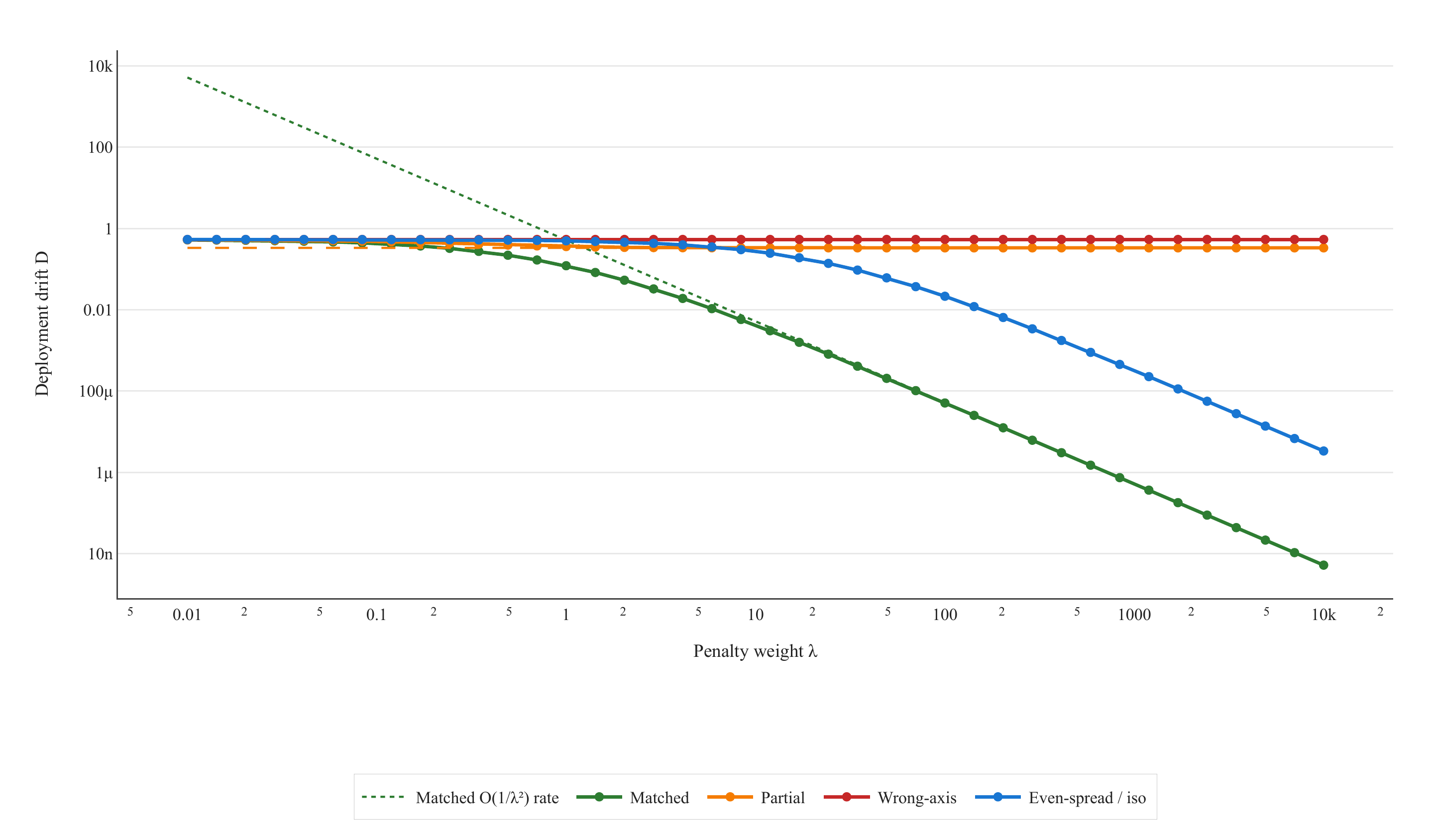}
\caption{\textbf{Even-spread / isotropic covers; a miss leaves a floor.}
Even-spread / isotropic and matched drive deployment drift down as $\lambda$ grows (matched faster);
partial or wrong-axis penalties leave a positive floor (dashed).}
\label{fig:ridge-floor-thm4}
\end{figure}

\begin{corollary}[Wrong probe $\neq$ coverage]
\label{cor:nfl-noise}
Let $Q$ be a structurally label-preserving deployment coupling with object
$\Sigmatask=\E_Q[\Delta\Delta^\top]$, and let $\Sigma'\succeq0$ be an effective response
matrix in the linear--quadratic game of Theorem~\ref{thm:game}
(under Assumption~\ref{ass:bridge}, or as the second-moment of a training noise /
augmentation / VAT law when that law induces such a response).
If $\rangeop(\Sigma')\not\supseteq\rangeop(\Sigmatask)$, then there exists
$v\in\rangeop(\Sigmatask)$ with a positive limiting drift floor that no penalty weight
removes.
Moreover, vanishing on finitely many fixed probes $v'$ need not imply vanishing on every
$v\in\rangeop(\Sigmatask)$: probes that are not the true deployment risk can look healthy
while a miss remains.
The floor is the LQ-game response; this does \emph{not} equate arbitrary VAT / aug / AT
training with the matched Jacobian penalty via the drift expansion alone.
\end{corollary}
\textit{Meaning.}
A healthy score on the probes you happened to check does not prove you covered deployment.
Only the true deployment risk does.

\begin{corollary}[Hit every direction vs.\ miss some]
\label{cor:rate-separation}
In a shared eigenbasis of $\Sigma'$ and $\Sigmatask$, write $a_i=\tau_i\tilde v_i^2$ for the
task-weighted components and let $f_\lambda(\mu)=\sum_i a_i/(1+2\lambda\mu_i)^2$ be the
ridge drift under penalty eigenvalues $\mu$ (Online Appendix~1, mismatch asymptotics).
\begin{enumerate}[label=(\roman*), leftmargin=*, itemsep=2pt]
\item \textbf{Full support.}
  If $\mu_i>0$ whenever $a_i>0$, then
  $f_\lambda(\mu)=\Theta(\lambda^{-2})$ as $\lambda\to\infty$, and for any two such penalties
  $\mu,\mu'$ the ratio $f_\lambda(\mu)/f_\lambda(\mu')$ converges to a finite positive
  constant.
  Matched, even-spread (isotropic penalty $\Sigma'\propto I$), and other full-support soft
  penalties therefore
  separate only by a bounded factor---never by a permanent floor.
  A wrong \emph{low-rank} subspace that misses support is the other case.
\item \textbf{Missing support.}
  If $\mu_j=0$ for some $a_j>0$, then $f_\lambda(\mu)\ge a_j=\Theta(1)$ for all $\lambda$:
  this is Theorem~\ref{thm:game}'s coverage floor.
\end{enumerate}
Even-spread soft penalties have full ambient support.
The recipe's matched / wrong-axis soft penalties at the nuisance rank (and oracle hard
projection) therefore miss support outside that subspace.
At finite recipe weight those floors need not yet appear; hard projection is what forces
them into view
(\S\ref{sec:soft-drift}; Online Appendix~1 soft $\lambda$ sweep).
\end{corollary}
\textit{Meaning.}
Once every deployment direction is hit, arm choice (matched vs.\ even-spread / isotropic)
changes only
the constant in front of $\lambda^{-2}$---not whether a floor exists.
Wrong low-rank penalties are the dangerous alternative to even-spread
(Table~\ref{tab:methods-as-estimators}): they can miss support entirely.

\subsection{A miss has no first-order pull}
\label{sec:local-obstruction}

\textit{Intuition.}
A penalty that never looks at a direction cannot push the encoder to ignore that direction.
Its gradient along the missing direction is zero to first order.

\begin{proposition}[Local first-order obstruction along missing directions]
\label{prop:local-obstruction}
Let $\phi_\theta$ be $C^1$ near $\theta_0$ and let
$\mathcal{L}_{\PMH(\Sigma')}$ be as in~\eqref{eq:pmh-family}.
If $q\in\ker(\Sigma')\cap\rangeop(\Sigmatask)$ with $\|q\|_2=1$, then the matched
Jacobian penalty is stationary to first order under infinitesimal encoder updates that change
only $J_\phi(\cdot)q$ at $\theta_0$.
Any first-order force that drives $\E_x[\|J_\phi q\|_2^2]\to 0$ must come from
$\mathcal{L}_{\mathrm{task}}$, not from the missing-direction penalty.
\end{proposition}
\textit{Meaning.}
A coverage miss is structural, not an optimisation accident: the penalty exerts no
first-order pull on the missing direction.

\subsection{Pointers}
\label{sec:bridge-pointer}
\label{sec:secondary-calculus}
\label{sec:worked-identifications}
\label{sec:three-predictions}

Under the ridge/linear bridge (Assumption~\ref{ass:bridge}), the matched Jacobian penalty
is the coverage game; secondary calculus and estimator constructions live in Online
Appendix~1.

%% file: sections/05_estimators.tex
\section{Charged objects and related methods}
\label{sec:estimators}
\label{sec:lemmas-D}
\label{sec:background}
\label{sec:Ak-choice}

Familiar robustness methods already penalise related second-moment geometry
\citep{sun2016coral,tzeng2014ddc,hendrycks2020augmix,wang2020mart}.
Table~\ref{tab:methods-as-estimators} names the object each method effectively hits; it is
a diagnostic map, not a coverage proof.
Whether that object covers $\Sigmatask$ is usually left unchecked.
(Seven coupling families $A_1$--$A_7$, eigengap pre-flight, and a visual summary:
Online Appendix~1.)

Even the Bayes predictor can retain an unavoidable linearised drift when the nuisance is
correlated with the label (Theorem~\ref{thm:p1-1}, Appendix~\ref{app:conditional-licenses}).
That is a second, irreducible floor---not the coverage miss of Theorem~\ref{thm:game}, which
covering $\rangeop(\Sigmatask)$ does remove.

\begin{table}[t]
\centering
\caption{What familiar methods penalise (diagnostic map; equality with $\Sigmatask$ is
usually unverified).}
\label{tab:methods-as-estimators}
\small
\setlength{\tabcolsep}{4pt}
\renewcommand{\arraystretch}{1.15}
\begin{tabularx}{\linewidth}{@{}>{\raggedright\arraybackslash}p{0.22\linewidth}
                             >{\raggedright\arraybackslash}X
                             >{\raggedright\arraybackslash}X@{}}
\toprule
Method (block) & What it effectively penalises & Main mismatch risk \\
\midrule
Adversarial training / PGD (T7B)
  & PGD-delta moment; needs PGD law $=$ deployment law
  & Attack directions may differ from deployment \\
CORAL (T1, T4A)
  & Marginal covariance difference
  & Not paired-displacement covariance \\
IRM / GroupDRO
  & Environment risks; needs explicit environments
  & Different estimand (label / spurious change) \\
Data augmentation (T2A, T3B)
  & Augmentation-delta covariance ($A_3$)
  & Corruption outside the aug.\ family \\
Mahalanobis metric learning (T1)
  & Within-class scatter ($A_1$)
  & Signal contamination at low eigengap \\
Jacobian reg.\ / VAT (T7B, T6B)
  & Even-spread / isotropic penalty ($\Sigma'\propto I$) or a local rank-$r$ projector
  & Rank-$r$ projector is not full-support cover \\
Style / KL-anchored DPO (T7A)
  & Style-pair Gram ($A_7$)
  & Preference aligned with style, not deployment \\
\bottomrule
\end{tabularx}
\end{table}

%% file: sections/06_recipe.tex
\section{Design recipe and diagnostics}
\label{sec:recipe}
\label{sec:diagnostic}
\label{sec:tdi-trajectory}
\label{sec:tdi-layout}
\label{sec:tdi-reporting}
\label{sec:dissociation}
\label{sec:scope-limits}
\label{sec:when-pmh-pointless}

Task-only training does not shrink representation energy.
Table~\ref{tab:audit-card} is the checklist; the four steps below are the default path
(name the shift, name the penalty, measure task and geometry).

\subsection{Recipe}

\begin{enumerate}[leftmargin=*, itemsep=3pt plus 0pt minus 0pt]
\item \textbf{Default: train with even-spread (isotropic penalty).}
Unless deployment axes are identified, take $\Sigma'\propto I$
(Corollary~\ref{cor:even-spread-vs-erm}): no estimate required, full ambient cover, and
strictly less deployment drift than task-only training in the linear--quadratic response.
Cap the penalty/task-loss ratio so the regulariser cannot swamp the task
(Proposition~\ref{prop:pmh-cap}).
\textit{How to run it:}~\eqref{eq:pmh-family} with $\Sigma'\propto I$; recipe $\lambda$ /
layers in Online Appendix~1.
In practice the isotropic default is additive noise $\Delta\sim\mathcal{N}(0,\sigma^2 I)$
or a Frobenius Jacobian penalty $\E\|J_\phi\|_F^2$ (full ambient support).
In the linear / lazy regime those coincide with classical Tikhonov / weight decay on $W$
(Bishop; Assumption~\ref{ass:bridge}); outside that regime weight decay is a heuristic
cousin, not an equivalence.
Start there; match only when axes are identified \emph{and} the eigengap is clear
(Table~\ref{tab:audit-card}); keep $\lambda$ small enough that the task loss still owns
training.
\item \textbf{Name the intervention when you can.}
What changes in the input while the label stays fixed
(Definition~\ref{def:label-preserving})?
\item \textbf{If axes are identified, match them.}
Choose a family $A_k$ (Online Appendix~1), build $\hat\Sigma_{\mathrm{task}}$, report the
eigengap if low-rank, and train
$\mathcal{L}_{\mathrm{task}}+\lambda\,\E_x[\Tr(J_\phi^\top J_\phi\,\hat\Sigma_{\mathrm{task}})]$
with even-spread and wrong-axis as controls---not as replacements for the default when the
estimate is weak.
\item \textbf{Report task and geometry separately.}
Do not collapse both into one leaderboard number; name the displacement law before the
claim.
\end{enumerate}

\begin{table}[!htbp]
  \centering
  \caption{Coverage checklist (minimal fields).}
  \label{tab:audit-card}
  \small
  \setlength{\tabcolsep}{4pt}
  \begin{tabular}{@{}lp{0.62\linewidth}@{}}
    \toprule
    Field & Report \\
    \midrule
    Law status & identified / algorithmic / proxy-only (glossary) \\
    Deployment object & named $Q$ and $\Sigmatask$ (or stated proxy) \\
    Estimator $A_k$ & which second-moment object training penalised \\
    Eigengap $\gamma_r$ & pre-flight when rank-$r$; distrust subspace cover if $\approx 1$ \\
    Primary endpoint & deployment task metric \\
    Geometry probe & drift / overlap, on a separate axis \\
    Bridge note & state if a near-affine deep claim is made (Appendix licenses) \\
    \bottomrule
  \end{tabular}
\end{table}

\subsection{Geometry probes}

Task accuracy alone can hide a miss: a model may look fine on a registered probe while
staying sensitive along an uncovered axis.
We therefore measure representation sensitivity separately from the task score.
When the \emph{deployment} law is roughly isotropic ($\Sigmatask\approx\sigma^2 I$) we use a
normalized multi-layer drift probe (TDI);
when an estimated subspace is available we use a nuisance-to-signal drift ratio
$D_N/D_S$ (definitions: Online Appendix~1).
When geometry ranks and task ranks disagree, that is expected under incomplete coverage
(\S\ref{sec:emp-test3})---not a reason to discard either axis.

\subsection{When covering helps}

Coverage helps when the named directions move at deployment while labels stay fixed, and
when the training penalty actually hits those directions.
It does not help when the shift changes the label, when $\Sigmatask$ is badly estimated, or
when the penalty misses the support of $\Sigmatask$.

\paragraph{When even-spread / isotropic penalty can hurt.}
Too-large $\lambda$ can cost clean accuracy; non-label-preserving shifts can make
all-direction shrinkage hit signal; even-spread is default regularisation versus none, not a
free win on
field methods' registered axes (\S\ref{sec:emp-tier3}).

This recipe applies when the coupling is structurally label-preserving and the
displacement law is named.
It does not apply to label-changing or unidentified shifts.
Known failure modes: weak eigengaps; soft training that fails to separate matched from
wrong-axis at recipe $\lambda$; signal-aligned penalties that hurt the named nuisance;
opposite rankings across geometry vs.\ registered adversarial metrics.

%% file: sections/08_empirical.tex
\section{Experiments}
\label{sec:empirical}

\input{sections/fragments/t4b_multiseed_numbers}
\input{sections/fragments/t1_estimated_w_numbers}

The theorems above are linear--quadratic.
Deep nets are not: we treat the blocks below as \emph{experiments under a specified
perturbation}, not as a proof that the ridge bridge holds for every architecture.
Whether a full nonlinear cover theorem exists is open
(Table~\ref{tab:open}, item~(O)); even-spread remains usable without waiting for one.

Reading order matches the design rule, not an arms contest:
(1)~named penalty vs.\ none (Table~\ref{tab:tier1-primary});
(2)~one controlled illustration that cover beats miss (\S\ref{sec:emp-tier2});
(3)~mixed field comparisons to show a single task score is not the claim
(Table~\ref{tab:tier3-competitors}).
Hard projection is oracle surgery on a constructed spike, not the training
loss~\eqref{eq:pmh-family}.
Detail for every block: Online Appendix~1.

\subsection{Named penalty beats none (default-regularisation breadth)}
\label{sec:emp-tier1}

\label{tab:flagships}
\label{tab:evidence-map}
\label{tab:accounting-synthesis}

Task-only training optimises the label and stops; it does not shrink representation energy.
Each row of Table~\ref{tab:tier1-primary} uses whichever named penalty that block's coupling
made available---matched, even-spread, aniso, multiscale, soft match, or hard-proj; the shared claim is
only ``named beats none,'' not that every arm is even-spread or that every cell instantiates
Corollary~\ref{cor:even-spread-vs-erm}.
Even-spread rows (T2A/T2B/T5A; Arm${=}$iso) are the no-axis default.
Do not read the table as a single-metric contest: report geometry beside the task score
in the appendix when both exist.

\begin{table}[!htbp]
  \centering
  \caption{Primary metric vs.\ unregularised baseline (B0).
  Arm = which penalty was used (not always even-spread).
  All fifteen primary rows improve. Full protocols: Online Appendix~1.}
  \label{tab:tier1-primary}
  \scriptsize
  \setlength{\tabcolsep}{2.2pt}
  \renewcommand{\arraystretch}{1.05}
  \begin{tabular}{@{}llp{1.55cm}ccc@{}}
    \toprule
    Block & Primary metric & Arm & B0 & Reg.\ & $\Delta$ \\
    \midrule
    T1 classical & Fashion SVM acc.\ ($\uparrow$) & matched & 76.46 & 84.70 & $+8.24$ \\
    T1 deep hard & Fashion OOD acc.\ ($\uparrow$) & hard-proj & $39.9\pm3.7$ & $75.8\pm1.6$ & $+35.9$ \\
    T1 soft & Fashion OOD acc.\ ($\uparrow$) & soft match & $44.0\pm6.4$ & $64.6\pm2.6$ & $+20.6$ \\
    T2A & ImageNet-C sev.\ 3 ($\uparrow$) & iso & 82.9 & 87.2 & $+4.3$ \\
    T2B & Worst-shift acc.\ ($\uparrow$) & iso & 62.50 & 69.23 & $+6.73$ \\
    T3A & PCK clean ($\uparrow$) & aniso$^\dagger$ & 47.35 & 53.03 & $+5.68$ \\
    T3B & AbsRel combined-hard ($\downarrow$) & aniso$^\dagger$ & 0.242 & 0.215 & $-0.027$ \\
    T4A & DomainNet acc.\ ($\uparrow$) & multiscale & 38.84 & 42.15 & $+3.31$ \\
    T4B & Rare-5 mIoU ($\uparrow$) & multiscale & $21.3\pm1.6$ & $23.5\pm2.4$ & $+2.2$ \\
    T5A & MAE@0.20\,\AA\ ($\downarrow$) & iso & 52.13 & 41.54 & $-10.59$ \\
    T5B & Rename ratio ($\uparrow$) & matched & 0.830 & 0.938 & $+0.108$ \\
    T6A & WER ($\downarrow$) & matched & 23.26 & 14.63 & $-8.63$ \\
    T6B & Bal.\ acc.\ $L{=}3$ ($\uparrow$) & matched & $0.279\pm0.028$ & $0.410\pm0.029$ & $+0.131$ \\
    T7B & Clean acc.\ ($\uparrow$) & matched & $78.0\pm1.4$ & $81.0\pm1.2$ & $+3.0$ \\
    T7B & Acc.\ @$\sigma{=}0.1$ ($\uparrow$) & matched & 50.2 & 73.8 & $+23.6$ \\
    \bottomrule
  \end{tabular}\\[0.35em]
  \raggedright\scriptsize
  $\Delta=$ regularised $-$ B0 (negative better when $\downarrow$).
  Arms: iso${=}$even-spread${=}$isotropic penalty $\Sigma'\propto I$;
  matched${=}\hat\Sigma_{\mathrm{task}}$ (well-identified / constructed axes);
  aniso${=}$structured non-isotropic $\hat\Sigma_{\mathrm{task}}$;
  soft match${=}$soft $\lambda$ on matched axes;
  multiscale${=}$structured multi-scale recipe; hard-proj${=}$oracle (not a trainer).\\
  $^\dagger$T3A/T3B use an estimated photometric-orbit $\hat\Sigma_{\mathrm{task}}$ with a weak
  eigengap (T3A spectrum ratio ${\approx}1.08$---distrust subspace identification; Online
  Appendix~1). Not the same confidence level as oracle/constructed matched rows.
\end{table}

All fifteen primary metrics improve over baseline
(sign test $p=2^{-14}\approx6\times10^{-5}$ against $\mathrm{Binomial}(15,1/2)$).
That is a check of \emph{directional consistency}---some named penalty beats none across
heterogeneous arms---not fifteen independent isotropic wins and not a meta-analytic effect
size.
Breadth $\neq$ even-spread-only; the LQ default-regularisation claim is
Corollary~\ref{cor:even-spread-vs-erm}.

\subsection{Cover beats miss (one controlled illustration)}
\label{sec:emp-tier2}
\label{sec:emp-test1}
\label{sec:T1}
\label{sec:T1-deep}

When axes are \emph{known by construction}, forcing cover recovers the Figure~\ref{fig:pipeline}
ordering---an illustration of Theorem~\ref{thm:game}, not the main product.
Fashion-MNIST spike (rank-$16$ $W$, $\sigma_{\mathrm{test}}=1.5$, five seeds): oracle hard
projection gives matched OOD $75.8\pm1.6\%$ vs.\ even-spread $65.5\pm4.5\%$, wrong-axis
$44.1\pm6.2\%$, baseline $39.9\pm3.7\%$ (Figure~\ref{fig:t1-deep-oracle}).
SVHN ResNet: matched$-$even-spread $+38.3$\,pp under the same protocol
(Figure~\ref{fig:t1-deep-oracle-svhn}).
Soft finite-$\lambda$ arms, estimated-$\hat W$, and $\lambda$ sweeps: Online Appendix~1
(Soft$\star$ remains open).

\begin{figure}[htbp]
\centering
\SubmissionFig{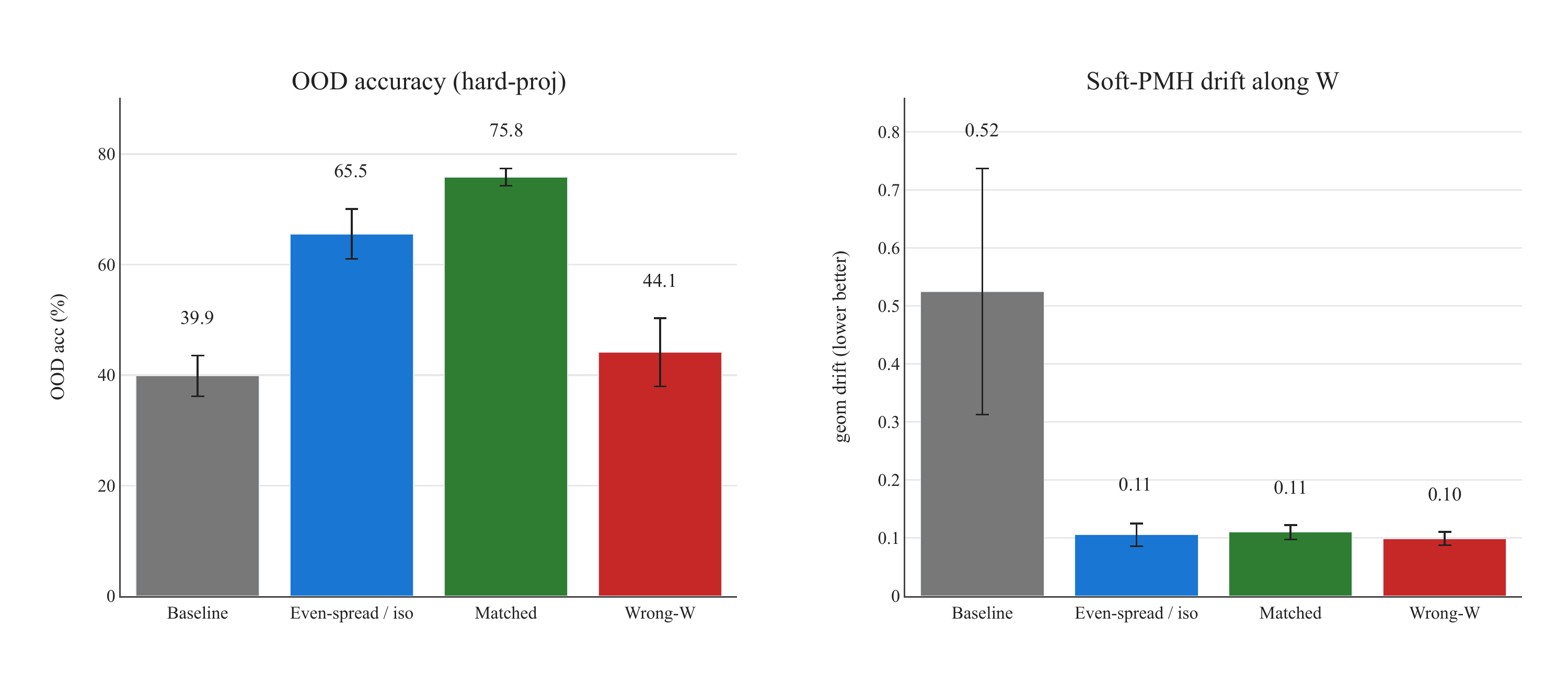}
\caption{\textbf{Oracle cartoon: cover beats miss.}
Hard-projection OOD on Fashion-MNIST: matched $\succ$ even-spread $\succ$ wrong
(soft finite-weight arms are weaker---Online Appendix~1).}
\label{fig:t1-deep-oracle}
\end{figure}

\begin{figure}[htbp]
\centering
\SubmissionFig{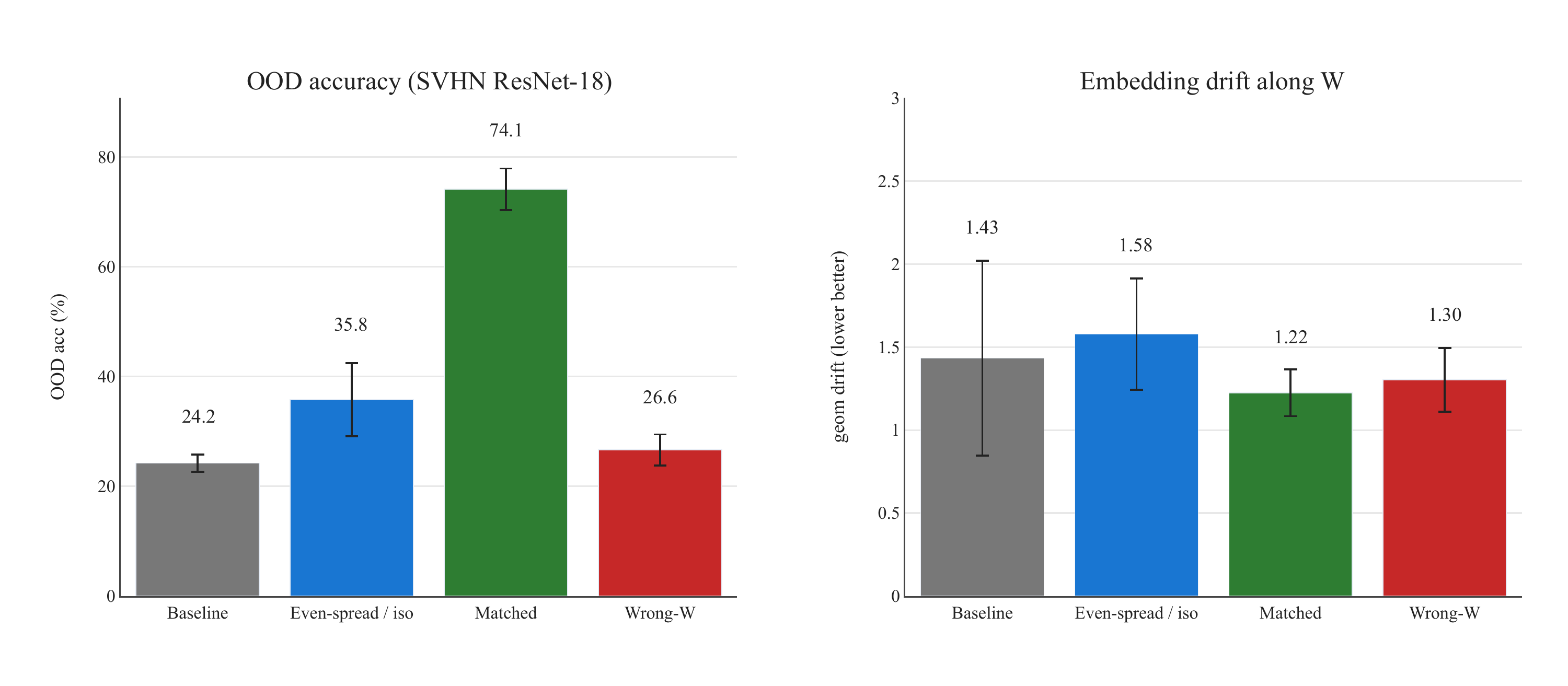}
\caption{\textbf{Same oracle ordering on SVHN (ResNet-18).}
Matched$-$even-spread $+38.3$\,pp under hard projection.}
\label{fig:t1-deep-oracle-svhn}
\end{figure}

\subsection{Split axes, not a single scoreboard}
\label{sec:emp-tier3}
\label{sec:emp-test3}
\label{sec:empirical-partials}
\label{sec:headline}
\label{sec:T7}
\label{par:tier3-cost}

A method can win the probes it regularises and lose the ones it misses
(Corollary~\ref{cor:nfl-noise})---so we refuse to collapse the paper to one leaderboard
number.
Table~\ref{tab:tier3-competitors} is that reminder, not a claim that matched PMH dominates
every field baseline.

On CIFAR-10 ViT (Figure~\ref{fig:t7b-multiseed-flagship}), matched wins clean /
Gaussian~$\sigma{=}0.1$ and loses registered PGD metrics to PGD-AT: opposite rankings on
purpose.
Elsewhere VAT wins T2B worst-shift; CE edges Whisper WER while matched still moves geometry.
Full protocols: Online Appendix~1.

\input{sections/fragments/t7b_flagship_figure}

\begin{table}[!htbp]
  \centering
  \caption{Field comparisons are mixed on purpose (task axis only).
  Two T7B/PGD-AT rows $=$ one comparison on two axes
  (Corollary~\ref{cor:nfl-noise}).
  Not the paper's main claim.}
  \label{tab:tier3-competitors}
  \small
  \setlength{\tabcolsep}{3pt}
  \begin{tabular}{@{}llp{2.2cm}p{2.3cm}l@{}}
    \toprule
    Block & Competitor & Matched (task) & Competitor (task) & Verdict \\
    \midrule
    T7B & PGD-AT
      \newline{\scriptsize (clean / Gaussian)} &
      clean $81.0$; $\sigma{=}0.1$ $73.8$ &
      clean $64.7$; $\sigma{=}0.1$ $62.2$ &
      win (clean$+$generic) \\
    T7B & PGD-AT
      \newline{\scriptsize (registered PGD)} &
      PGD@4 $16.6$; drift $0.93$ &
      PGD@4 $44.6$; drift $0.18$ &
      lose (PGD axis) \\
    T2B & VAT & worst-shift $69.2$ &
      worst-shift $81.7$ &
      lose (head-to-head) \\
    T7A & iso & syc.\ $13.5$; MC1 $0.55$ &
      syc.\ $5.8$; MC1 $0.65$ &
      lose (pilot) \\
    T6A & CE re-run & WER $14.63$ &
      WER $13.91$ &
      lose (WER); win (geom.) \\
    \bottomrule
  \end{tabular}
\end{table}

\subsection{Limits}
\label{sec:soft-drift}
\label{sec:limits-emp}
\label{par:t6b-nonsel}
\label{par:t6b-tension}

We proved the ridge case; deep blocks are experiments under a specified perturbation.
T3A/T3B even-spread shortfalls are additive-$\Delta$ scope mismatches
(\S\ref{sec:scope}), not reasons to abandon default regularisation where the model fits.
Soft$\star$ / Scale (Table~\ref{tab:open}) are open---upgrades after the default, not
prerequisites for it.
\label{sec:T2}\label{sec:T2B}\label{sec:T3}\label{sec:T3A}\label{sec:T3B}
\label{sec:T4A}\label{sec:T5}\label{sec:T5A}\label{sec:T5B}
\label{sec:T6}\label{sec:T6A}\label{sec:T6B}\label{sec:T7A}
\label{sec:predicted-failures}\label{sec:evidence-summary}
\label{sec:empirical-inventory}\label{sec:empirical-glance}

%% file: sections/fragments/t4b_multiseed_numbers.tex
\providecommand{\TFourBSeeds}{$\{7,13,42\}$}
\providecommand{\TFourBBZeroMIoU}{21.3\%}
\providecommand{\TFourBEOneMIoU}{19.0\%}
\providecommand{\TFourBMsMIoU}{23.5\%}
\providecommand{\TFourBBZeroMoto}{10.9\%}
\providecommand{\TFourBEOneMoto}{10.0\%}
\providecommand{\TFourBMsMoto}{9.0\%}
\providecommand{\TFourBBZeroMIoUStd}{$\pm$1.6}
\providecommand{\TFourBEOneMIoUStd}{$\pm$2.0}
\providecommand{\TFourBMsMIoUStd}{$\pm$2.4}

%% file: sections/fragments/t1_estimated_w_numbers.tex
\providecommand{\TOneEstMatched}{76.2\%}
\providecommand{\TOneEstIso}{67.7\%}
\providecommand{\TOneEstWrong}{44.4\%}
\providecommand{\TOneEstBZero}{45.6\%}
\providecommand{\TOneEstMatchedStd}{$\pm$1.5}
\providecommand{\TOneEstIsoStd}{$\pm$2.8}
\providecommand{\TOneEstWrongStd}{$\pm$8.7}
\providecommand{\TOneEstBZeroStd}{$\pm$4.9}
\providecommand{\TOneEstMatchedMinusIso}{+8.5}
\providecommand{\TOneEstPairedP}{0.0046}
\providecommand{\TOneEstMeanCos}{0.988}

%% file: sections/fragments/t7b_flagship_figure.tex
\IfFileExists{fig11_t7b_multiseed_flagship.pdf}{%
  \def\TSevenBFlagshipPDF{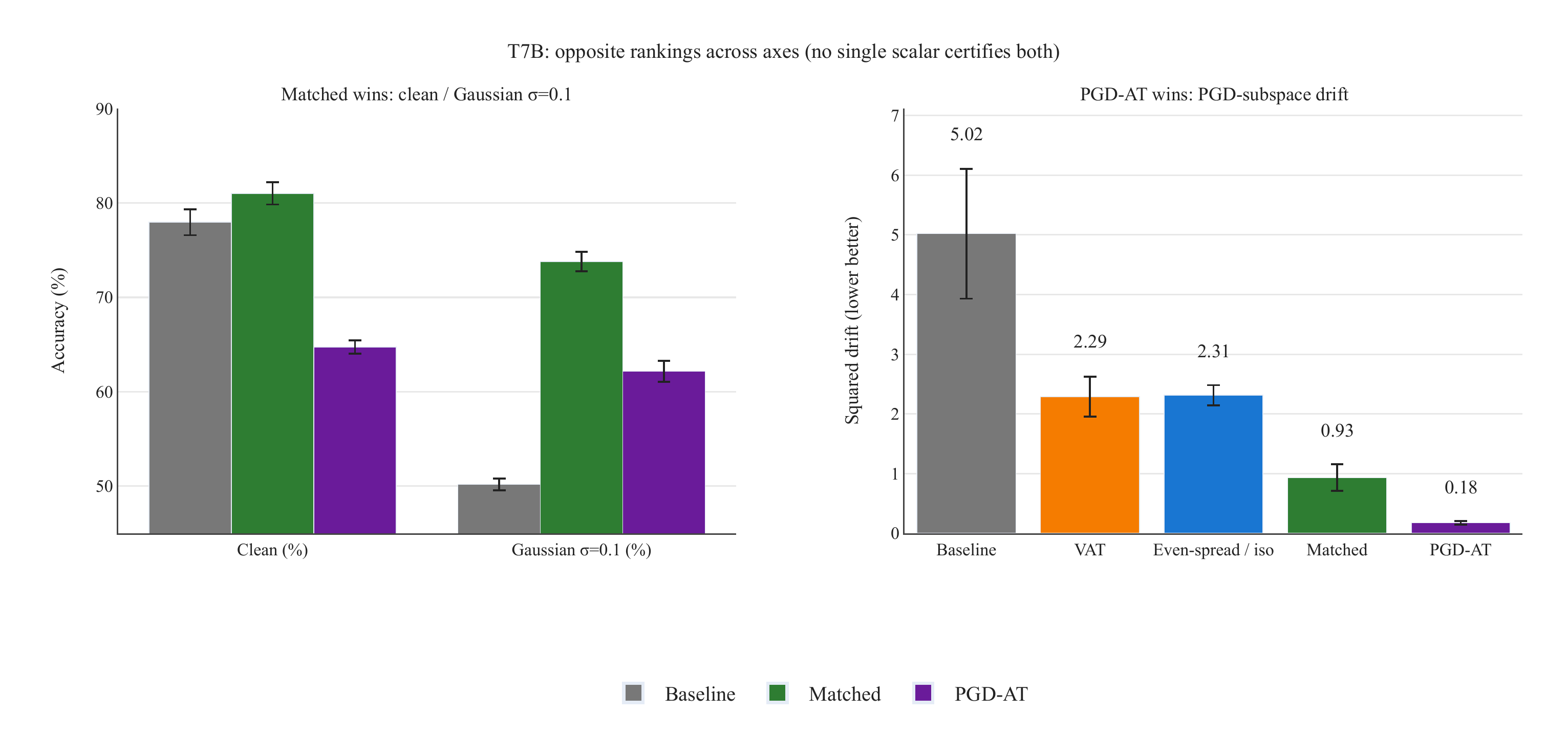}%
}{%
  \IfFileExists{figures/main/fig11_t7b_multiseed_flagship.pdf}{%
    \def\TSevenBFlagshipPDF{figures/main/fig11_t7b_multiseed_flagship.pdf}%
  }{%
    \IfFileExists{\ManuscriptRoot/figures/main/fig11_t7b_multiseed_flagship.pdf}{%
      \def\TSevenBFlagshipPDF{\ManuscriptRoot/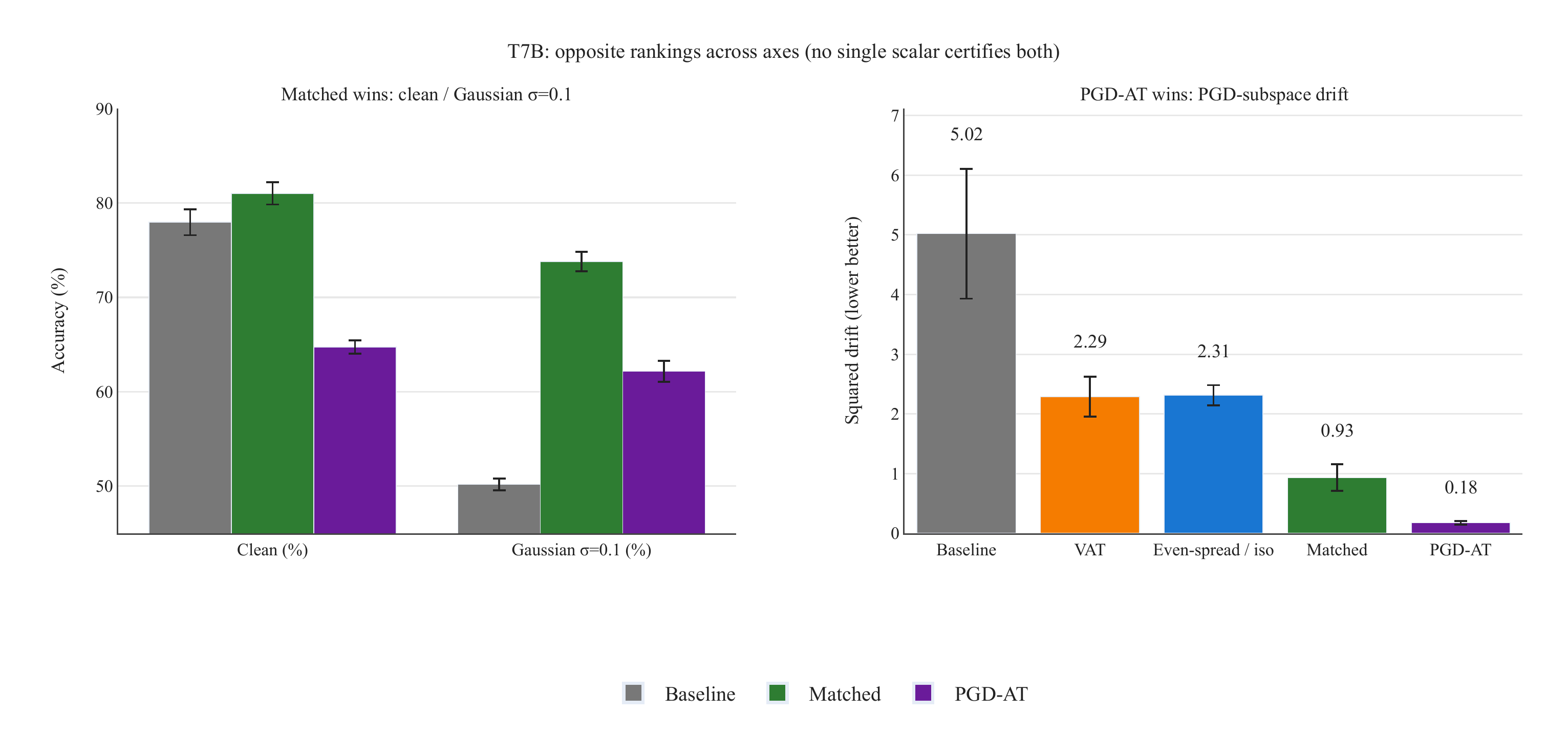}%
    }{%
      \def\TSevenBFlagshipPDF{}%
    }%
  }%
}
\ifx\TSevenBFlagshipPDF\empty
\else
\begin{figure}[htbp]
\centering
\SubmissionFig{\TSevenBFlagshipPDF}
\caption{\textbf{Do not collapse to one metric (T7B, CIFAR-10 ViT).}
\emph{Left:} clean / Gaussian---matched leads.
\emph{Right:} PGD drift---PGD-AT leads.
Opposite rankings are the point (Corollary~\ref{cor:nfl-noise}).}
\label{fig:t7b-multiseed-flagship}
\end{figure}
\fi

%% file: sections/09_discussion.tex
\section{Discussion}
\label{sec:discussion}

Ordinary training finishes when the task loss is small; it never fixes internal
representation energy, which is why small label-preserving noise so often breaks models.
The first design move is even-spread / isotropic ($\Sigma'\propto I$)---no deployment
knowledge, no architecture
change, and---in the linear--quadratic response---no coverage miss
(Corollary~\ref{cor:even-spread-vs-erm})---not a clever estimate of $\Sigmatask$.
Arm comparisons come after.

\subsection{Implications}
\label{sec:discussion-implications}

\textbf{Gap and default regularisation.}
Task loss does not control Jacobian energy; fragility under mild nuisance is predictable.
Classical noise injection already fixes that at second order
\citep{bishop1995training,an1996effects}---but only when used as default regularisation,
which current practice does not do.
Even-spread / isotropic ($\Sigma'\propto I$) needs no axis estimate, strictly beats
task-only training drift in the linear--quadratic response
(Corollary~\ref{cor:even-spread-vs-erm}), and is
the default when axes are unknown---safer \emph{as that default} than betting on a one-axis
field penalty whose registered law may miss deployment
(Table~\ref{tab:methods-as-estimators}). Named penalties of several kinds beat none in
Table~\ref{tab:tier1-primary}.

\textbf{Then match; what this opens.}
Full-support penalties differ only by a constant as $\lambda\to\infty$
(Corollary~\ref{cor:rate-separation}); a miss leaves a floor (Theorem~\ref{thm:game}).
Once internal energy is named, the programme is broader than matched PMH: losses and
architectures that make sensitivity cheap to control.
Oracle arm orderings and field head-to-heads are illustrations
(\S\ref{sec:empirical})---keep them separate from the default-regularisation claim.

\subsection{Soft training and the estimation gap}
\label{sec:discussion-soft}
\label{sec:estimation-gap}

Default regularisation already helps at recipe $\lambda$.
Soft$\star$ / Scale (Table~\ref{tab:open}) are open upgrades after the default---whether a
trainable soft penalty recovers hard-projection geometry, and how to estimate
$\Sigmatask$---not reasons to wait before using even-spread.
We do not claim a completed nonlinear cover proof; item~(O) stays open on purpose.

\subsection{Scope}
\label{sec:not-claims}
\label{sec:scope}

Formal claims apply to structurally label-preserving couplings with a named displacement
law.
Corollary~\ref{cor:even-spread-vs-erm} assumes additive $\Delta$ with bounded moments
(Lemma~\ref{lem:lindrift}); Proposition~\ref{prop:lazy-even-spread} extends it under a lazy
Jacobian by $O(\varepsilon_{\mathrm{jac}})$.
Support is scoped to additive blocks (T1, T2A, T2B)---e.g.\ T2B proxy $0.044\to0.011$ under
even-spread.
T3A (occlusion) and T3B (shared photometric orbit; secondary Gaussian steer) are
nuisance-type mismatches outside that model, not lazy/nonlinearity failures: their soft
even-spread $\varepsilon_{\mathrm{jac}}$ proxies are small and below T2B
(Online Appendix~1; T3B synthetic batch only).

\subsection{What to do with this paper}
\label{sec:practical-read}

Checklist: Table~\ref{tab:audit-card}.
Default: even-spread / isotropic.
Known structured axes: match.
Estimating $\Sigmatask$ from unlabeled data remains open.

\subsection{Open problems}
\label{sec:open-questions}
\label{sec:tables}

\input{tables_open}

Item~(O): Proposition~\ref{prop:lazy-even-spread} covers even-spread; matched cover under
GD and $\Sigmatask$ identification remain open.

%% file: tables_open.tex
\begin{table}[t]
  \centering
  \caption{Open problems left by the framework.}
  \label{tab:open}
  \label{tab:scale}
  \small
  \setlength{\tabcolsep}{3pt}
  \begin{tabular}{@{}lp{0.72\linewidth}@{}}
    \toprule
    ID & Problem \\
    \midrule
    (O) & Nonlinear bridge: general matched cover under GD (even-spread: Prop.~\ref{prop:lazy-even-spread}) \\
    D$^\dagger$ & Choosing a subspace estimate when the eigengap is weak \\
    Pred & Pre-registered forecast study (scored predictions) \\
    Scale & Upgrade: cover real directional shifts when $\Sigmatask$ must be estimated \\
    Soft$^\star$ & Upgrade: a trainable soft penalty that recovers hard-projection geometry \\
    \bottomrule
  \end{tabular}
\end{table}

%% file: sections/10_conclusion.tex
\section{Conclusion}
\label{sec:conclusion}

Task-only training optimises the label and stops; it never fixes internal representation
energy, so mild label-preserving noise can break an otherwise accurate model.
The Matching Principle asks whether the training penalty $\Sigma'$ covers deployment
$\Sigmatask$.

Three takeaways, in that order.
First, even-spread / isotropic is the no-thinking default---no axis estimate, no architecture
change, strictly less drift than task-only training in the linear--quadratic response---and a named
second-moment penalty beats none across seven domains.
Second, covering is necessary as $\lambda\to\infty$; a controlled illustration recovers
match $\succ$ even-spread $\succ$ wrong-axis when axes are forced---illustration, not the
product; Soft$\star$ is an open upgrade, not a reason to skip default regularisation.
Third, we proved the ridge case; deep blocks are experiments under a specified perturbation,
and single task scores are
not the claim (field comparisons stay mixed on purpose).

Fix internal energy by default; match when axes are known; treat losses and architectures
that control representation sensitivity as first-class design.
Checklist: \S\ref{sec:recipe}.
\MatchingPmhPackage{} implements the estimators and penalty.

\paragraph{Reproducibility.}
Frozen per-task JSON; full protocols, screens, and tables for T1--T7B are in the separately
uploaded Online Appendix~1 PDF
(\MatchingPmhRepo{}\MatchingPmhInstall).

%% file: appendix/proofs_main.tex
\section{Proofs of load-bearing results}
\label{app:proofs}

This appendix proves the main-text spine: the linearisation remainder, the consumer-class
bound, deployment consistency, the coverage game, rate separation, and the first-order
obstruction.
License statements that are not needed for those claims live in
Appendix~\ref{app:conditional-licenses}.

\begin{definition}[Structurally label-preserving deployment]
\label{def:label-preserving}
Let $C$ be label-relevant latent state, $Z$ deployment-varying state, and $E_Y$ label noise, with
$X=g(C,Z)$ and $Y=\ell(C,E_Y)$.  A deployment intervention replaces $Z$ by
$Z^+\sim K(\cdot\mid C,Z)$ while holding $C,E_Y$ fixed:
$X^+=g(C,Z^+)$ and $Y^+=\ell(C,E_Y)=Y$ almost surely.
\end{definition}
Structural label preservation is a property of the coupling, not of
$p(y\mid x)$.  It neither requires nor is implied by
$p_{\mathrm{train}}(y\mid x+\delta)=p_{\mathrm{train}}(y\mid x)$.

\begin{lemma}[Linearised vs.\ exact drift]
\label{lem:lindrift}
Let $\phi$ have $\beta$-Lipschitz Jacobian and $\sup_x\|J_\phi(x)\|_F\leq M$.
Let $\Delta$ be zero-mean with second-moment matrix
$\Sigmatask=\E[\Delta\Delta^\top]$ and scale parameter $\sigma$ (e.g.\
$\Delta=\sigma Z$ with $\E\|Z\|_2=O(1)$ and $\E\|Z\|_2^3<\infty$), independent of $x$
in the additive specialisation used by the matrix theory.
Then
\[
D_Q(\phi)=\E_{x,\Delta}\bigl[\|\phi(x+\Delta)-\phi(x)\|_2^2\bigr]
=\E_x\bigl[\Tr\!\bigl(J_\phi(x)^\top J_\phi(x)\,\Sigmatask\bigr)\bigr]+R(\phi,\sigma),
\]
with $|R(\phi,\sigma)|\leq C\bigl(M\beta\,\E\|\Delta\|_2^3+\beta^2\,\E\|\Delta\|_2^4\bigr)$
for a universal $C>0$ (hence $O(\sigma^3)$ under the stated normalisation).
Under $C^3$ regularity with integrably bounded third derivatives and central symmetry of
$\Delta$, the cubic remainder vanishes in expectation and $R(\phi,\sigma)=O(\sigma^4)$.
In the isotropic-\emph{deployment} special case $\Sigmatask=\sigma^2 I_d$ the leading term
reduces to
$\sigma^2\E_x[\|J_\phi(x)\|_F^2]$.
\end{lemma}

Short proofs follow.
Conditional licenses are stated in Appendix~\ref{app:conditional-licenses}; secondary
calculus and D1/D2/D7 remain in Online Appendix~1.

\paragraph{Proof of Lemma~\ref{lem:lindrift}.}
Write $\phi(x+\Delta)-\phi(x)=J_\phi(x)\Delta+r(x,\Delta)$ with integral remainder
$\|r\|_2\le\tfrac12\beta\|\Delta\|_2^2$ from $\beta$-Lipschitz Jacobians.
Then
\[
\|\phi(x+\Delta)-\phi(x)\|_2^2
=\|J_\phi\Delta\|_2^2+2\langle J_\phi\Delta,r\rangle+\|r\|_2^2.
\]
Taking expectations, the leading term is
$\E[\|J_\phi\Delta\|_2^2]=\E_x[\Tr(J_\phi^\top J_\phi\Sigmatask)]$.
The cross term is bounded by
$\E[2\|J_\phi\|_{\mathrm{op}}\|\Delta\|_2\|r\|_2]
\le M\beta\,\E\|\Delta\|_2^3$ (using $\|J_\phi\|_F\le M$), and
$\E\|r\|_2^2\le\tfrac14\beta^2\E\|\Delta\|_2^4$, giving the stated
$O(\E\|\Delta\|_2^3)$ remainder under the Lipschitz hypotheses.
Under the stronger $C^3$ and integrability assumptions, the only cubic cross term in a
third-order Taylor expansion is odd in $\Delta$; its expectation vanishes by central
symmetry, leaving $O(\E\|\Delta\|_2^4)=O(\sigma^4)$.
The isotropic-\emph{deployment} specialisation $\Sigmatask=\sigma^2 I$ recovers
$\sigma^2\E\|J_\phi\|_F^2$.
\(\square\)

\paragraph{Proof of Theorem~\ref{thm:consumer-class}.}
For any $L$-Lipschitz $h$,
$|h(\phi(x+n))-h(\phi(x))|\le L\|\phi(x+n)-\phi(x)\|_2$, so squaring and taking
expectations yields the upper bound $L^2 D_Q(\phi)$.
In the linearised regime with linear consumers $h(z)=\langle u,z\rangle$, $\|u\|_2\le L$,
\[
\sup_{\|u\|_2\le L}\E|\langle u,J_\phi n\rangle|^2
=L^2\lambda_{\max}\!\bigl(\E[J_\phi\Sigmatask J_\phi^\top]\bigr)
\le L^2\Tr\!\bigl(\E[J_\phi\Sigmatask J_\phi^\top]\bigr)
=L^2\tilde D_Q(\phi),
\]
with equality throughout when the Gram has rank one (and a gap of up to $d_\phi$ otherwise).
\(\square\)

\paragraph{Proof of Theorem~\ref{thm:deployment-consistency}.}
Apply Lemma~\ref{lem:lindrift} with deployment scale $\sigma$: the exact embedding drift
expands as
$\E\|\phi(x+\Delta)-\phi(x)\|_2^2
=\E_x[\Tr(J_\phi(x)^\top J_\phi(x)\Sigmatask)]+O(\sigma^3)$
under the lemma's regularity (and $O(\sigma^4)$ under its stronger symmetry hypotheses).
The population matched-\PMH{} penalty with $\Sigma'=\Sigmatask$ is exactly the displayed
leading quadratic form, hence the second-order shadow of \emph{deployment embedding drift}
under the true coupling~$Q$.
The classical noise-injection / Tikhonov reading
\citep{bishop1995training,an1996effects,wager2013dropout} is the same expansion for
matched penalties; the present contribution is the deployment object and the matching design
rule, not a proved expansion of task-only training risk.
\(\square\)

\paragraph{Proof of Theorem~\ref{thm:game}.}
By the spectral theorem,
$(I+2\lambda\Sigma')^{-1}\to\Pi_{\ker(\Sigma')}$ in operator norm.  Continuity of the
quadratic form gives
\[
D_{\Sigmatask}(w_\lambda(\Sigma';v))
\longrightarrow
\|\Sigmatask^{1/2}\Pi_{\ker(\Sigma')}v\|_2^2,
\]
which proves the fixed-$v$ assertion.

If $\ker(\Sigma')\subseteq\ker(\Sigmatask)$, the limit is zero for every $v$.
Conversely suppose $\ker(\Sigma')\not\subseteq\ker(\Sigmatask)$.  Choose
$k\in\ker(\Sigma')$ with $\Sigmatask^{1/2}k\neq0$ and set
$R=\rangeop(\Sigmatask)$ and $v=\Pi_Rk$.  Then $v\neq0$.  Put
$p=\Pi_{\ker(\Sigma')}v$.  If $p\in\ker(\Sigmatask)=R^\perp$, then
$0=\langle v,p\rangle=\|p\|_2^2$, so $p=0$.  But that would make
$v\perp\ker(\Sigma')$, contradicting
$\langle v,k\rangle=\|\Pi_Rk\|_2^2=\|v\|_2^2>0$.
Thus $\Sigmatask^{1/2}p\neq0$ and the limit is positive for this
$v\in R$.  Since symmetric PSD matrices satisfy
$\ker(\Sigma')\subseteq\ker(\Sigmatask)$ iff
$\rangeop(\Sigma')\supseteq\rangeop(\Sigmatask)$, the uniform equivalence follows.
Under range coverage the kernel component is annihilated by $\Sigmatask$ and every non-kernel
component of the resolvent is $O(\lambda^{-1})$; squaring yields
$O(\lambda^{-2})$.
\(\square\)

\paragraph{Proof of Corollary~\ref{cor:even-spread-vs-erm}.}
With $\Sigma'=cI$ the resolvent is scalar:
$w_\lambda(cI;v)=(1+2\lambda c)^{-1}v$.
Hence
$D_{\Sigmatask}(w_\lambda(cI;v))=(1+2\lambda c)^{-2}\,v^\top\Sigmatask v$.
For $\lambda=0$ this equals $D_0:=v^\top\Sigmatask v$; for $\lambda>0$ and $c>0$ the
prefactor is strictly smaller than~$1$ whenever $D_0>0$.
The $\lambda\to\infty$ limit is~$0$, which also follows from Theorem~\ref{thm:game}
because $\rangeop(cI)=\mathbb{R}^{d_x}\supseteq\rangeop(\Sigmatask)$.
\(\square\)

\paragraph{Proof of Corollary~\ref{cor:nfl-noise}.}
Identify $\Sigma'$ with an effective response matrix in Theorem~\ref{thm:game}
(Assumption~\ref{ass:bridge}, or the second-moment of a training law when that law induces
such a response).
Failure of $\rangeop(\Sigma')\supseteq\rangeop(\Sigmatask)$ yields some
$v\in\rangeop(\Sigmatask)$ with positive limiting drift
(contrapositive of Theorem~\ref{thm:game}).
Fixed-$v$ vanishing requires only $\Pi_{\ker(\Sigma')}v\in\ker(\Sigmatask)$, so finitely many
fixed probes can be matched while a positive floor remains on some unseen $v'$.
Scalars that are not the true deployment risk under~$Q$ can therefore look healthy while
a miss remains; the true risk under~$Q$ does see the floor.
This scopes the floor to the LQ / bridge response; it does not by itself identify every
VAT / aug / AT procedure with \PMH{}.
\(\square\)

\paragraph{Proof of Corollary~\ref{cor:rate-separation}.}
In a shared eigenbasis write $a_i=\tau_i\tilde v_i^2$ and
$f_\lambda(\mu)=\sum_i a_i/(1+2\lambda\mu_i)^2$ (Online Appendix~1 mismatch asymptotics).
If $\mu_j=0$ and $a_j>0$, the $j$th summand equals $a_j$ for every $\lambda$, so
$f_\lambda(\mu)\ge a_j$ (Type~B).
If $\mu_i>0$ whenever $a_i>0$, expand
$(1+2\lambda\mu_i)^{-2}=(4\lambda^2\mu_i^2)^{-1}+O(\lambda^{-3})$ to obtain
$f_\lambda(\mu)=\frac{1}{4\lambda^2}\sum_{a_i>0}a_i/\mu_i^2+O(\lambda^{-3})$.
The leading coefficient is positive and finite for any such full-support penalty, so
$f_\lambda=\Theta(\lambda^{-2})$ and the ratio of two Type~A penalties tends to the ratio of
those coefficients.
\(\square\)

\paragraph{Proof of Proposition~\ref{prop:local-obstruction}.}
Write $P(\theta)=\E_x[\Tr(J_\phi(x)^\top J_\phi(x)\Sigma')]$.
Along a $C^1$ path $\theta(t)$ with $\theta(0)=\theta_0$,
\[
\frac{\dd}{\dd t}P(\theta(t))
=2\,\E_x\!\left[\Tr\!\bigl(J_\phi(x)^\top\,(\partial_t J_\phi(x))\,\Sigma'\bigr)\right].
\]
If the variation changes only the action of $J_\phi$ on $q\in\ker(\Sigma')$, then
$(\partial_t J_\phi)\,\Sigma'=0$ wherever $\partial_t J_\phi$ is supported on multiples of
$q$, because $\Sigma'q=0$.  Thus $\frac{\dd}{\dd t}P\big|_{t=0}=0$.
Any first-order decrease of $\E_x[\|J_\phi q\|_2^2]$ at a critical point of
$\mathcal{L}_{\mathrm{task}}+\lambda P$ must therefore arise from $\mathcal{L}_{\mathrm{task}}$.
The same identity holds at every $t$ along gradient flow while the path exists.
\(\square\)

\paragraph{Proof of Proposition~\ref{prop:lazy-bridge}.}
If $\varepsilon_{\mathrm{jac}}=0$, then $J_\phi(x)=\bar J$ almost everywhere, so
\[
\E_x[\Tr(J_\phi(x)^\top J_\phi(x)\Sigma')]
=\Tr(\bar J^\top\bar J\Sigma').
\]
Under Assumption~\ref{ass:bridge} with $W\equiv\bar J$, this is exactly the population
penalty in Proposition~\ref{prop:bridge}, and the coverage claim reduces to
Theorem~\ref{thm:game}.
For $\varepsilon_{\mathrm{jac}}>0$ and general matched $\Sigma'$, no coverage envelope is
claimed; see Proposition~\ref{prop:lazy-even-spread} for even-spread.
\(\square\)

\paragraph{Proof of Proposition~\ref{prop:lazy-even-spread}.}
This is a leading-order population argument under the affine factorisation of
Assumption~\ref{ass:bridge} applied to $\bar J$---not an independent derivation that
post-training deep-net drift tracks the resolvent under gradient descent.
Write $J_\phi=\bar J+E$ with $\|E\|_{\mathrm{op}}\le\varepsilon_{\mathrm{jac}}$ a.e.\ and
$\Sigma'=cI$.
Then
$\E[\Tr(J_\phi^\top J_\phi\,cI)]=c\|\bar J\|_F^2+O(\varepsilon_{\mathrm{jac}})$
($\|\bar J\|_F$ and dimensions fixed; $\|E\|_F\le\sqrt{r}\,\varepsilon_{\mathrm{jac}}$).
Carrying that remainder through the scalar resolvent of
Corollary~\ref{cor:even-spread-vs-erm} yields
$\tilde D_{\Sigmatask}=D_0/(1+2\lambda c)^2+O(\varepsilon_{\mathrm{jac}})$.
\(\square\)

\paragraph{Proof of Proposition~\ref{prop:bridge}.}
Under Assumption~\ref{ass:bridge}, $J_\phi(x)=W$ is constant, so the population \PMH{}
penalty is $\Tr(W\Sigma'W^\top)$ (up to the fixed decoder).  Row-wise completion of the
square yields the displayed objective and the resolvent minimisers; Theorem~\ref{thm:game}
applies with penalty matrix $\Sigma'$.
\(\square\)

\paragraph{Proof of Proposition~\ref{prop:gauge}.}
Immediate from $J_{\phi_c}=c\,J_\phi$ and bilinearity of the Frobenius pairing.
\(\square\)

\paragraph{Proof of Proposition~\ref{prop:pmh-cap}.}
Set $r=R_{\PMH}/\mathcal{L}_{\mathrm{task}}\in[0,\mathrm{cap}]$.
Then $f=r/(1+r)$ is increasing in $r$, so
$0\leq f\leq\mathrm{cap}/(1+\mathrm{cap})$, with equality exactly when the cap is active.
\(\square\)

%% file: appendix/conditional_licenses.tex
\section{Conditional licenses and secondary formal statements}
\label{app:conditional-licenses}
\label{sec:bridge}
\label{app:bridge}

Optional licenses (not Core Theory): when the coverage game is a training statement, the
task-only sensitivity floor / training cap, and the predictor gauge.

\subsection{Linear and lazy licenses}
\label{sec:conditional-licenses-bridge}

\begin{assumption}[Constrained linear factorisation]
\label{ass:bridge}
Fix $\Sigma'\succeq0$ and $\lambda>0$.
\begin{enumerate}[label=(\roman*), leftmargin=*, itemsep=1pt]
\item \textbf{Affine encoder:} $\phi(x)=Wx+b$ with $W\in\mathbb{R}^{d_\phi\times d_x}$.
\item \textbf{Fixed decoder:} the decoder $h$ is held fixed (not jointly optimised) and is
$L$-Lipschitz.
\item \textbf{Quadratic task proxy:} the trainable encoder parameters are scored by
$\mathcal{L}_{\mathrm{task}}^{\mathrm{lin}}(W)=\tfrac12\|W-V\|_F^2$
for a fixed target $V\in\mathbb{R}^{d_\phi\times d_x}$.
\end{enumerate}
\end{assumption}

\begin{proposition}[Population \PMH{} is the game under constrained linear factorisation]
\label{prop:bridge}
Under Assumption~\ref{ass:bridge}, population \PMH{}~\eqref{eq:pmh-family} reduces to
independent row copies of Theorem~\ref{thm:game}:
$\mathcal{L}_{\PMH(\Sigma')}^{\mathrm{lin}}(W)
=\sum_j\bigl(\tfrac12\|w_j-v_j\|_2^2+\lambda\,w_j^\top\Sigma' w_j\bigr)$,
with $w_j^\star=(I+2\lambda\Sigma')^{-1}v_j$.
Linearised drift $\tilde D_Q(\phi)=\Tr(W\Sigmatask W^\top)$ vanishes uniformly over
$\rangeop(\Sigmatask)$ iff $\rangeop(\Sigma')\supseteq\rangeop(\Sigmatask)$.
\end{proposition}

\begin{assumption}[Lazy / near-constant Jacobian regime]
\label{ass:lazy}
There exists a fixed matrix $\bar J$ such that
$\varepsilon_{\mathrm{jac}}:=\mathrm{ess\,sup}_x\|J_\phi(x)-\bar J\|_{\mathrm{op}}$
is small throughout the relevant training window, in the sense of the lazy-training /
neural-tangent regime
\citep{jacot2018ntk,chizat2019lazy,lee2019wide}.
\end{assumption}

\begin{proposition}[Lazy-regime reduction to the coverage game]
\label{prop:lazy-bridge}
If Assumption~\ref{ass:lazy} holds with $\varepsilon_{\mathrm{jac}}=0$ (frozen Jacobian
$\bar J$), then the population matched-\PMH{} penalty equals
$\E_x[\Tr(\bar J^\top\bar J\Sigma')]$, and under Assumption~\ref{ass:bridge} with
$W\equiv\bar J$ the objective reduces to independent copies of Theorem~\ref{thm:game}
(cover still requires $\rangeop(\Sigma')\supseteq\rangeop(\Sigmatask)$).
A matched-$\Sigma'$ coverage envelope for $\varepsilon_{\mathrm{jac}}>0$ is not claimed
(Remark~\ref{rem:near-affine}); the even-spread case is
Proposition~\ref{prop:lazy-even-spread}.
\end{proposition}

\begin{proposition}[Even-spread / isotropic penalty under the lazy bridge]
\label{prop:lazy-even-spread}
Under Assumptions~\ref{ass:lazy} and~\ref{ass:bridge} with frozen mean Jacobian $\bar J$,
even-spread / isotropic penalty $\Sigma'=cI$ ($c>0$), and $D_0:=v^\top\Sigmatask v>0$, the
population
penalty is $c\|\bar J\|_F^2+O(\varepsilon_{\mathrm{jac}})$ and the linearised deployment
drift satisfies
$\tilde D_{\Sigmatask}=D_0/(1+2\lambda c)^2+O(\varepsilon_{\mathrm{jac}})$
(dimensions and $\|\bar J\|_F$ fixed).
Corollary~\ref{cor:even-spread-vs-erm} therefore degrades smoothly in
$\varepsilon_{\mathrm{jac}}$.
This is a leading-order population argument under the affine factorisation of
$\bar J$, not a gradient-descent theorem for deep nets
(Appendix~\ref{app:conditional-licenses} licenses).
\end{proposition}

\begin{remark}[Near-affine Jacobian envelope]
\label{rem:near-affine}
\label{prop:near-affine}
Proposition~\ref{prop:lazy-even-spread} is the positive-$\varepsilon_{\mathrm{jac}}$
even-spread envelope; Proposition~\ref{prop:lazy-bridge} does not claim a matched
coverage envelope for $\varepsilon_{\mathrm{jac}}>0$.
Hard-projection $\varepsilon_{\mathrm{jac}}$ proxies are large
(\texttt{results/theory\_redesign/eps\_jac\_gono\_go.json}).
Soft additive-block proxies (T2B) and the T3 nuisance-type scope reading are in Online
Appendix~1; local obstruction: Proposition~\ref{prop:local-obstruction}.
\end{remark}

\begin{proposition}[Predictor gauge]
\label{prop:gauge}
For $\phi_c:=c\phi$ and $h_c(z):=h(z/c)$, one has $h_c\circ\phi_c\equiv f$ but
$J_{\phi_c}=c\,J_\phi$, so both the \PMH{} penalty and $\tilde D_Q$ scale by $c^2$:
$f$ alone does not pin Jacobian energy.
\end{proposition}
(Theorem~\ref{thm:consumer-class} dissolves the gauge at the downstream interface.)
Proofs of Propositions~\ref{prop:bridge},~\ref{prop:lazy-bridge},~\ref{prop:lazy-even-spread},
and~\ref{prop:gauge}: Appendix~\ref{app:proofs}.

\subsection{Task-only floor and training cap}
\label{sec:conditional-licenses-floor}

\begin{theorem}[Gaussian Bayes sensitivity floor]
\label{thm:p1-1}
Linear-Gaussian Online Appendix~1 model, nuisance coefficient $\rho>0$, $L$-Lipschitz head
at fixed representation scale: population-MSE Bayes factorisation satisfies
$\tilde D_Q(\phi^*_\theta,\sigma)\ge\sigma^2\rho^2/L^2$.
\end{theorem}

\begin{proposition}[Penalty budget cap]
\label{prop:pmh-cap}
If $R_{\PMH}\leq\mathrm{cap}\cdot\mathcal{L}_{\mathrm{task}}$, then the penalty fraction is at most
$\mathrm{cap}/(1+\mathrm{cap})$.
\end{proposition}
Proofs: Online Appendix~1; cap also in Appendix~\ref{app:proofs}.

%% file: appendix/task_supplement.tex
\input{appendix/task_supplement_intro_full}
\input{appendix/task_supplement_core_body}
\input{appendix/task_supplement_blocks}
\input{appendix/moved_from_main}

%% file: appendix/task_supplement_intro_full.tex
\section{Per-task protocol cards}
\label{app:tasks-full}

Compact record behind the main-text experiments (\S\ref{sec:empirical}).
Each card uses the same fields: \textbf{Setup}, \textbf{Protocol}, \textbf{Headline} or
\textbf{Takeaway}, \textbf{Failure mode} when relevant, \textbf{Competitors} when relevant,
and \textbf{JSON}.
Main-text flagship figures are not reprinted.
Figure~\ref{fig:negatives-appendix} collects representative failures.

\paragraph{Card vocabulary (quick).}
\begin{description}[leftmargin=1.4em,style=nextline,itemsep=1pt plus 0pt minus 0pt,parsep=0pt,topsep=2pt plus 0pt minus 0pt,font=\small]
  \item[$A_1$--$A_7$]
    Nuisance family index; short map in Table~\ref{tab:app-block-index}, full detail in
    Table~\ref{tab:Ak-summary}.
  \item[B0]
    Unregularised baseline (task loss only).
  \item[matched / even-spread / isotropic / wrong]
    Same as the main text: even-spread $=$ isotropic training penalty $\Sigma'\propto I$.
    Main-text Arm column shorthand: \texttt{iso}.
    Cards below spell ``even-spread'' for the same arm.
    If the \emph{deployment} law is itself isotropic ($\Sigmatask\propto I$), matched and
    even-spread coincide.
  \item[E1]
    Block recipe Jacobian-penalty arm (matched when axes are penalised; even-spread /
    isotropic on $A_2$).
  \item[aniso / multiscale]
    Structured non-isotropic recipes (not the isotropic default).
    Main-text T3A/T3B use \texttt{aniso}$^\dagger$ for estimated photometric-orbit
    $\hat\Sigma_{\mathrm{task}}$ with a weak eigengap (distrust subspace ID).
  \item[E1S]
    Signal-aligned control (penalises the wrong coordinates on purpose).
  \item[TDI / VAT / PGD / PCK / WER]
    Drift probe (\S\ref{app:tdi-detail}); virtual / PGD adversarial training; pose / ASR metrics.
\end{description}

Soft$\star$ (soft-$\lambda$ / hard-projection gap) and estimated-$\hat W$ narratives are
canonical in \S\ref{app:soft-detail} and \S\ref{sec:T1-estimated}; the T1 cards below point
there rather than repeating them.

%% file: appendix/task_supplement_core_body.tex
\begin{table}[!htbp]
\centering
\small
\caption{Block index: what each experiment \emph{names} vs.\ what it actually
uses as the penalised directions (aligned with Table~\ref{tab:nuisance-families}).
Skim for orientation; details live in the cards below.}
\label{tab:app-block-index}
\begin{tabular}{@{}llll@{}}
\toprule
Block & Family & Named construction & Charged in experiments \\
\midrule
T1 classical & $A_1$ & subspace ($n{=}W\eta$) & Identified $W$ (oracle / linear check) \\
T1 deep & $A_1$ & subspace & Oracle hard projection (Fashion / SVHN) \\
T1 est.\ $\hat W$ & $A_1$ & subspace (estimated) & Hard projection with estimated $W$ \\
T2A & $A_2$ & even-spread / isotropic & As named ($\Sigma'\propto I$) \\
T2B & $A_2$ & even-spread / isotropic & Feature-MSE / isotropic acquisition \\
T3A--T3B & $A_3$ & photometric & Augmentation / photometric modes \\
T4A--T4B & $A_4$ & domain & Representation-Gram \emph{proxy} \\
T5A--T5B & $A_5$ & compositional & T5A: isotropic position; T5B: identifier-matched \\
T6A--T6B & $A_6$ & temporal & Residual / sensor-scatter \emph{proxy} \\
T7A & $A_7$ & style & Style-Gram proxy (pilot) \\
T7B & $A_7$ & adversarial & Algorithmic PGD-delta Gram \\
\bottomrule
\end{tabular}
\end{table}

\subsection{Calibrated failures}
\label{app:negatives}

\begin{figure}[H]
\centering
\SubmissionFig[width=\linewidth]{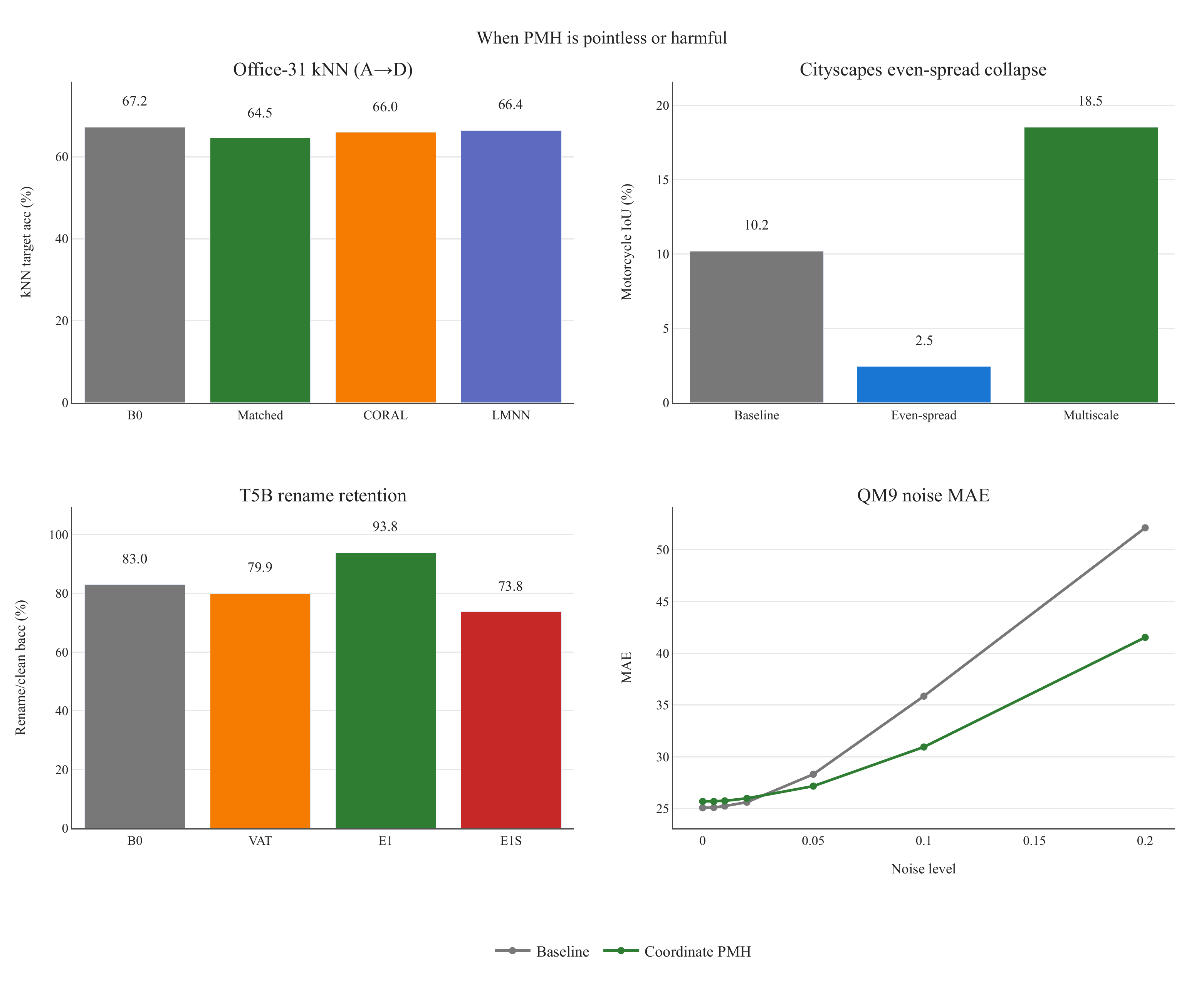}
\caption{\textbf{Where the method is predicted to struggle.}
Weak eigengap (Office-31), single-class Cityscapes miss, signal-aligned CodeBERT stress,
and QM9 clean--robust trade-off.}
\label{fig:negatives-appendix}
\end{figure}

\paragraph{Reproducibility.}
\label{app:repro-quickstart}
Frozen JSON roots (under \texttt{revision\_hardened\_2026/} where present).
Library: \MatchingPmhRepo{}.
{\footnotesize
\begin{center}
\begin{tabular}{@{}lp{0.68\linewidth}@{}}
\toprule
Block & Path \\
\midrule
T1 Fashion & \path{T1/classical_pmh/results/revision_hardened_2026/fashion_mnist/} \\
T1 deep & \path{T1/deep_oracle_pmh/results/revision_hardened_2026/} \\
T1 est.\ $\hat W$ & \path{T1/deep_oracle_pmh/results/estimated_w_20260727/deep_cnn_estimated_fashion_mnist_hard_merged.json} \\
T1 Office-31 & \path{T1/classical_pmh/results/office31_results.json} \\
T2A & \path{T2/Task2A/results/run1/} \\
T2B & \path{T2/Task2B/runs/eval_out_robust/compare_results_robust.json} \\
T3A / T3B & \path{T3/Task3A/runs/eval_out/}, \path{T3/Task3B/runs/eval_out/depth_robustness.json} \\
T4A & \path{T4/Task4A/runs/DomainNet/} \\
T4B & \path{T4/Task4B/runs/multiseed_2026/summary_ci.json} \\
T5A & \path{T5/Task5A/runs/QM9/pipeline_evals/} \\
T5B & \path{T5/Task5B/runs/Clone_pilot/pipeline_summary.json} \\
T6A & \path{T6/task6A/results/wer_results.json} \\
T6B & \path{T6/task6B/artifacts/multiseed/multiseed_summary.json} \\
T7A & \path{T7/task7A/results/behavioral_eval_7b_tqa500/} \\
T7B & \path{T7/task7B/results/revision_hardened_2026/summary_ci.json} \\
T1 soft $\lambda$ & \path{T1/deep_oracle_pmh/results/soft_lambda_cap0_20260727/sweep_summary.json} \\
$\varepsilon_{\mathrm{jac}}$ & \path{results/theory_redesign/eps_jac_gono_go.json},
  \path{results/theory_redesign/eps_jac_boundary_blocks.json} \\
\bottomrule
\end{tabular}
\end{center}
}
\FloatBarrier

\subsection{$\varepsilon_{\mathrm{jac}}$ measurement and additive-$\Delta$ scope}
\label{app:eps-jac}

Hard-projection flagships void the \emph{exact} affine license: Fashion CNN $\approx 0.72$
and ResNet-18 SVHN $\approx 3.38$ (\S\ref{sec:scope}).
Those numbers are batch standard deviations of a directional finite-difference sensitivity
(64 inputs, 8 unit directions, step $10^{-2}$): a spatial-variation \emph{proxy}, not an
exact operator norm.
\textbf{JSON:} \JsonPath{results/theory\_redesign/eps\_jac\_gono\_go.json}.
Proposition~\ref{prop:lazy-even-spread} supplies the positive-$\varepsilon_{\mathrm{jac}}$
even-spread envelope for additive $\Delta$ (Lemma~\ref{lem:lindrift}).

The same proxy on soft checkpoints
(\JsonPath{results/theory\_redesign/eps\_jac\_boundary\_blocks.json}):
\begin{itemize}[leftmargin=*,itemsep=1pt]
\item \textbf{T2B (additive default regularisation holds).}
 B0 $0.044$ $\to$ even-spread E1 $0.011$
  (${\approx}4\times$ drop) on pneumonia test images---aligned with
  Corollary~\ref{cor:even-spread-vs-erm} / Proposition~\ref{prop:lazy-even-spread}.
\item \textbf{T3A / T3B (outside the additive model).}
 Soft even-spread proxies are small
  (T3A ${\approx}6{\times}10^{-4}$ on COCO val; T3B ${\approx}4{\times}10^{-4}$),
  \emph{smaller} than T2B.
  Occlusion masks (T3A) and a shared photometric orbit with only a secondary Gaussian steer
  (T3B) are not additive bounded-moment $\Delta$, so even-spread underperformance there is a
  \emph{nuisance-type scope mismatch}, not evidence that large Jacobian wander breaks the
  ridge bridge.
  \footnote{T3B used a synthetic ImageNet-normalised batch: the NYU loader failed at
  measurement (encoding error). Do not treat that number as a measurement on the NYU
  distribution; the qualitative point---small proxy, not a large-$\varepsilon_{\mathrm{jac}}$
  story---is unchanged if the synthetic stand-in is discarded.}
\end{itemize}
Deployment consistency still supplies the second-order reading under
Lemma~\ref{lem:lindrift} when $\Delta$ is additive.

%% file: appendix/task_supplement_blocks.tex
\setlength{\abovecaptionskip}{4pt}
\setlength{\belowcaptionskip}{2pt}

\subsection{T1 --- Classical oracle-\texorpdfstring{$W$}{W} and data-driven drift}
\label{app:T1}

\noindent\textbf{Setup:} Linear subspace nuisance (family $A_1$): oracle $W$ when known,
SVD estimate otherwise.\\
\textbf{Protocol:} Classical heads (ridge / SVM / $k$NN / logistic) on synthetic data
($d{=}50$, $r{=}5$, 20 seeds), MNIST / Fashion-MNIST (oracle $+$ DCT), SVHN, and Office-31
Amazon$\to$DSLR (ResNet-18, 3 seeds).\\
\textbf{Headline:} Oracle ridge MSE stays flat at $0.101$ vs.\ baseline $0.553$ at
$\sigma_W{=}12$ (even-spread $0.409$, wrong $0.529$).
Fashion-MNIST SVM matched $84.70\%$ vs.\ baseline $76.46\%$ ($+8.24$\,pp, 5 seeds,
bootstrap CI $[+7.46,+8.82]$); DCT SVM $70.2\%$ vs.\ baseline $54.8\%$.\\
\textbf{Failure mode:} Office-31 weak eigengap ($\lambda_r/\lambda_{r+1}\approx1.00$):
$k$NN miss (matched $64.5\%$ vs.\ baseline $67.2\%$); SVM near-null.\\
\textbf{Competitors:} Table~\ref{tab:T1-baselines} is an \emph{audit} table: matched
Jacobian penalty gets the oracle nuisance subspace; CORAL / LMNN-style do not---unequal
information, not a fair head-to-head.
Fashion-SVM ``LMNN'' $38.0\%$ is a documented metric$\times$classifier collapse (below).\\
\textbf{JSON:} \JsonPath{T1/classical\_pmh/results/revision\_hardened\_2026/fashion\_mnist/};
\JsonPath{T1/classical\_pmh/results/office31\_results.json};
\JsonPath{T1/classical\_pmh/results/coverage\_staircase.json}.

\begin{table}[!htbp]
\centering
\small
\caption{\emph{Audit table} (skim Headline above).
T1 four-arm oracle-$W$ OOD accuracy (\%).
Columns: baseline / even-spread / matched / wrong.
Fashion-MNIST: 5-seed hardened; MNIST: archived 3-seed.
Not the competitor head-to-head in Table~\ref{tab:T1-baselines}.}
\label{tab:T1-oracle}
\begin{tabular}{@{}lcccc@{}}
\toprule
Classifier & B0 & even-spread & matched & wrong \\
\midrule
\multicolumn{5}{@{}l}{\textit{Fashion-MNIST (5-seed hardened)}} \\
SVM & 76.46 & 77.62 & \textbf{84.70} & 76.38 \\
$k$-NN & 77.72 & 77.74 & \textbf{79.14} & 77.56 \\
Logistic & 66.36 & 72.66 & \textbf{81.62} & 72.90 \\
\midrule
\multicolumn{5}{@{}l}{\textit{MNIST (archived 3-seed)}} \\
SVM & 90.3 & 90.2 & \textbf{94.2} & 90.0 \\
$k$-NN & 88.5 & 88.6 & \textbf{90.5} & 88.4 \\
Logistic & 79.1 & 82.7 & \textbf{89.0} & 82.5 \\
\midrule
\multicolumn{5}{@{}l}{\textit{SVHN (archived 3-seed SVM)}} \\
SVM & 23.4 & 25.6 & \textbf{43.9} & 23.6 \\
\bottomrule
\end{tabular}
\end{table}

\begin{table}[!htbp]
\centering
\small
\caption{\emph{Audit table} (unequal information---not continuous reading).
T1 competitor head-to-head under oracle $W$ (\%; archived 3-seed with CORAL /
LMNN-style).
Matched Jacobian penalty is given ground-truth $W$; CORAL / LMNN-style are not.
\textbf{Different run from Table~\ref{tab:T1-oracle}}---do not compare B0 across tables.
Bold = matched. CORAL is the strongest usable competitor; Fashion-SVM LMNN-style
$38.0\%$ collapsed (see below).}
\label{tab:T1-baselines}
\begin{tabular}{@{}llcccc@{}}
\toprule
Dataset & Classifier & B0 & \textbf{matched} & CORAL & LMNN-style \\
\midrule
MNIST & SVM & 90.9 & \textbf{94.0} & 90.3 & 83.9 \\
MNIST & $k$-NN soft & 88.8 & \textbf{89.8} & 86.3 & 84.5 \\
MNIST & Logistic & 82.4 & \textbf{87.8} & 83.7 & 81.0 \\
Fashion-MNIST & SVM & 75.8 & \textbf{84.6} & 76.6 & 38.0 \\
Fashion-MNIST & $k$-NN & 76.4 & \textbf{76.9} & 75.2 & 74.3 \\
Fashion-MNIST & Logistic & 72.9 & \textbf{80.9} & 75.5 & 64.9 \\
\bottomrule
\end{tabular}
\end{table}

\paragraph{Why Fashion-SVM LMNN-style collapses.}
The ``LMNN'' column is an LDA-style fallback (PCA to $64$ dims, then within-class
whitening on labelled source only), not a tuned large-margin metric.
On Fashion-MNIST that transform leaves $k$-NN near baseline but drives RBF-SVM to
$38.0\%$---a metric$\times$kernel failure of that baseline stack, not evidence against
matched projection (matched SVM $84.6\%$).
Head-to-head claims use CORAL.

\begin{figure}[!htbp]
\centering
\SubmissionFig[width=\linewidth]{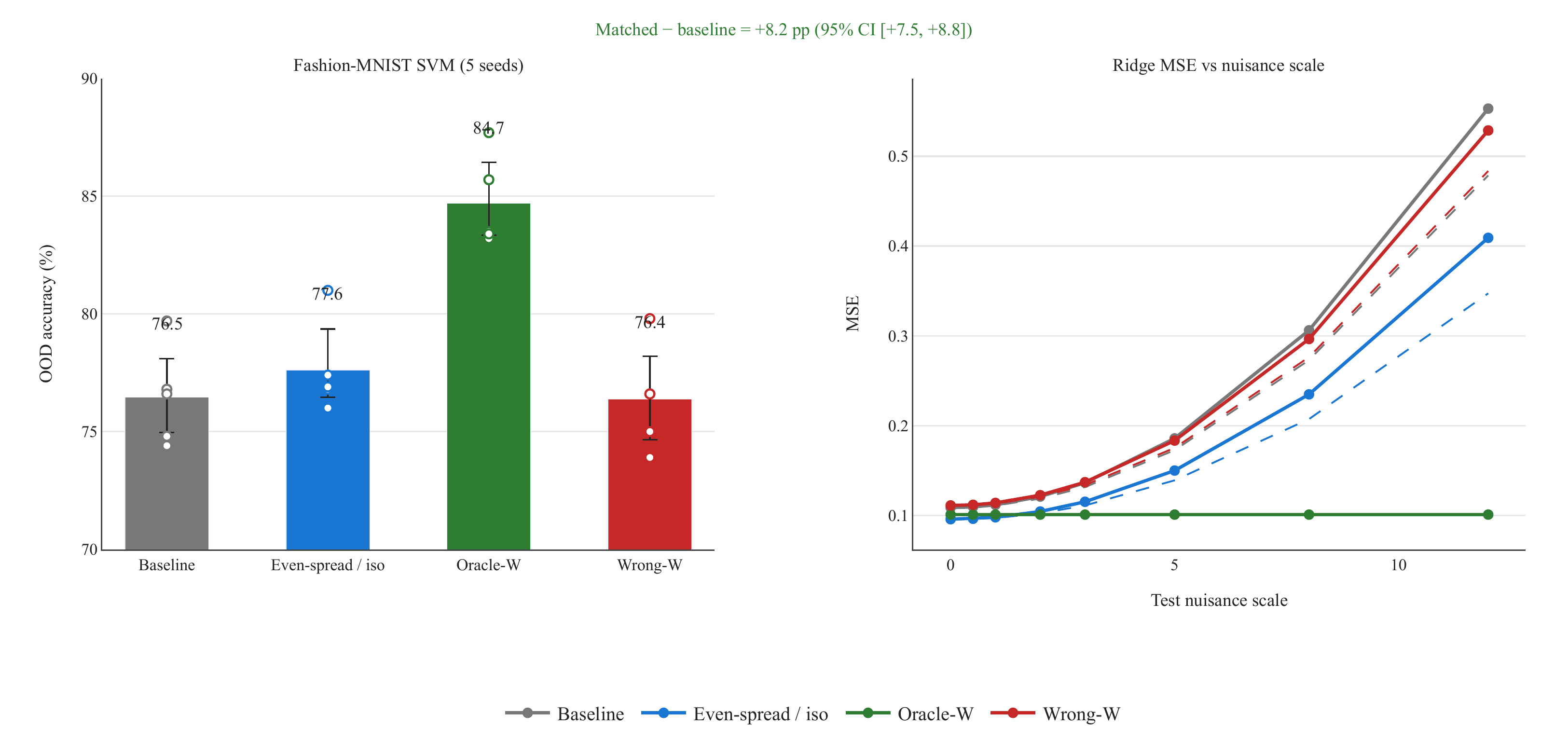}
\caption{\textbf{T1 oracle-$W$ flagship.}
Fashion-MNIST SVM (5 seeds) and closed-form ridge check.}
\label{fig:real-t1}
\label{fig:t1-oracle-flagship}
\end{figure}

\paragraph{Coverage-rank staircase (Theorem~\ref{thm:game}).}
\label{par:coverage-staircase}
At fixed penalty budget $\Tr(\Sigma')=B$, increasing the rank of $\Sigma'$ inside the
deployment range lowers the limiting drift floor until full coverage collapses it
(Figure~\ref{fig:coverage-staircase}, left).
A wrong subspace leaves the floor high.
The right panel is a descriptive deep sweep on one T7B seed---illustration, not a
closed-form claim.

\begin{figure}[!htbp]
\centering
\SubmissionFig[width=\linewidth]{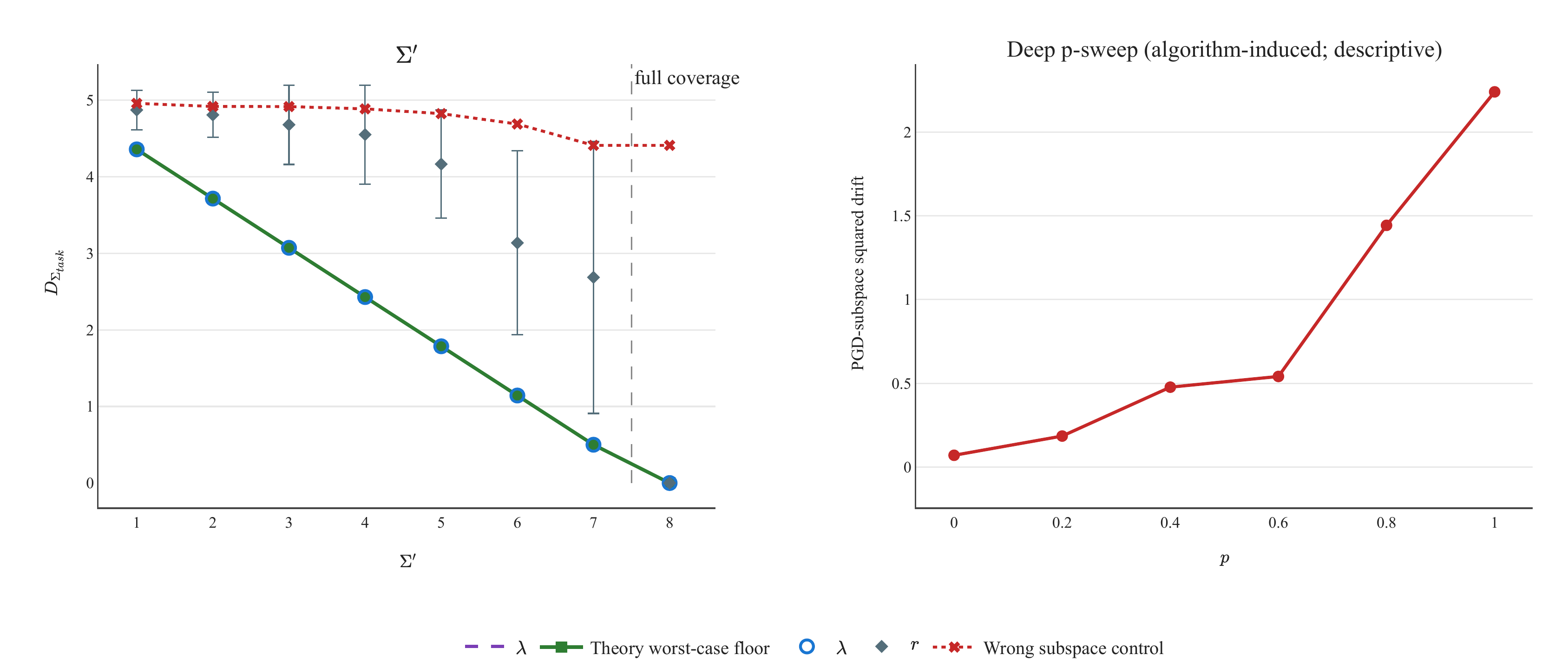}
\caption{\textbf{Coverage-rank staircase.}
Left: closed-form linear response at fixed $\Tr(\Sigma')=B$---nested penalty of the top-$r$
task directions (theory floor, large-$\lambda$ resolvent markers, random / wrong-subspace
controls); floor stays positive for $r<R$ and collapses at full coverage.
Right: descriptive deep $p$-sweep of PGD-subspace drift on T7B seed~7
(algorithm-induced; not a closed-form claim).}
\label{fig:coverage-staircase}
\end{figure}

\begin{figure}[!htbp]
\centering
\IfFileExists{\ManuscriptRoot/figures/main/fig_cor5_probe_lie.pdf}{%
\SubmissionFig[width=\linewidth]{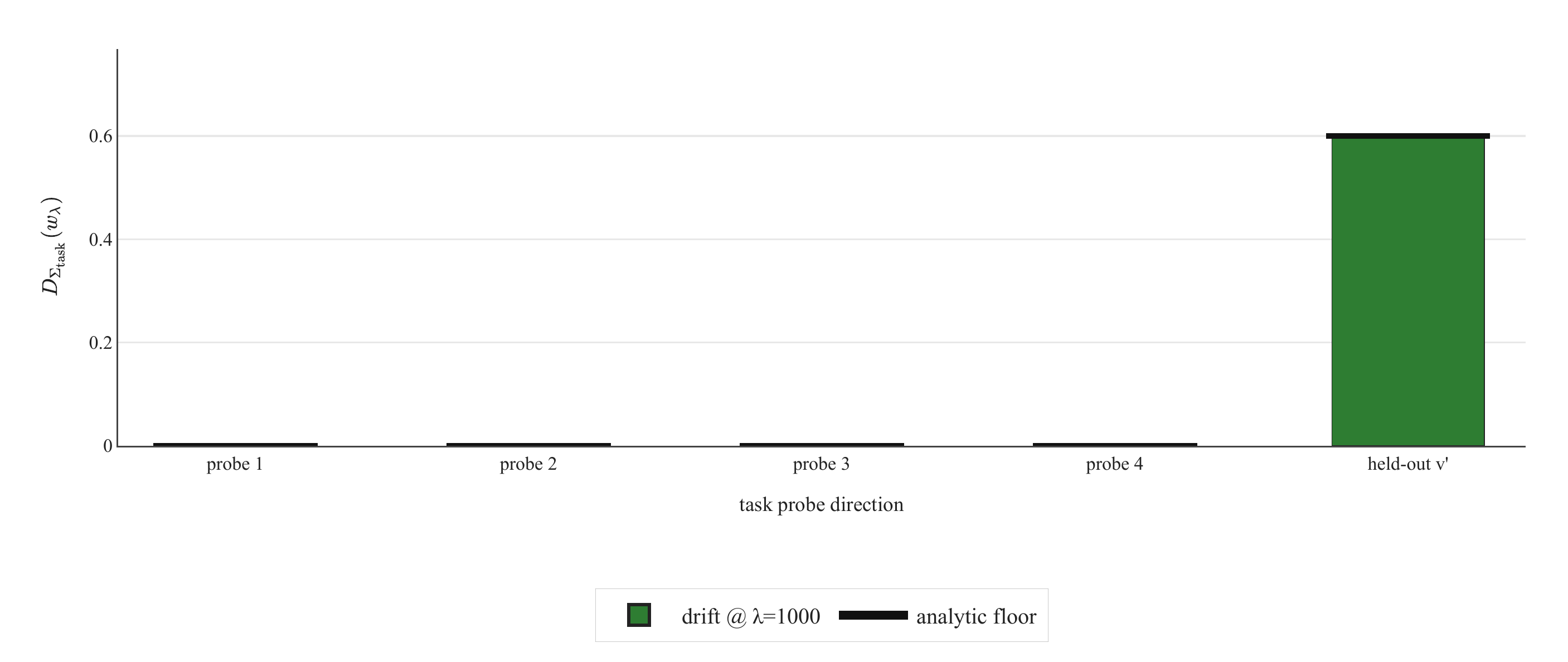}%
}{\fbox{\parbox{0.85\linewidth}{\centering Missing \texttt{fig\_cor5\_probe\_lie.pdf}; run \texttt{plot\_cor5\_probe\_lie.py}.}}}
\caption{\textbf{Closed-form probe-lie (Corollary~\ref{cor:nfl-noise}).}
$\Sigma'$ penalises a proper subspace of $\rangeop(\Sigmatask)$; $k{=}4$ probes show
vanishing drift at large $\lambda$, while a held-out $v'$ retains the analytic floor.
Not a deep-net experiment---pedagogical illustration of finite-probe non-certification.}
\label{fig:cor5-probe-lie}
\end{figure}

\begin{figure}[!htbp]
\centering
\includegraphics[width=\linewidth]{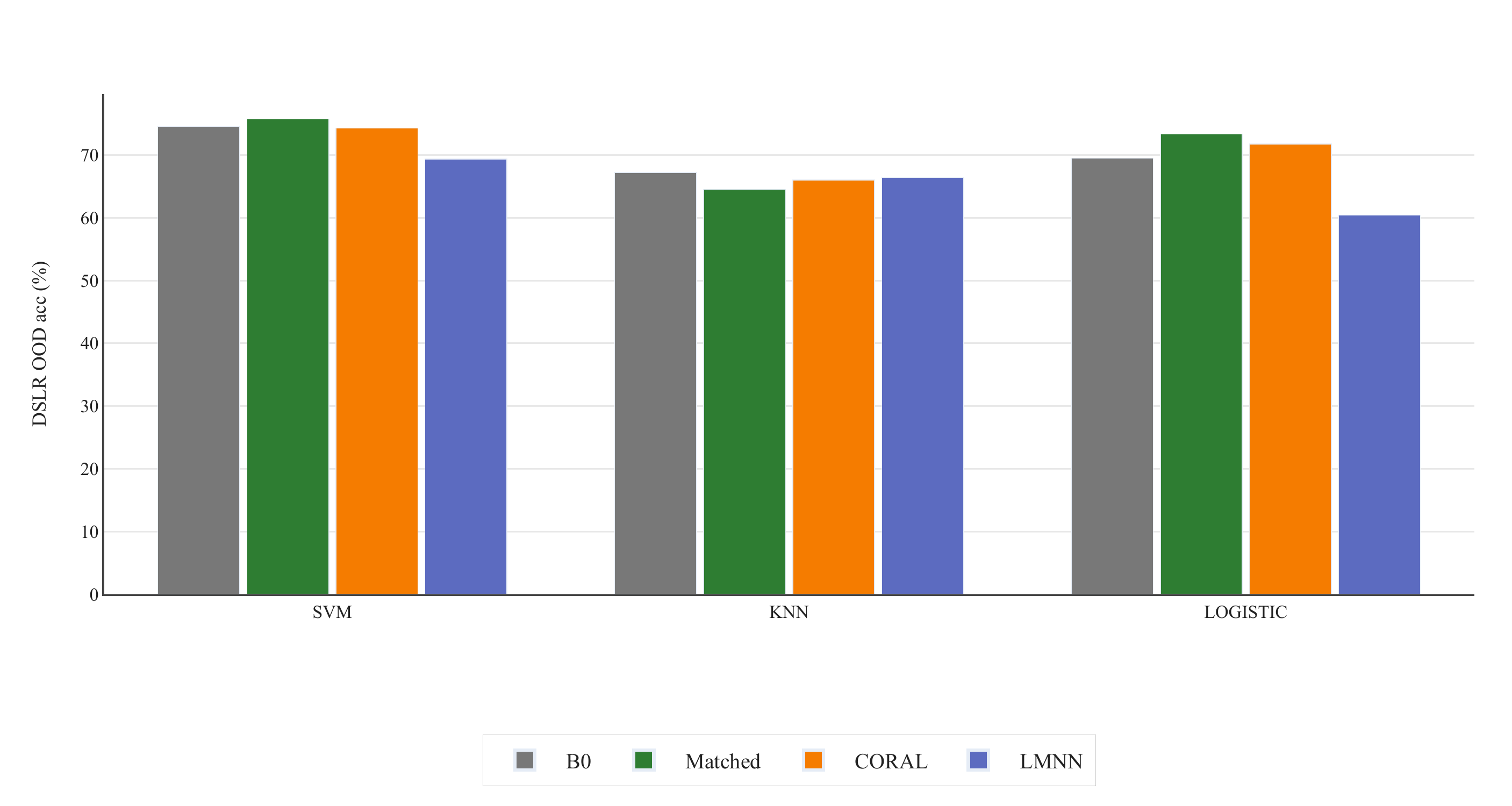}
\caption{\textbf{T1 Office-31 calibrated negative.}
Weak eigengap $\Rightarrow$ unreliable $\hat W$; $k$NN miss vs.\ B0.}
\label{fig:real-t1-office}
\end{figure}

\subsection{T1 deep --- Coupling-identified deep twins}
\label{app:T1-deep}

\noindent\textbf{Setup:} Same identified Fashion / SVHN subspace coupling; hard projection is
the coverage control.\\
\textbf{Protocol:} Fashion SmallCNN / SVHN ResNet-18; $n_{\mathrm{train}}{=}2000$,
$n_{\mathrm{test}}{=}1000$, 5 seeds, $\sigma{=}1.5$, rank~$16$.\\
\textbf{Figures:} main text (Figures~\ref{fig:t1-deep-oracle},~\ref{fig:t1-deep-oracle-svhn}).\\
\textbf{JSON:} \JsonPath{T1/deep\_oracle\_pmh/results/revision\_hardened\_2026/}.

\begin{table}[!htbp]
\centering
\small
\caption{T1 deep OOD accuracy (\%, 5 seeds).
Hard-projection rows are the oracle coverage control; Fashion soft is the proposed loss
(weak matched$-$even-spread---see \S\ref{sec:soft-drift}).
Bold = best hard-proj matched.}
\label{tab:T1-deep}
\begin{tabular}{@{}llcccc@{}}
\toprule
Setting & Model & B0 & even-spread & matched & wrong \\
\midrule
Fashion hard & SmallCNN & $39.9\pm3.7$ & $65.5\pm4.5$ & $\mathbf{75.8\pm1.6}$ & $44.1\pm6.2$ \\
Fashion hard & ResNet-18 & $51.7\pm4.6$ & $76.9\pm2.1$ & $\mathbf{82.6\pm1.4}$ & $54.3\pm4.0$ \\
MNIST hard & ResNet-18 & $81.2\pm8.3$ & $92.8\pm6.0$ & $\mathbf{97.0\pm0.6}$ & $88.0\pm2.2$ \\
SVHN hard & ResNet-18 & $24.2\pm1.6$ & $35.8\pm6.7$ & $\mathbf{74.1\pm3.8}$ & $26.6\pm2.8$ \\
Fashion soft & SmallCNN & $44.0\pm6.4$ & $61.6\pm4.1$ & $64.6\pm2.6$ & $61.8\pm5.6$ \\
\bottomrule
\end{tabular}
\end{table}

\subsection{T1 soft $\lambda$ (JSON only; Soft$\star$ write-up $\rightarrow$)}
\label{app:soft-lambda}

\noindent\textbf{Soft$\star$ canonical write-up:} \S\ref{app:soft-detail}
(Figure~\ref{fig:soft-lambda-sweep}; also main-text \S\ref{sec:soft-drift}).
This card is only the frozen path.\\
\textbf{JSON:} \JsonPath{T1/deep\_oracle\_pmh/results/soft\_lambda\_cap0\_20260727/sweep\_summary.json}.

\subsection{T1 projector-soft probe}
\label{app:projector-soft}

\noindent\textbf{Setup:} Same Fashion spike; feature-MSE soft vs.\ a column
finite-difference penalty along $W$.\\
\textbf{Protocol:} SmallCNN, three seeds, recipe $\lambda{=}2$; feature-MSE
matched/wrong, projector matched/wrong, and hard-proj matched.\\
\textbf{Takeaway:} Feature-MSE soft barely separates matched from wrong on $D_W$.
A column finite-difference projector is slightly more selective
(matched$-$wrong $+3.6$\,pp OOD) but still far below hard projection
($76.0\pm1.7\%$ OOD).
Directional soft is a step toward larger effective $\lambda$, not a replacement for
oracle geometry.\\
\textbf{JSON:} \JsonPath{T1/deep\_oracle\_pmh/results/projector\_soft\_probe/projector\_soft\_summary.json}.

\subsection{T1 estimated-\texorpdfstring{$\hat W$}{W-hat}}
\label{app:T1-estimated}

\noindent\textbf{Canonical write-up:} \S\ref{sec:T1-estimated}.
Audit numbers below; OOD stress still on the true $W$.\\
\textbf{JSON:} \JsonPath{T1/deep\_oracle\_pmh/results/estimated\_w\_20260727/deep\_cnn\_estimated\_fashion\_mnist\_hard\_merged.json}.

\begin{table}[!htbp]
\centering
\small
\caption{\emph{Audit table.}
T1 estimated-$\hat W$ hard-proj Fashion OOD (\%, 5 seeds; mean$\pm$SD).
Bold = matched.}
\label{tab:T1-estimated}
\begin{tabular}{@{}lcccc@{}}
\toprule
Setting & B0 & even-spread & matched($\hat W$) & wrong \\
\midrule
Fashion hard (est.\ $\hat W$) & $45.6\pm4.9$ & $67.7\pm2.8$ & $\mathbf{76.2\pm1.5}$ & $44.4\pm8.7$ \\
\bottomrule
\end{tabular}
\end{table}

\subsection{T2A --- ImageNet ViT-B/16, even-spread / isotropic}
\label{app:T2A}

\noindent\textbf{Setup:} Isotropic \emph{deployment} (family $A_2$: $\Sigmatask\propto I$):
matched and even-spread coincide by construction---a scale check, not a coverage test.\\
\textbf{Protocol:} ViT-B/16 on a 100-class ImageNet validation subset; even-spread /
isotropic Jacobian
penalty uses $1.2\times$ task-only epochs.\\
\textbf{Headline:} ImageNet-C severity~$3$ rises $82.9\%\to87.2\%$ ($+4.3$\,pp).
At $\sigma{=}0.10$, trajectory TDI falls $58\%$ while $\|J\|_F$ falls only $9.1\%$.\\
\textbf{Orientation:} Gated ILSVRC matched-vs-wrong orientation stress is near-null
($+0.45$\,pp); public ImageNet-100 likewise
(\path{T2/Task2A/results/orient\_ilsvrc\_aniso\_stress\_s42.json}).

\begin{figure}[!htbp]
\centering
\includegraphics[width=\linewidth]{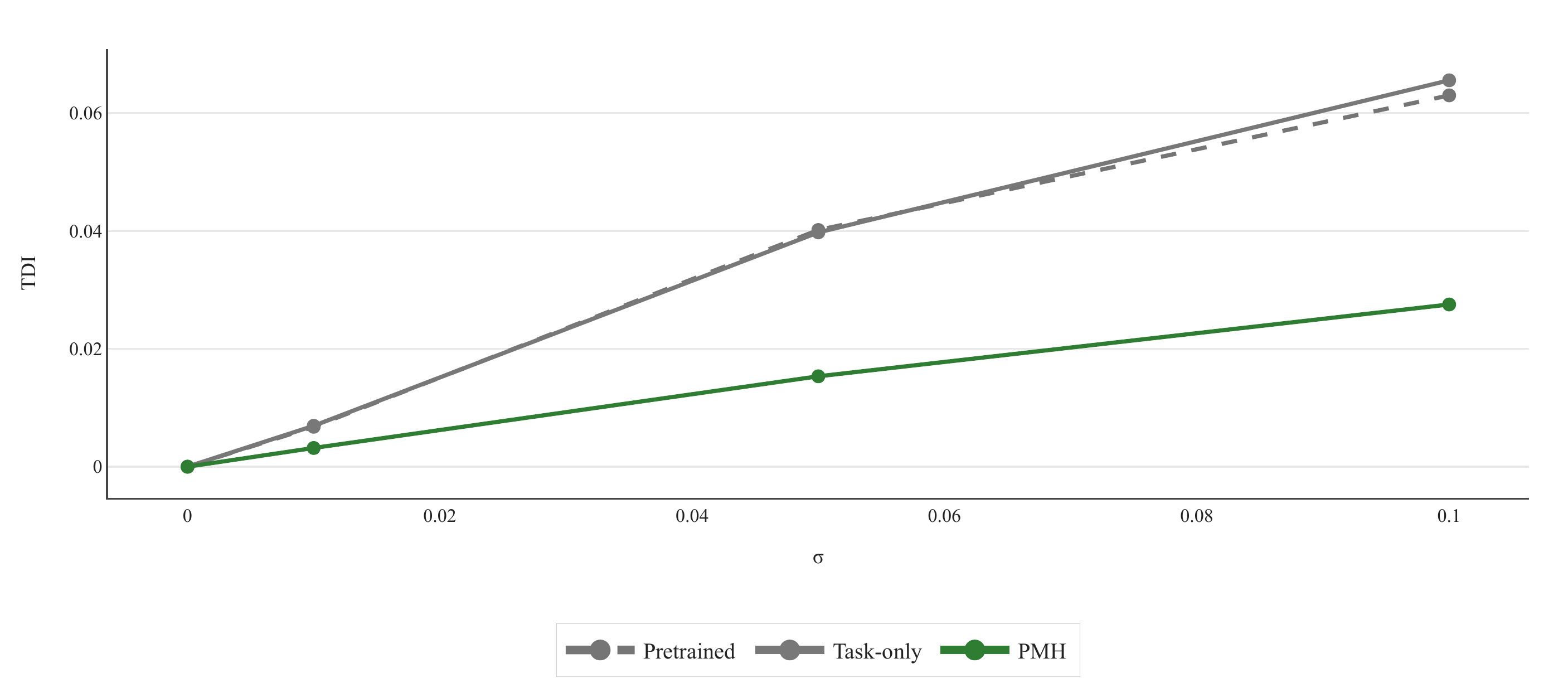}
\caption{\textbf{T2A:} TDI vs.\ Gaussian $\sigma$ (matched $=$ isotropic penalty by design).}
\label{fig:real-t2a}
\end{figure}

\begin{figure}[!htbp]
\centering
\includegraphics[width=\linewidth]{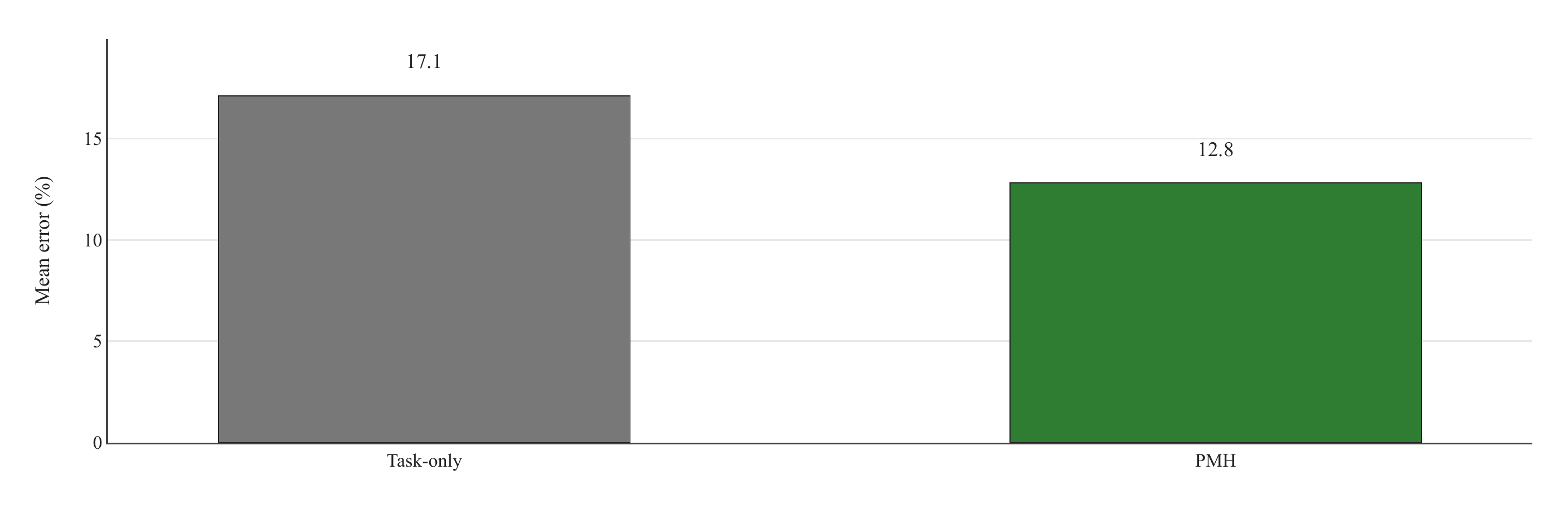}
\caption{\textbf{T2A:} ImageNet-C error by corruption family at severity~$3$;
noise/blur families improve under the even-spread Jacobian penalty.}
\end{figure}

\subsection{T2B --- Chest X-ray (geom$\neq$task)}
\label{app:T2B}

\noindent\textbf{Setup:} Chest X-ray robustness under even-spread / feature-MSE penalties
(not a matched-subspace test).\\
\textbf{Protocol:} ResNet-18 pneumonia CXR; 30 epochs, Adam $10^{-4}$, batch~$32$, seed~$42$;
baseline / VAT / Jacobian-penalty without feature-MSE / with feature-MSE (E1).
VAT: Miyato-style ($\varepsilon{=}2.0$, $1$ power iter).
E1: even-spread multi-scale feature-MSE $+$ pixel noise.\\
\textbf{Takeaway:} VAT wins clean accuracy and worst-shift; E1 wins layer-4 drift.
Both arms penalise generic objects under equal budget---useful as a competitiveness
comparison, not as a matched-vs-wrong direction test
(\S\ref{sec:emp-tier3}).\\
\textbf{JSON:} \JsonPath{T2/Task2B/runs/eval\_out\_robust/compare\_results\_robust.json}.

\begin{table}[!htbp]
\centering
\small
\caption{\emph{Audit table.} T2B headline metrics (protocol v3). Bold = column best.}
\label{tab:T2B-summary}
\label{tab:T2B-headline-app}
\label{tab:T2B-headline-main}
\begin{tabular}{@{}lcccc@{}}
\toprule
Arm & Clean$\uparrow$ & Worst-shift$\uparrow$ & L4 drift$\downarrow$ & Saliency$\uparrow$ \\
\midrule
B0 & 90.71 & 62.50 & 22.54 & \textbf{0.657} \\
VAT & \textbf{92.55} & \textbf{81.65} & 13.44 & 0.611 \\
E1-no-penalty & 90.22 & 66.91 & 16.01 & 0.643 \\
E1 (even-spread Jac.) & 89.10 & 69.23 & \textbf{9.70} & 0.628 \\
\bottomrule
\end{tabular}
\end{table}

\begin{table}[!htbp]
\centering
\small
\caption{\emph{Audit table.} T2B stage-wise mean L2 embedding drift at Gaussian $\sigma{=}0.08$ ($\downarrow$). Bold = best L4.}
\label{tab:T2B-stages}
\begin{tabular}{@{}lcccc@{}}
\toprule
Arm & L1 & L2 & L3 & L4 (pooled) \\
\midrule
B0 & 2.06 & 1.70 & 2.39 & 22.54 \\
E1-no-penalty & 2.30 & 1.13 & 1.38 & 16.01 \\
E1 (even-spread Jac.) & 1.44 & 0.78 & 0.84 & \textbf{9.70} \\
\bottomrule
\end{tabular}
\end{table}

\begin{figure}[!htbp]
\centering
\SubmissionFigCompact[width=\linewidth]{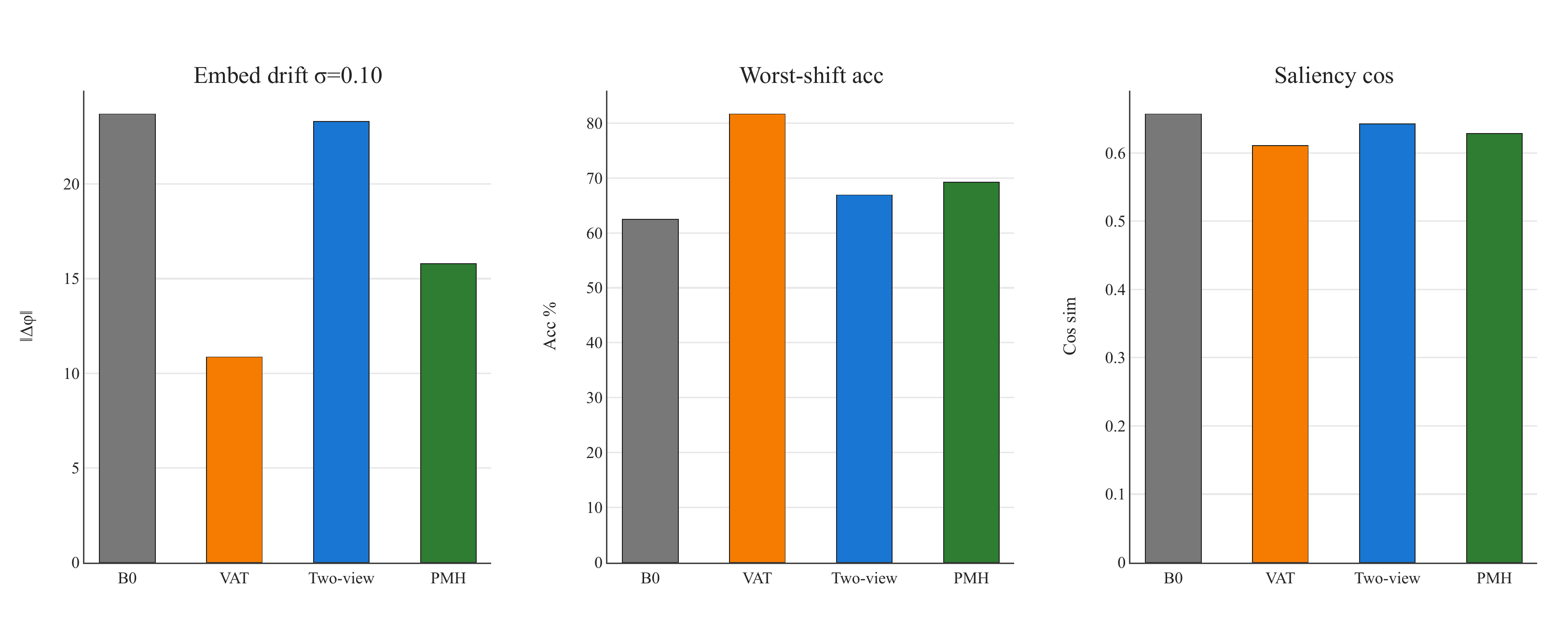}
\caption{\textbf{T2B dissociation.} E1 wins L4 drift; VAT wins task robustness; B0 wins saliency.}
\label{fig:supp-t2b}
\end{figure}

\subsection{T3A --- COCO pose / occlusion}
\label{app:T3A}

\noindent\textbf{Setup:} Pose under photometric / augmentation modes (family $A_3$); estimated
$\hat W$ has a weak eigengap (spectrum ratio $1.08$)---distrust subspace match.\\
\textbf{Protocol:} ResNet-18$+$heatmap on COCO val; equal budget 30 epochs;
baseline / VAT / even-spread / matched (photometric-orbit) Jacobian-penalty arms.\\
\textbf{Takeaway:} Clean PCK favours matched penalty ($53.03$ vs.\ baseline $47.35$).
Under heavy occlusion, even-spread falls below baseline ($27.99$ vs.\ $33.98$)---state that
plainly: occlusion is a masked removal, not additive $\Delta$
(Lemma~\ref{lem:lindrift}; \S\ref{sec:scope}).
VAT clean $13.09\%$ is a collapsed arm, kept for completeness only.\\
\textbf{JSON:} \JsonPath{T3/Task3A/runs/eval\_out/robustness\_comparison.json}.

\begin{table}[!htbp]
\centering
\small
\caption{\emph{Audit table.} T3A equal-budget PCK (\%). Bold = column best.}
\label{tab:T3A}
\begin{tabular}{@{}lcc@{}}
\toprule
Arm & PCK@0.05 clean$\uparrow$ & PCK@0.05 occ$_{0.40}$$\uparrow$ \\
\midrule
baseline & 47.35 & \textbf{33.98} \\
VAT & 13.09 & 10.00 \\
even-spread & 51.62 & 27.99 \\
matched & \textbf{53.03} & 33.54 \\
\bottomrule
\end{tabular}
\end{table}

\begin{figure}[!htbp]
\centering
\includegraphics[width=\linewidth]{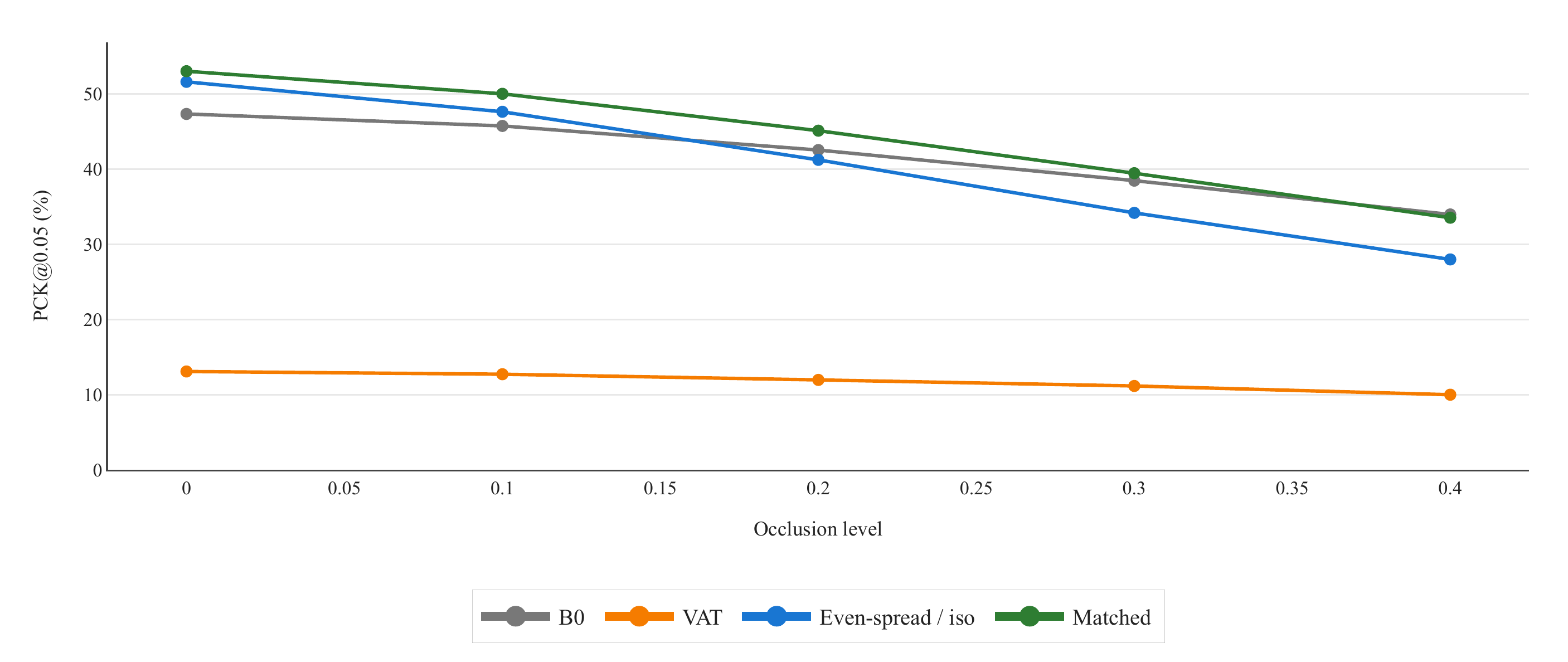}
\caption{\textbf{T3A:} PCK@0.05 vs.\ occlusion (clean gains do not transfer cleanly to heavy occ).}
\label{fig:real-t3a}
\end{figure}

\begin{figure}[!htbp]
\centering
\includegraphics[width=\linewidth]{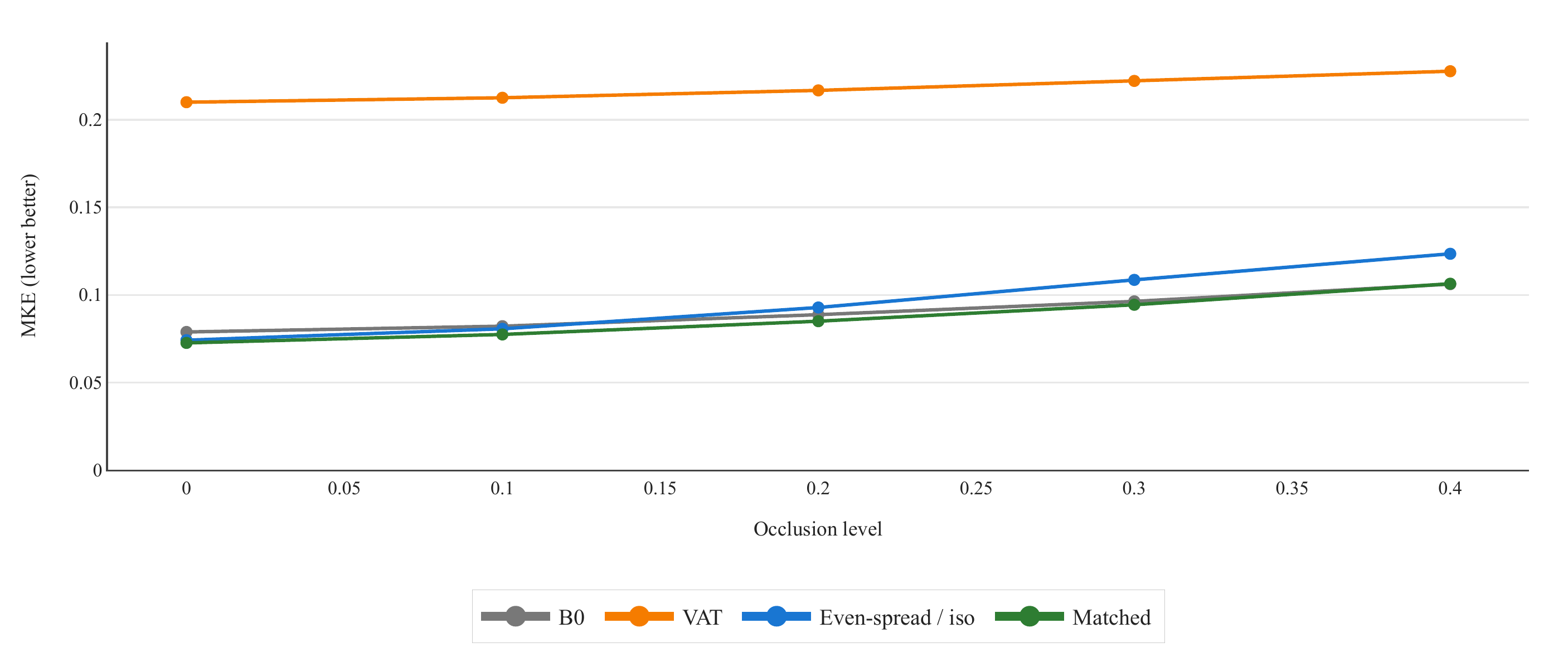}
\caption{\textbf{T3A:} MKE vs.\ occlusion (lower better).}
\end{figure}

\subsection{T3B --- NYU Depth photometric}
\label{app:T3B}

\noindent\textbf{Setup:} Photometric depth shift (family $A_3$; eigengap $\approx3.5$).
All Jacobian-penalty arms share the same photometric jitter orbit; matched/wrong only
steer a secondary Gaussian in span($W$) vs.\ a random ONB.\\
\textbf{Protocol:} ResNet-18$+$U-Net on NYU Depth V2, 50 epochs;
baseline / even-spread / matched / wrong.\\
\textbf{Note:} Wrong-$W$ wins \emph{all} combined-hard task metrics---evidence that direction
barely matters once the photo orbit is shared, not a matched$\succ$wrong claim.
The shared orbit dominates; even-spread/matched/wrong only steer a secondary Gaussian---outside
the additive-$\Delta$ model of Corollary~\ref{cor:even-spread-vs-erm}
(\S\ref{sec:scope}).\\
\textbf{JSON:} \JsonPath{T3/Task3B/runs/eval\_out/depth\_robustness.json}.

\begin{table}[!htbp]
\centering
\small
\caption{\emph{Audit table.} T3B AbsRel ($\downarrow$). Bold = column best on combined-hard.
Wrong also leads combined-hard RMSE / $\delta_1$ / SiLog in the JSON.}
\label{tab:T3B-summary}
\begin{tabular}{@{}lcc@{}}
\toprule
Arm & Clean AbsRel & Combined-hard AbsRel \\
\midrule
baseline & 0.203 & 0.242 \\
even-spread & \textbf{0.195} & 0.219 \\
matched & 0.197 & 0.215 \\
wrong & 0.197 & \textbf{0.211} \\
\bottomrule
\end{tabular}
\end{table}

\begin{figure}[!htbp]
\centering
\SubmissionFigCompact[width=\linewidth]{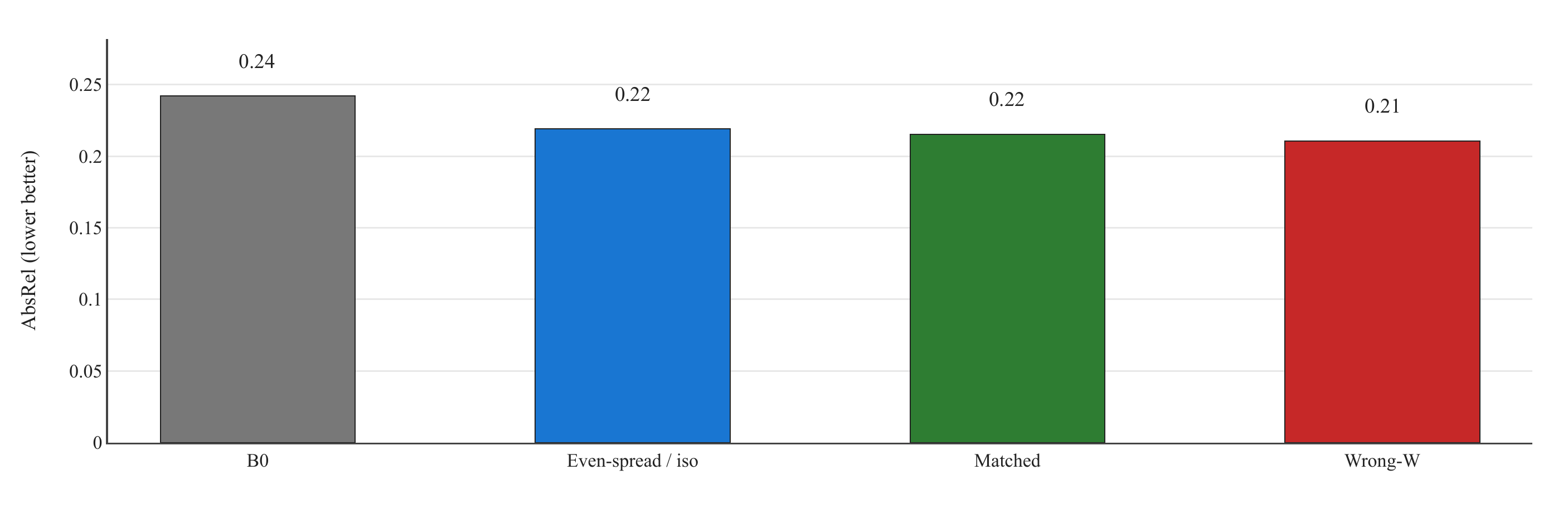}
\caption{\textbf{T3B photometric stress.}
Wrong-$W$ is best on combined-hard AbsRel; matched and even-spread only beat baseline.
Direction barely matters once the photo orbit is shared
(\S\ref{sec:discussion-soft}).}
\label{fig:supp-t3b}
\end{figure}

\subsection{T4A --- DomainNet Real$\to$Sketch}
\label{app:T4A}

\noindent\textbf{Setup:} DomainNet Real$\to$Sketch using a representation-Gram \emph{proxy}
(not the named domain construction).\\
\textbf{Protocol:} ResNet-50, 20 epochs, Gram rank 64; baseline / even-spread-pixel / multiscale.\\
\textbf{Note:} A misaligned Gram erases the gain; TDI need not track accuracy.

\begin{table}[!htbp]
\centering
\small
\caption{T4A test accuracy (\%). Bold = best.}
\label{tab:T4A-acc}
\begin{tabular}{@{}lc@{}}
\toprule
Arm & Test acc.$\uparrow$ \\
\midrule
baseline & 38.84 \\
even-spread-pixel & 39.34 \\
multiscale & \textbf{42.15} \\
\bottomrule
\end{tabular}
\end{table}

\begin{figure}[!htbp]
\centering
\SubmissionFigCompact[width=\linewidth]{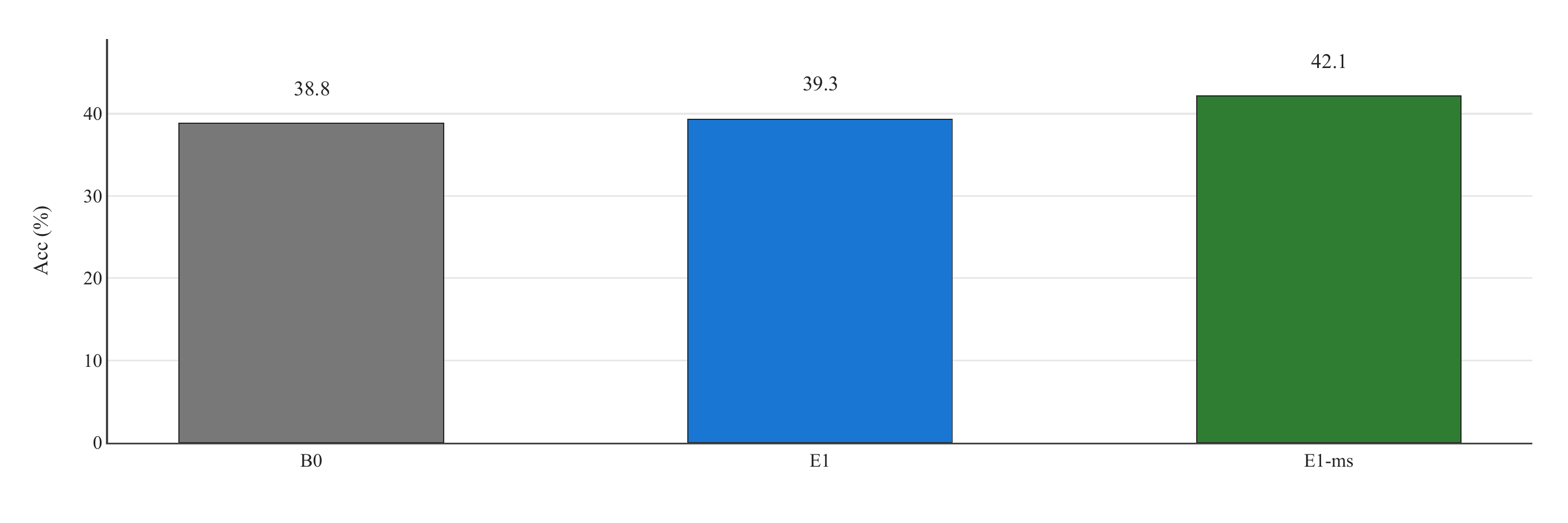}
\caption{\textbf{T4A DomainNet.} Multiscale Gram lift on test accuracy.}
\label{fig:real-t4a}
\end{figure}

\begin{figure}[!htbp]
\centering
\includegraphics[width=\linewidth]{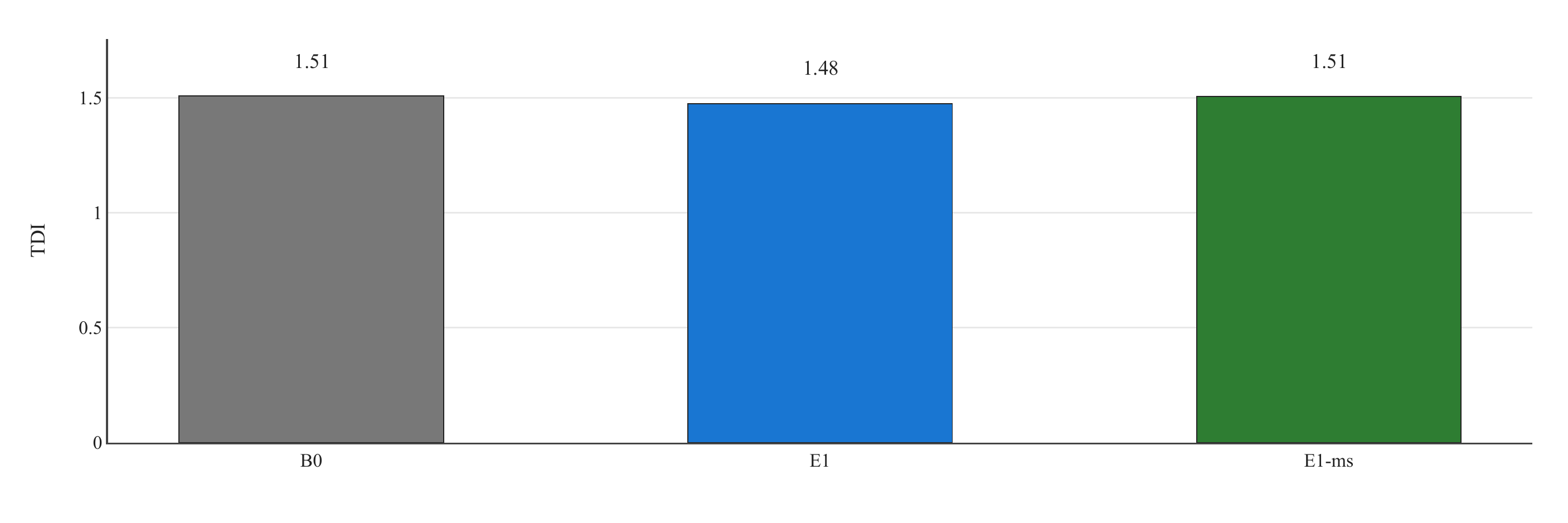}
\caption{\textbf{T4A TDI.} Accuracy and TDI need not agree.}
\end{figure}

\subsection{T4B --- GTA5$\to$Cityscapes Rare-5}
\label{app:T4B}

\noindent\textbf{Setup:} GTA5$\to$Cityscapes rare-5 segmentation with a representation-Gram
\emph{proxy}; primary endpoint rare-5 mIoU.\\
\textbf{Protocol:} DeepLabV3-ResNet50, 30 epochs; seeds $\{7,13,42\}$; B0 / even-spread / multiscale.\\
\textbf{Note:} Motorcycle IoU need not track the rare-5 aggregate; this is not a D4 coverage
claim.\\
\textbf{JSON:} \JsonPath{T4/Task4B/runs/multiseed\_2026/summary\_ci.json}.

\begin{table}[!htbp]
\centering
\small
\caption{T4B three-seed rare-5 mIoU and motorcycle IoU (\%). Bold = best rare-5 mIoU.}
\label{tab:T4B-class}
\begin{tabular}{@{}lcc@{}}
\toprule
Arm & Rare-5 mIoU & Motorcycle IoU \\
\midrule
B0 & $21.3\pm1.6$ & $10.9\pm2.6$ \\
E1 (even-spread) & $19.0\pm2.0$ & $10.0\pm5.2$ \\
E1-multiscale & $\mathbf{23.5\pm2.4}$ & $9.0\pm8.6$ \\
\bottomrule
\end{tabular}
\end{table}

\begin{figure}[!htbp]
\centering
\SubmissionFigCompact[width=\linewidth]{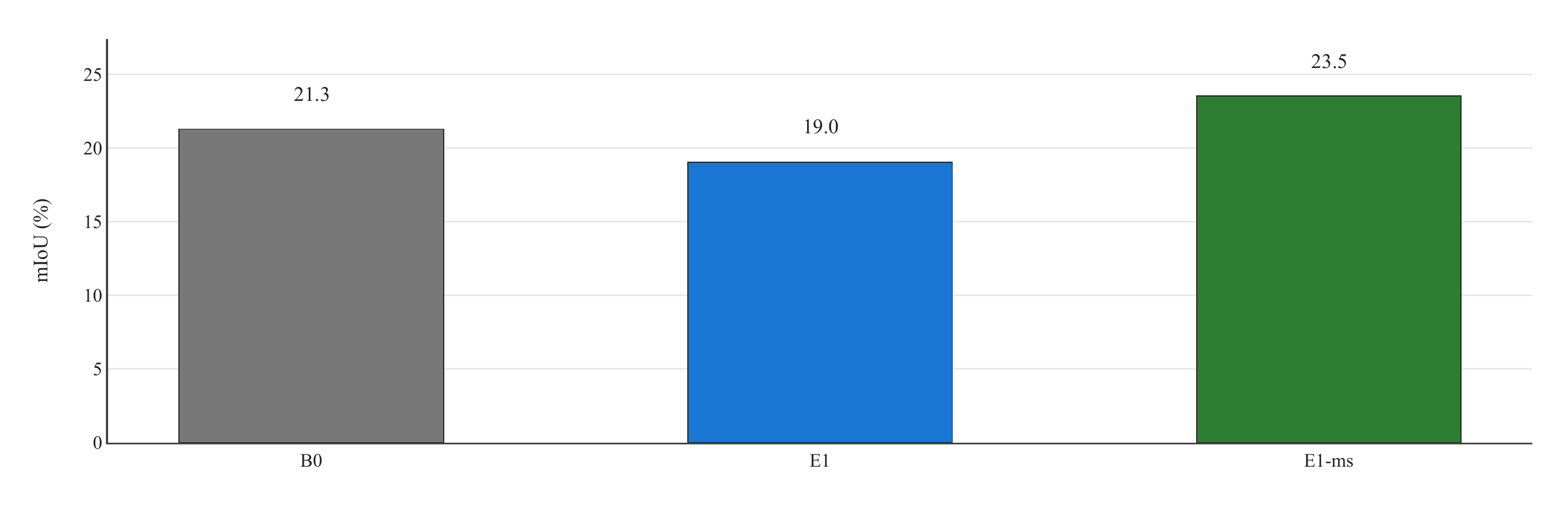}
\caption{\textbf{T4B aggregate.} Multiscale wins rare-5 mIoU; motorcycle remains high-variance.}
\label{fig:supp-t4b}
\end{figure}

\begin{figure}[!htbp]
\centering
\includegraphics[width=\linewidth]{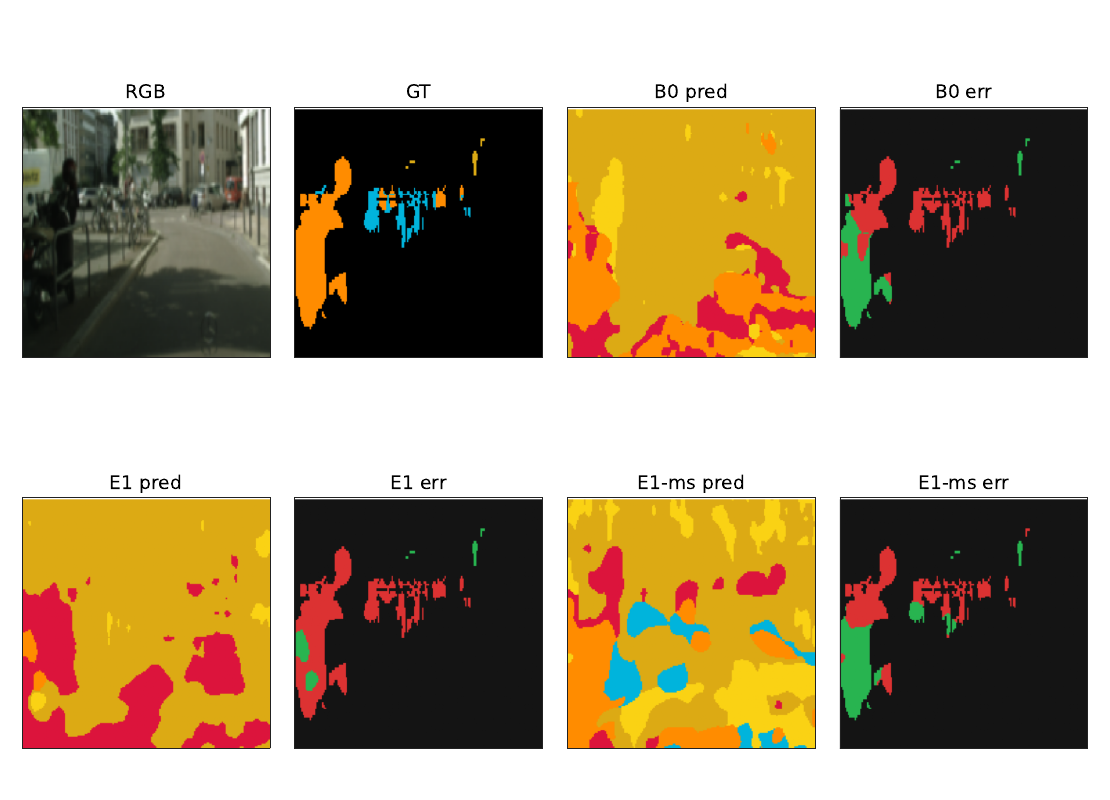}
\caption{\textbf{T4B crop (illustrative).} Multiscale recovers rare classes missed by B0/even-spread.}
\label{fig:real-t4b}
\end{figure}

\subsection{T5A --- QM9 position noise}
\label{app:T5A}

\noindent\textbf{Setup:} QM9 molecular property prediction under isotropic position-coordinate
noise ($A_5$: $\delta\sim\mathcal{N}(0,\sigma^2 I)$ on atom positions; Arm${=}$iso).\\
\textbf{Protocol:} MolGCN, 19 targets, 100 epochs; B0 / E1-small ($\sigma{=}0.01$) /
E1-large ($0.15$) / VAT.\\
\textbf{Takeaway:} Classic clean--robust trade-off: E1-small is best on clean MAE; E1-large
is best under position noise at $\sigma{=}0.20$\,\AA.\\
\textbf{JSON:} \JsonPath{T5/Task5A/runs/QM9/pipeline\_evals/run\_evals\_both\_E1/eval\_summary.json}.

\begin{table}[!htbp]
\centering
\small
\caption{T5A MAE ($\downarrow$). Bold = best clean / best $\sigma{=}0.20$.}
\label{tab:T5A-summary}
\begin{tabular}{@{}lcc@{}}
\toprule
Arm & Clean MAE & MAE@$\sigma{=}0.20$\,\AA \\
\midrule
B0 & 25.08 & 52.13 \\
VAT & 28.15 & --- \\
E1-small & \textbf{24.92} & 47.42 \\
E1-large & 25.69 & \textbf{41.54} \\
\bottomrule
\end{tabular}
\end{table}

\begin{figure}[!htbp]
\centering
\SubmissionFigCompact[width=\linewidth]{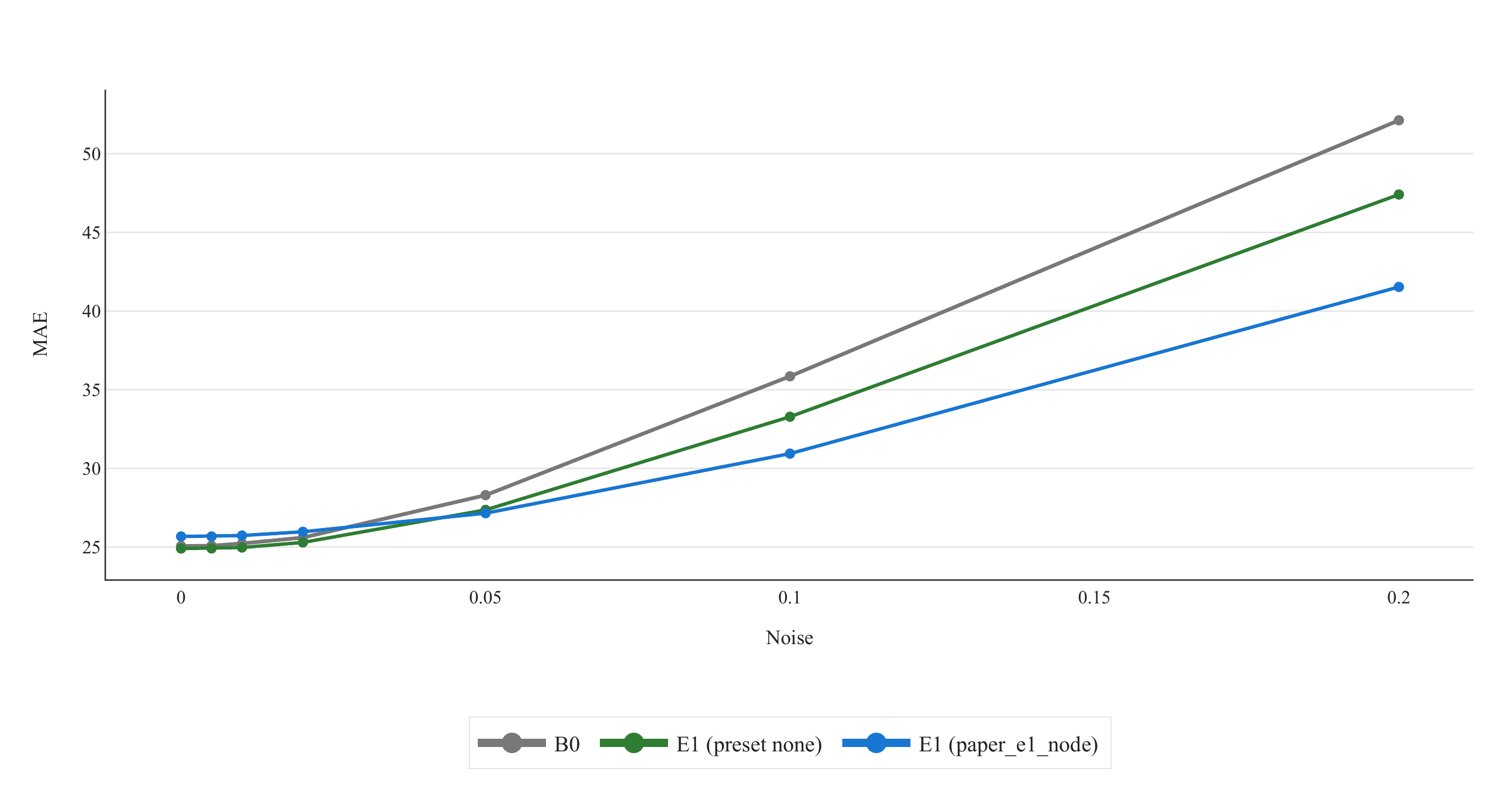}
\caption{\textbf{T5A clean--robust tradeoff.} E1-large dominates large position noise at a clean-MAE cost.}
\label{fig:supp-t5a}
\end{figure}

\subsection{T5B --- BigCloneBench}
\label{app:T5B}

\noindent\textbf{Setup:} Code clone detection under identifier rename (family $A_5$).
Matched arm penalises identifier coordinates; E1S is the signal-aligned wrong partition.\\
\textbf{Protocol:} CodeBERT, 10 epochs, $5$k/$5$k stratified, seed~$42$;
baseline / VAT / identifier-matched / signal-aligned wrong (E1S).\\
\textbf{Takeaway:} Clean balanced accuracy is a coin flip between VAT and matched at one seed
(${\approx}0.7$\,pp).
The rename-ratio gap is the real signal: identifier-matched reaches $0.938$ vs.\
VAT $0.799$ / baseline $0.830$, while signal-aligned E1S falls below baseline.
Charging the named nuisance coordinates matters when they are identified.\\
\textbf{JSON:} \JsonPath{T5/Task5B/runs/Clone\_pilot/pipeline\_summary.json}.

\begin{table}[!htbp]
\centering
\small
\caption{T5B rename retention. Bold = best rename\_bacc\_ratio.}
\label{tab:T5B}
\begin{tabular}{@{}lccc@{}}
\toprule
Arm & Clean bacc$\uparrow$ & Rename bacc$\uparrow$ & Rename ratio$\uparrow$ \\
\midrule
baseline & 0.932 & 0.773 & 0.830 \\
VAT & \textbf{0.938} & 0.750 & 0.799 \\
identifier-matched & 0.931 & \textbf{0.873} & \textbf{0.938} \\
signal-aligned (E1S) & 0.933 & 0.688 & 0.738 \\
\bottomrule
\end{tabular}
\end{table}

\begin{figure}[!htbp]
\centering
\SubmissionFigCompact[width=\linewidth]{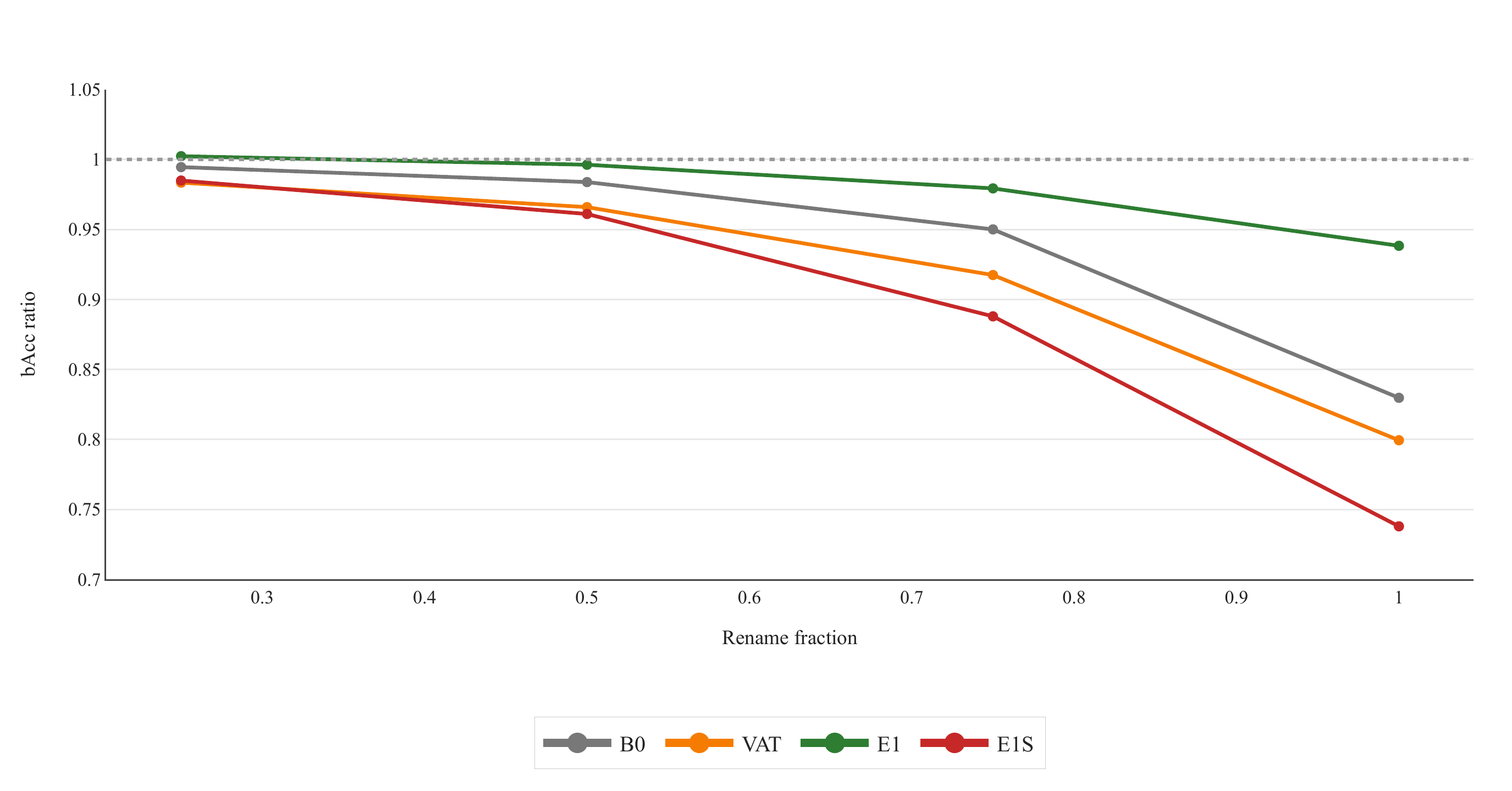}
\caption{\textbf{T5B predicted failure.} Signal-aligned wrong underperforms identifier-matched.}
\label{fig:edges-t5b}
\end{figure}

\subsection{T6A --- Whisper accent}
\label{app:T6A}

\noindent\textbf{Setup:} Whisper ASR under accent-like residual/scatter proxy
(not the full temporal construction).\\
\textbf{Protocol:} Whisper-small, $4$k LibriSpeech, 5 epochs; baseline / matched / wrong /
CE-only re-run (\texttt{accent\_adapted}).
The CE re-run uses the same data pool---a second CE seed, not accent-labelled supervision.\\
\textbf{Note:} CE re-run wins WER ($13.91$); wrong also beats matched on WER
($14.03$ vs.\ $14.63$); only matched substantially lowers temporal TDI
($0.381$ vs.\ wrong $0.644$)---geometry and task disagree
(\S\ref{sec:discussion-soft}).\\
\textbf{JSON:} \JsonPath{T6/task6A/results/wer\_results.json}.

\begin{table}[!htbp]
\centering
\small
\caption{T6A WER and temporal TDI ($\downarrow$). Bold = best WER / best TDI.
Wrong-$W$ beats matched on WER; matched alone wins TDI---geom$\neq$task.
\texttt{accent\_adapted} $=$ CE-only Libri re-run (not accent-CSV oracle).}
\label{tab:T6A}
\begin{tabular}{@{}lcc@{}}
\toprule
Arm & WER (\%) & Temporal TDI \\
\midrule
baseline & 23.26 & 1.096 \\
matched & 14.63 & \textbf{0.381} \\
wrong-$W$ & 14.03 & 0.644 \\
CE re-run (\texttt{accent\_adapted}) & \textbf{13.91} & $1.101$ \\
\bottomrule
\end{tabular}
\end{table}

\begin{figure}[!htbp]
\centering
\SubmissionFigCompact[width=\linewidth]{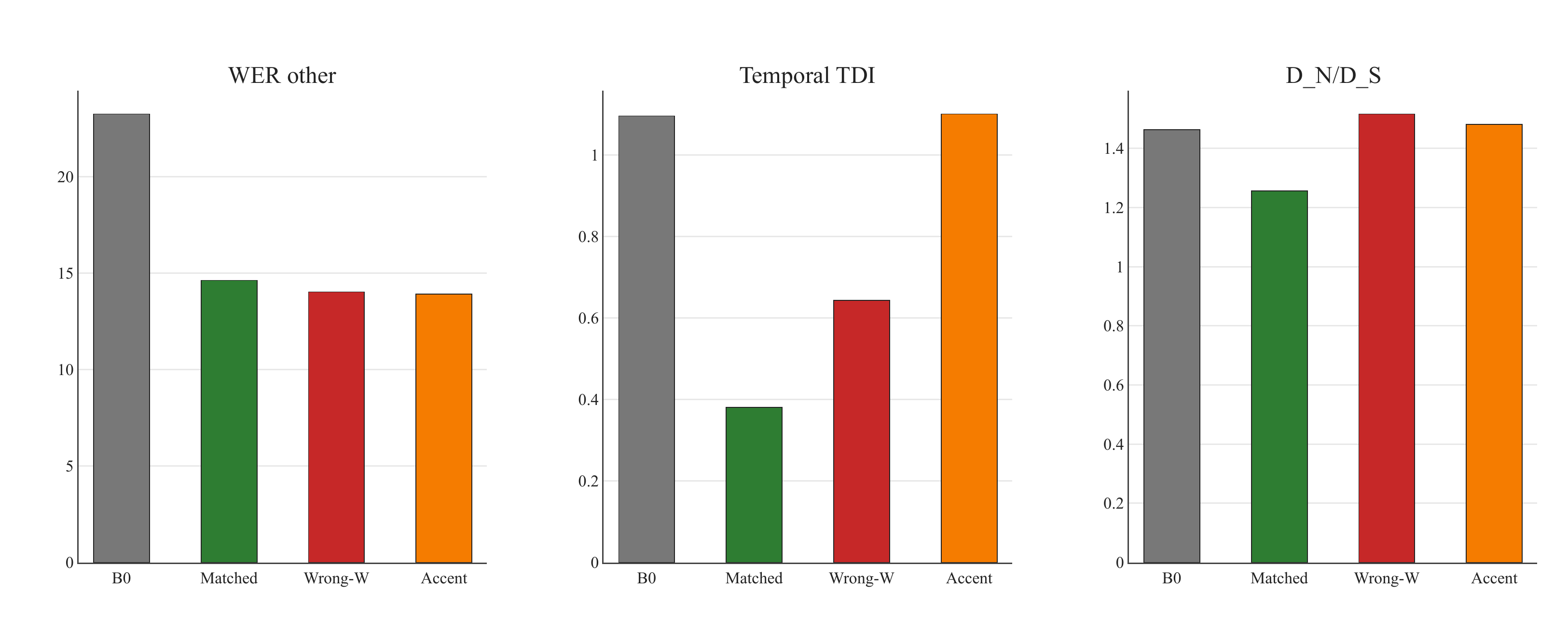}
\caption{\textbf{T6A Whisper.} A CE-only re-run can win WER while leaving temporal TDI near
baseline; only matched moves geometry (geom$\neq$task).}
\label{fig:supp-t6a}
\end{figure}

\subsection{T6B --- UCI HAR}
\label{app:T6B}

\noindent\textbf{Setup:} sensor-scatter proxy (rank 48).\\
\textbf{Protocol:} TCN, UCI HAR ($n_{\mathrm{test}}{=}2947$), 35 epochs, 3 seeds $\{0,1,42\}$;
arms baseline / wrong / even-spread (full-space) / matched.\\
\textbf{JSON:} \JsonPath{T6/task6B/artifacts/multiseed/multiseed\_summary.json}.\\
\textbf{Takeaway:} At stress $L{=}3.0$, matched ($0.410\pm0.029$) ties even-spread
($0.408\pm0.054$) and both beat wrong and baseline.
Soft training does not separate matched from even-spread here; this is a sensor-scatter
proxy and cannot carry a coverage claim
(\S\ref{par:t6b-tension}).\\

\begin{table}[!htbp]
\centering
\small
\caption{T6B balanced accuracy (3-seed mean$\pm$std). Bold = best at $L{=}3.0$
(matched and even-spread agree within noise).}
\label{tab:T6B-stress}
\begin{tabular}{@{}lcccc@{}}
\toprule
Level & B0 & wrong-$W$ & even-spread & matched \\
\midrule
clean & $0.917\pm0.023$ & $0.923\pm0.002$ & $0.925\pm0.008$ & $0.931\pm0.010$ \\
$L{=}1.0$ & $0.772\pm0.032$ & $0.825\pm0.002$ & $0.826\pm0.014$ & $0.843\pm0.013$ \\
$L{=}2.0$ & $0.515\pm0.041$ & $0.630\pm0.014$ & $0.652\pm0.038$ & $0.679\pm0.026$ \\
$L{=}3.0$ & $0.279\pm0.028$ & $0.349\pm0.031$ & $0.408\pm0.054$ & $\mathbf{0.410\pm0.029}$ \\
\bottomrule
\end{tabular}
\end{table}

\begin{figure}[!htbp]
\centering
\SubmissionFigCompact[width=\linewidth]{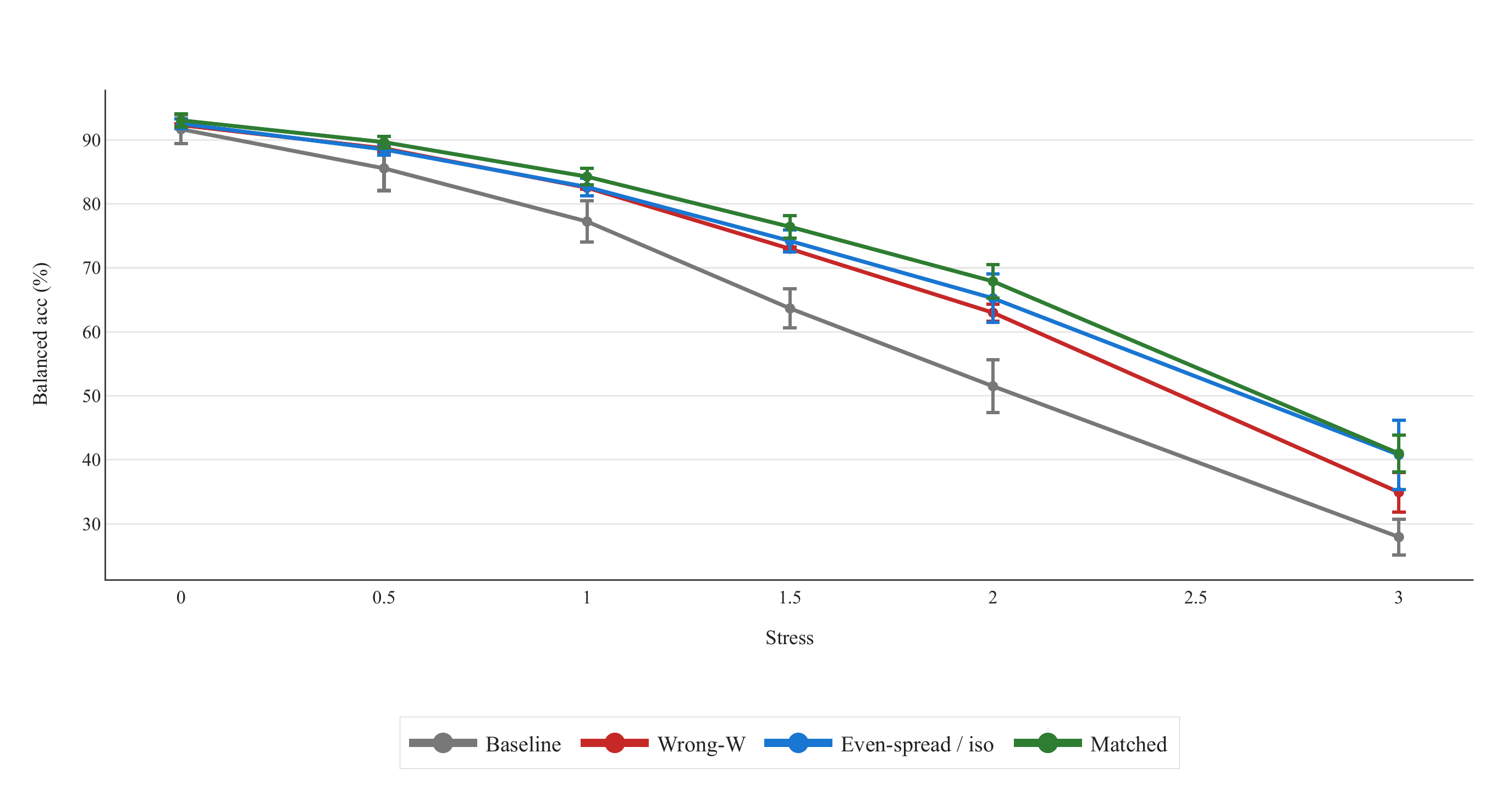}
\caption{\textbf{T6B stress.}
Matched $\approx$ even-spread, both above wrong and baseline---soft training does not
separate direction on this sensor-scatter proxy.}
\label{fig:supp-t6b}
\end{figure}

\subsection{T7A --- Qwen2.5-7B style geometry (pilot)}
\label{app:T7A}

\noindent\textbf{Setup:} Language-model style geometry via a style-Gram proxy (appendix pilot).\\
\textbf{Protocol:} Qwen2.5-7B; reward model 20 epochs; DPO LoRA 1 epoch / 240 pairs;
TruthfulQA $n{=}500$.\\
\textbf{Headline:} Matched sycophancy falls $38.5\to13.5\%$; style TDI matched DPO $1.84$
vs.\ standard $2.41$.\\
\textbf{Note:} Even-spread can beat matched on raw sycophancy / MC1---pilot, not a Tier-2
direction claim.

\begin{figure}[!htbp]
\centering
\SubmissionFigCompact[width=\linewidth]{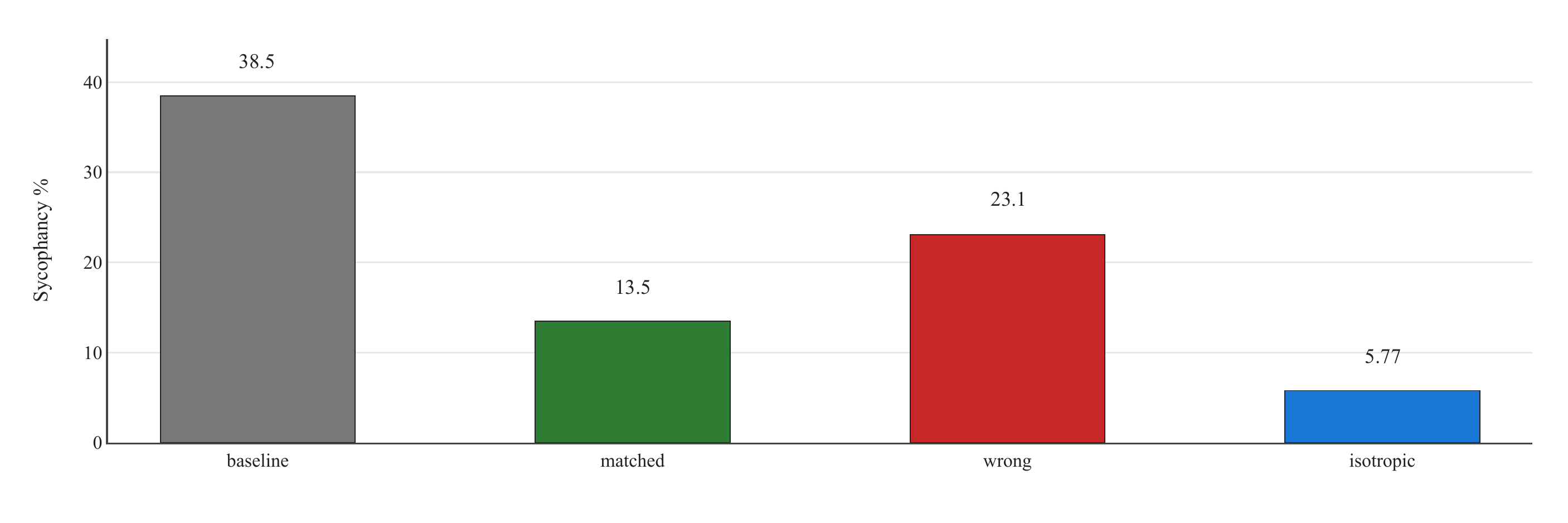}
\caption{\textbf{T7A pilot.} Style-Gram DPO geometry.}
\label{fig:supp-t7a}
\end{figure}

\begin{figure}[!htbp]
\centering
\SubmissionFigCompact[width=\linewidth]{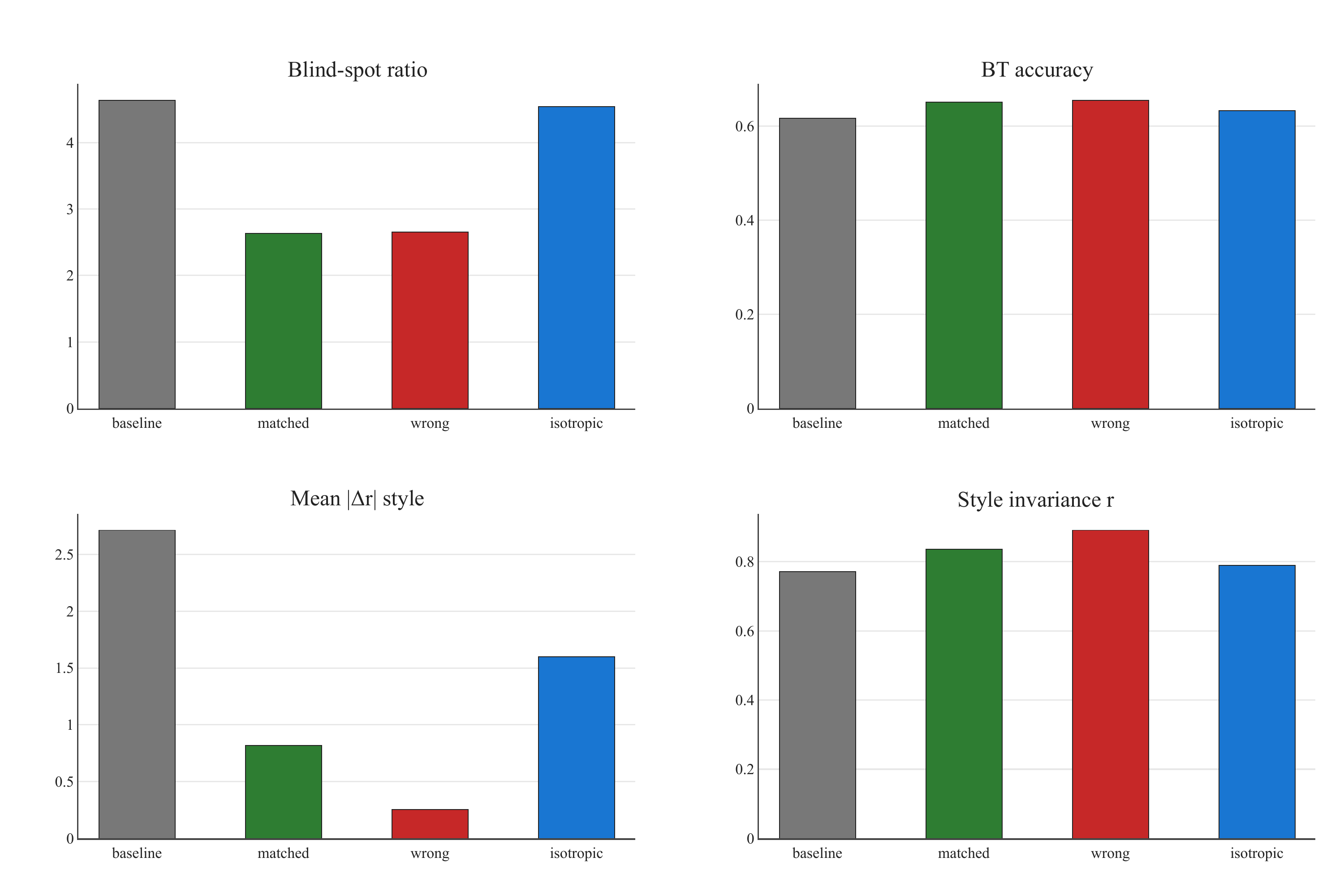}
\caption{\textbf{T7A RM metrics.} Mixed selectivity vs.\ even-spread.}
\label{fig:real-t7a}
\end{figure}

\subsection{T7B --- CIFAR-10 ViT PGD geometry}
\label{app:T7B}

\noindent\textbf{Setup:} CIFAR-10 ViT under an algorithm-induced PGD subspace penalty
(family $A_7$).\\
\textbf{Protocol:} ViT-Small, 75 epochs, batch 1024; seeds $\{7,13,42\}$
(\texttt{revision\_hardened\_2026}).\\
\textbf{Relation to April preprint.}
T7B hardens the earlier CIFAR-10 ViT matched-Jacobian / PGD column (shared noise-recipe
lineage; not the same run).
April reported single-seed B0 clean $69.95\%$ in a stressful from-scratch regime;
canonical numbers here are three-seed means (B0 $\approx 78.0\%$)---do not equate the
scalars.\\
\textbf{Takeaway:} Axes split on real data.
Matched training wins clean accuracy, Gaussian noise at $\sigma{=}0.1$
($73.8\pm1.0\%$ vs.\ PGD-AT $62.2\pm1.1\%$ / baseline $50.2\pm0.6\%$), and class-layout
geometry (\TDI{} $0.88\pm0.05$ vs.\ $1.48\pm0.04$).
PGD-AT wins registered adversarial metrics (subspace drift $0.18\pm0.03$ vs.\
$0.93\pm0.22$; PGD@4 $44.6\pm0.5\%$ vs.\ $16.6\pm6.3\%$, with matched also below
baseline $22.2\pm6.0\%$).
No single registered scalar certifies both stories.\\
\textbf{Wall-clock.}
Same recipe: B0 $\approx\!2.8$\,min; matched median $\approx\!6.2$\,min
(seed~$42$ slowed by contention); PGD-AT $\approx\!21$\,min ($\approx\!3.4\times$ matched);
VAT median $\approx\!11$\,min (seed~$42$ contended).\\
\textbf{Figures:} Figure~\ref{fig:t7b-multiseed-flagship}.\\
\textbf{JSON:} \JsonPath{T7/task7B/results/revision\_hardened\_2026/summary\_ci.json}.
A gradient-coupled arm in that JSON has incomplete geometry metrics and is excluded from
claims.
Wave~B prediction seeds hit the same qualitative split at different absolute levels;
paper numbers stay on $\{7,13,42\}$.

\begin{table}[!htbp]
\centering
\small
\caption{\emph{Audit table.} T7B three-seed aggregate (mean$\pm$CI half-width; seeds $\{7,13,42\}$).
Bold $=$ best on each column.
Opposite rankings on clean/Gaussian/TDI vs.\ registered PGD metrics are the point.}
\label{tab:T7B-subspace-full}
\begin{tabular}{@{}lccccc@{}}
\toprule
Arm & Clean (\%)$\uparrow$ & $\sigma{=}0.1$ (\%)$\uparrow$ & TDI$\downarrow$ & Subspace drift$\downarrow$ & PGD@4 (\%)$\uparrow$ \\
\midrule
baseline & $78.0\pm1.4$ & $50.2\pm0.6$ & $1.11\pm0.07$ & $5.02\pm1.09$ & $22.2\pm6.0$ \\
VAT & $77.2\pm1.8$ & $57.6\pm0.9$ & $1.37\pm0.04$ & $2.29\pm0.34$ & $11.7\pm0.8$ \\
even-spread & $78.4\pm0.1$ & $73.6\pm0.7$ & $0.97\pm0.03$ & $2.31\pm0.17$ & $20.3\pm2.6$ \\
matched & $\mathbf{81.0\pm1.2}$ & $\mathbf{73.8\pm1.0}$ & $\mathbf{0.88\pm0.05}$ & $0.93\pm0.22$ & $16.6\pm6.3$ \\
PGD-AT & $64.7\pm0.7$ & $62.2\pm1.1$ & $1.48\pm0.04$ & $\mathbf{0.18\pm0.03}$ & $\mathbf{44.6\pm0.5}$ \\
\bottomrule
\end{tabular}
\end{table}

\begin{figure}[!htbp]
\centering
\SubmissionFigCompact[width=\linewidth]{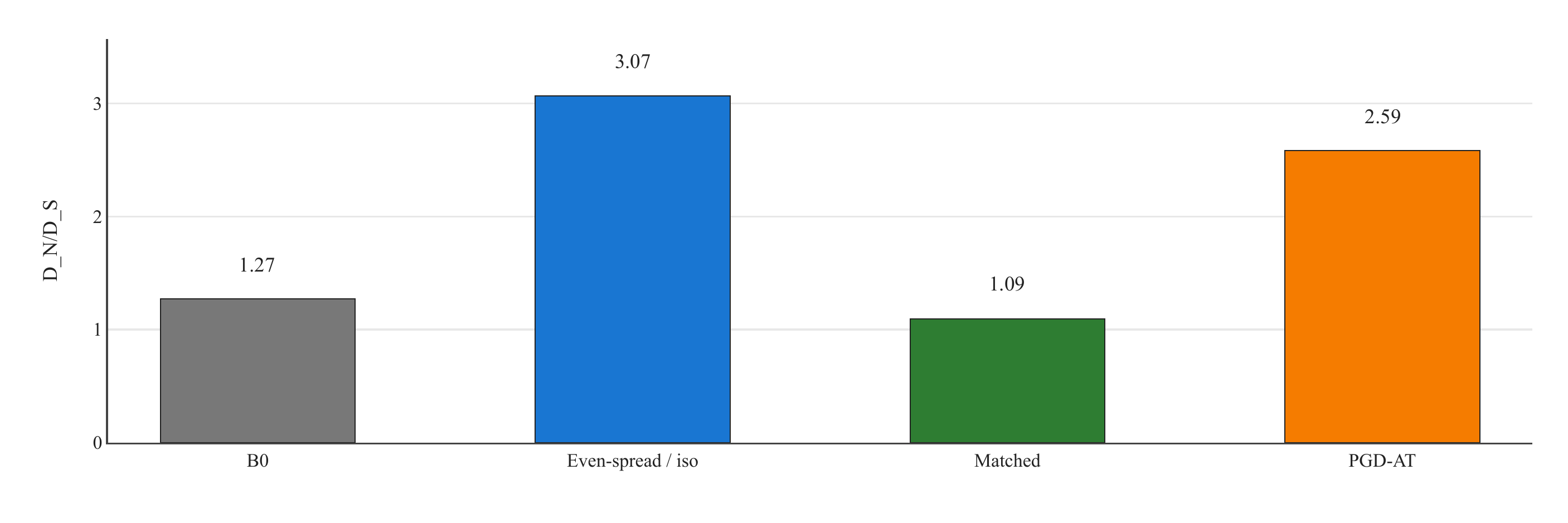}
\caption{\textbf{T7B secondary geometry ($D_N/D_S$).}
Three-seed means from \texttt{summary\_ci.json}; lower is better.
Not the clean/Gaussian/TDI vs.\ PGD-drift dissociation (that is
Figure~\ref{fig:t7b-multiseed-flagship}).}
\label{fig:edges-t7b}
\end{figure}

%% file: appendix/moved_from_main.tex

\input{sections/fragments/t1_estimated_w_numbers}
\input{sections/fragments/t4b_multiseed_numbers}
\section{Material moved from the main text}
\label{app:moved-from-main}

Screens that support the main-text spine but are not needed to follow the argument
(oracle T1 ordering $\to$ T7B field check $\to$ checklist).
\textbf{Canonical} soft-$\lambda$ and estimated-$\hat W$ write-ups live here; the
per-task cards in \S\ref{app:tasks-full} point back to these subsections.
Main text: \S\ref{sec:empirical}.

\subsection{Estimator families and eigengap pre-flight}
\label{app:Ak-tables}

\begin{table}[t]
\centering
\caption{What each experiment names vs.\ what training actually penalises.}
\label{tab:Ak-summary}
\label{tab:nuisance-families}
\small
\setlength{\tabcolsep}{2.5pt}
\begin{tabularx}{\linewidth}{@{}lcc>{\raggedright\arraybackslash}X>{\raggedright\arraybackslash}Xl@{}}
\toprule
Family & $A_k$ & D & Named construction & Penalised in experiments & Blocks \\
\midrule
Subspace & 1 & 1 & $n{=}W\eta$; top-$r$ left singular subsp.
  & Identified subsp.\ (oracle~$W$ hard-proj) & T1 \\
Even-spread & 2 & 2 & $\mathcal{N}(0,\sigma^2 I)$; $\hat\sigma^2 I$
  & As named (even-spread / isotropic) & T2A, T2B \\
Photometric & 3 & 3 & aug.\ modes $\beta_k$; mean Gram
  & Constr./aug.\ modes (shared orbit) & T3A, T3B \\
Domain & 4 & 4 & paired $x_T{-}x_S$; paired-diff.\ moment
  & Rep.-Gram proxy (not constr.~D4) & T4A/B \\
Compositional & 5 & 5 & coord.\ spike on $\mathcal{V}_n$; block cov.
  & Coord.\ / construction~D5 & T5A, T5B \\
Temporal & 6 & 6 & deployment-like $\Delta h$; incr.\ cov.
  & Residual/scatter proxy (not D6 estimator) & T6A/B \\
Adv.\ / align. & 7 & 7 & PGD $\hat\delta^\star$ or style; Gram
  & T7A: style-Gram proxy; T7B: PGD-delta Gram & T7A, T7B \\
\bottomrule
\end{tabularx}
\end{table}

\begin{table}[t]
\centering
\caption{Heuristic map from deployment symptom to candidate family $A_k$.}
\label{tab:Ak-choice}
\small
\begin{tabularx}{\linewidth}{@{}>{\raggedright\arraybackslash}X c c@{}}
\toprule
Deployment symptom & $A_k$ & Block \\
\midrule
Known / estimable low-rank subspace & $A_1$ & T1 \\
No preferred direction (sensor noise) & $A_2$ & T2 \\
Finite aug.\ / photometric modes & $A_3$ & T3 \\
Paired multi-domain shift & $A_4$ & T4 \\
Nuisance in named coordinates & $A_5$ & T5 \\
Label-constant temporal / sensor drift & $A_6$ & T6 \\
Learned deltas (PGD) or style rewrites & $A_7$ & T7 \\
\bottomrule
\end{tabularx}
\end{table}

If the estimate is low-rank, report the eigengap $\gamma_r=\lambda_r/\lambda_{r+1}$.
When $\gamma_r\approx 1$ there is no clear subspace to cover---treat that as a predicted
failure (Office-31), not as a counterexample to matching.
Figure~\ref{fig:eigengap-vs-matched} cross-checks this pre-flight on frozen blocks:
Office-31 ($\gamma_r{\approx}1.00$) shows the expected miss; elsewhere the gap is a
distrust signal for $\hat W$, not a monotone predictor of gain size.

\begin{figure}[t]
\centering
\SubmissionFig[width=0.72\linewidth]{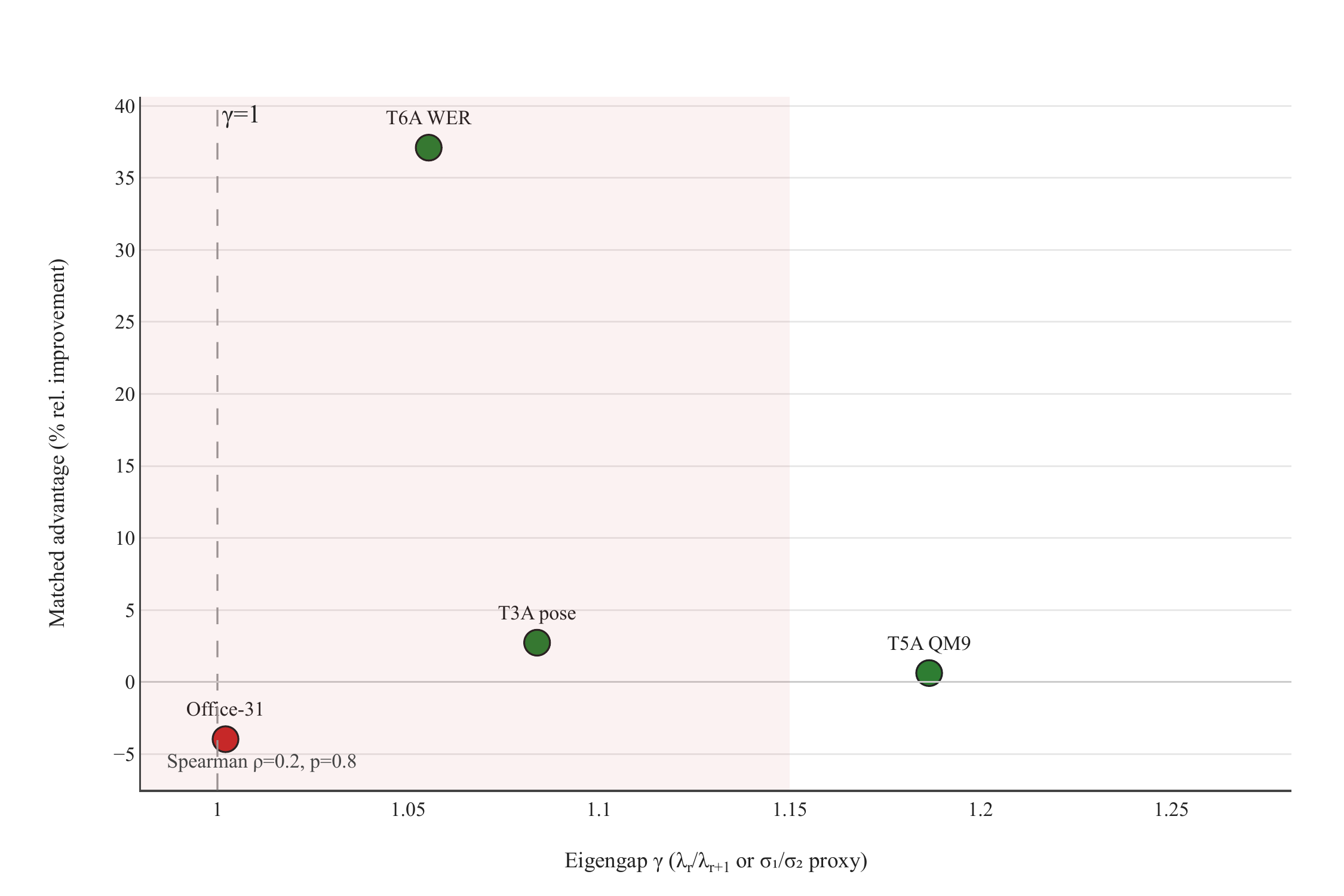}
\caption{Eigengap pre-flight vs.\ matched advantage on each block's primary task
metric (frozen JSON; $\sigma_1/\sigma_2$ proxy when rank-$r$ gap is not exported).}
\label{fig:eigengap-vs-matched}
\end{figure}

\subsection{Estimated-$\hat W$ companion (T1)}
\label{sec:T1-estimated}

Does the ordering survive when $W$ is estimated rather than handed to the trainer?
On the same Fashion-MNIST spike, hard projection with $\hat W$ from SVD of cluttered
deltas still separates matched from even-spread
(matched OOD \TOneEstMatched\TOneEstMatchedStd{} vs.\ even-spread
\TOneEstIso\TOneEstIsoStd{}; matched$-$even-spread \TOneEstMatchedMinusIso\,pp;
paired $t$ $p{\approx}\TOneEstPairedP$; mean principal cosine $\TOneEstMeanCos$).
Easy estimation repeats the coverage control; it does not close Cityscapes or ImageNet
orientation.

\subsection{Soft$\star$: trainable soft regime and $\lambda$ sweep}
\label{app:soft-detail}

Soft$\star$ (main Table~\ref{tab:open}): at recipe penalty weight, soft training does not
recover the oracle hard-projection ordering (main-text \S\ref{sec:soft-drift}).
On Fashion-MNIST (T1), matched and wrong penalise a rank-$16$ subspace while even-spread
penalises everything; the three arms are near-parity at $\lambda{=}2$.
Raising $\lambda$ uncapped pulls matched drift below even-spread only for
$\lambda\gtrsim 32$, while wrong stays competitive through $\lambda{=}128$
(Figure~\ref{fig:soft-lambda-sweep})---still short of hard projection.
\begin{itemize}[leftmargin=1.4em,itemsep=2pt,parsep=0pt,topsep=2pt]
  \item \textbf{Fashion-MNIST (T1 soft).}
    OOD matched $64.6\pm2.6\%$ vs.\ even-spread $61.6\%$ / wrong $61.8\%$;
    $D_W$ ratios $\approx 0.210$ / $0.201$ / $0.189$ $\times$ baseline.
  \item \textbf{NYU Depth (T3B).}
    AbsRel matched $0.215$ / wrong $0.211$ / B0 $0.242$---all three share the same
    photometric orbit, so direction barely matters.
  \item \textbf{Whisper (T6A).}
    WER matched $14.63$ vs.\ B0 $23.26$; wrong also beats matched on WER.
  \item \textbf{UCI HAR (T6B).}
    Balanced accuracy matched $0.410\pm0.029$ $\approx$ even-spread $0.408\pm0.054$
    $\succ$ wrong $0.349\pm0.031$ $\succ$ B0 $0.279\pm0.028$.
\end{itemize}
\noindent\textbf{CodeBERT rename control (T5B).}
\label{par:t5b-main}
Identifier-matched training raises rename ratio from $0.830$ to $0.938$; a
signal-aligned wrong partition falls to $0.738$.
On the identifier-rename pilot, VAT edges clean balanced accuracy by only
${\approx}0.7$\,pp over E1 at a single seed ($0.938$ vs.\ $0.931$; $n{=}1$)---noise-level
on that axis---while E1's rename-ratio gap is large ($0.938$ vs.\ VAT $0.799$ /
B0 $0.830$).

\begin{figure}[htbp]
\centering
\SubmissionFig{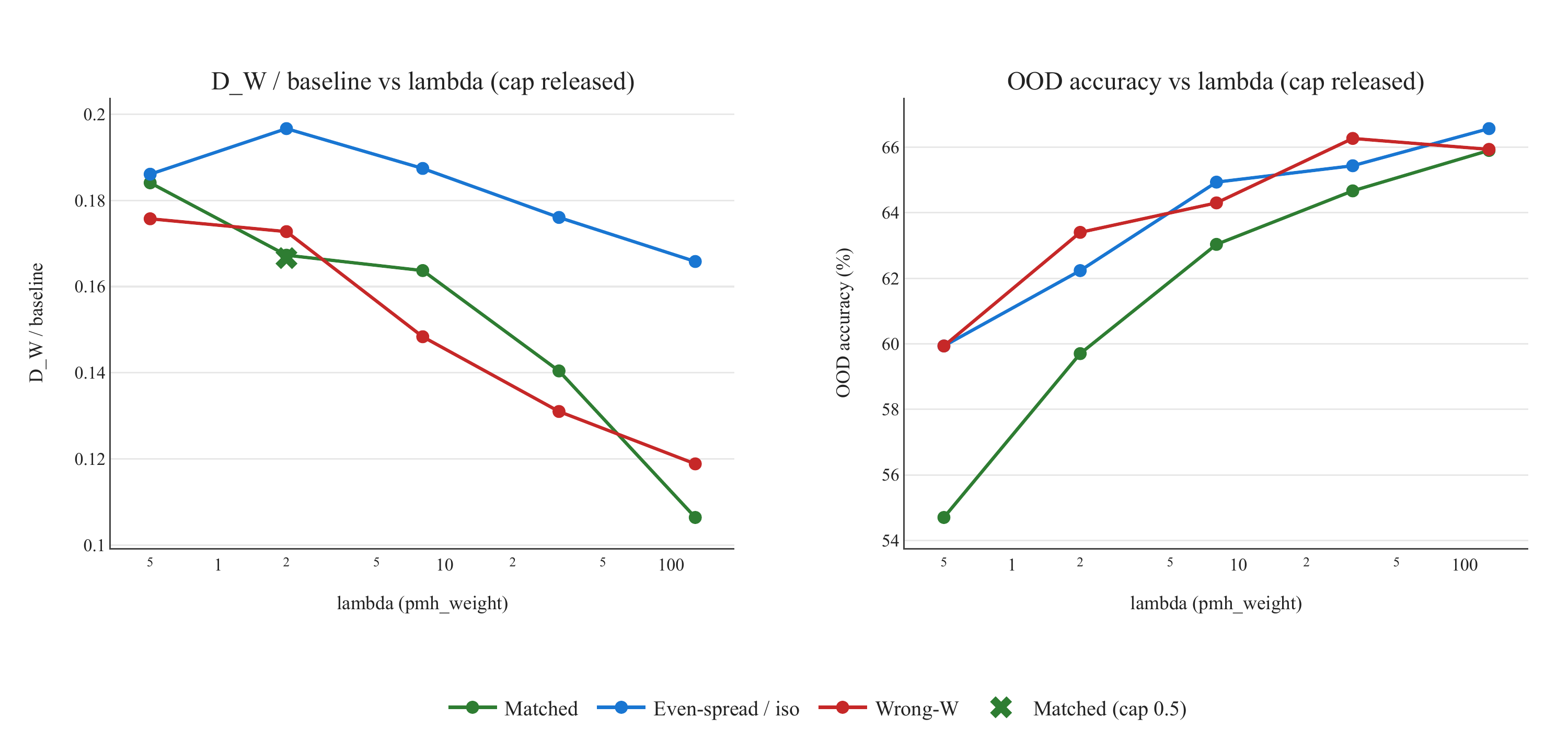}
\caption{\textbf{Soft weight is not oracle weight.}
As $\lambda$ grows, matched drift pulls below even-spread, but wrong stays competitive
through $\lambda{=}128$; recipe $\lambda{=}2$ never reaches the hard-projection regime.}
\label{fig:soft-lambda-sweep}
\end{figure}

\paragraph{Projector-soft probe (Fashion-MNIST).}
\label{sec:projector-soft}
See the protocol card in \S\ref{app:projector-soft} (column finite-difference vs.\
feature-MSE on the same Fashion spike).

\subsection{Real directional proxies}
\label{sec:emp-orient-boundary}
\label{sec:T2A}
\label{sec:emp-test2}
\label{sec:empirical-passes}
\label{sec:T4}

\paragraph{ImageNet orientation near-null.}
On gated ILSVRC (and ImageNet-100), matched-vs-wrong orientation stress is near-null
($\Delta{=}{+}0.45$\,pp at $\sigma{=}0.10$ on gated ILSVRC; public-100 $\Delta{=}{-}0.3$\,pp).
Direction-matching does not help when the estimated anisotropy is weak.

\paragraph{ImageNet-C scale check (T2A).}
When $\Sigmatask\propto I$, matched and even-spread coincide.
ImageNet-C severity~$3$ improves by $+4.3$\,pp with clean preserved
($1.2\times$ task-only epochs)---a scale screen, not a coverage headline.

\paragraph{Cityscapes rare-5 proxy (T4B).}
On Cityscapes rare-5, multi-scale soft lift vs.\ even-spread is directional but underpowered
($n{=}3$, Welch $p{\approx}0.07$; Figure~\ref{fig:p2-right-object}).
The aggregate improves while motorcycle IoU does not
($9.0\pm8.6$ vs.\ B0 $10.9\pm2.6$).
Treat this as a weak proxy signal, not evidence that coverage worked.
Office-31: see \S\ref{app:T1} (weak-eigengap failure mode).

\begin{figure}[htbp]
\centering
\SubmissionFig{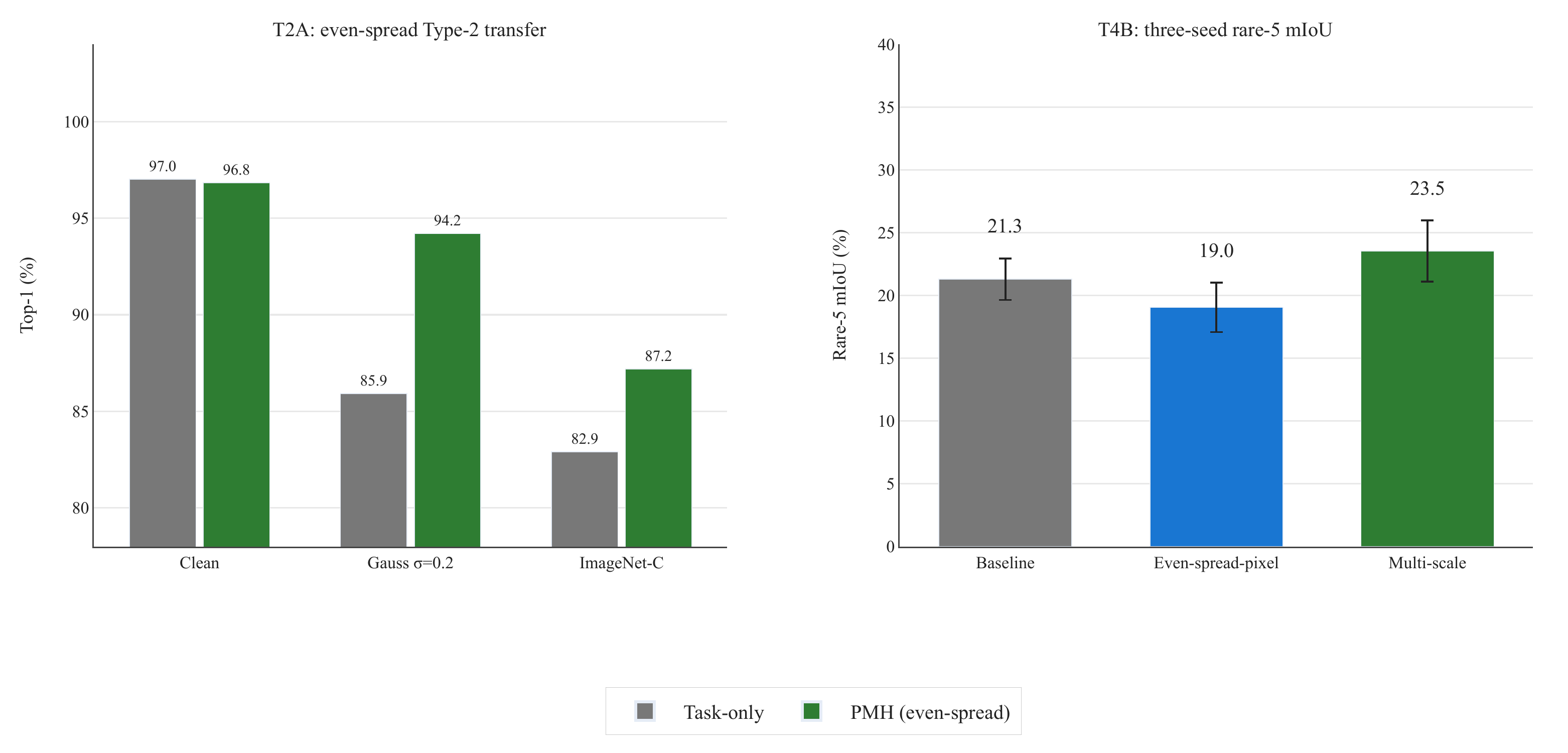}
\caption{\textbf{The penalised object matters.}
Right: T4B rare-5 mIoU (mean$\pm$SD; underpowered proxy).
Left: even-spread Type-2 control when $\Sigmatask\propto I$.}
\label{fig:p2-right-object}
\end{figure}

\subsection{Full Tier-3 competitor table}
\label{app:tier3-full}

\begin{table}[htbp]
  \centering
  \caption{Tier~3 (full, with protocol footnotes): matched Jacobian penalty vs.\ field
  alternatives on task metrics.
  Verdict is on the stated task axis only.
  The two T7B/PGD-AT rows are one comparison on two axes
  (Corollary~\ref{cor:nfl-noise}).
  Main text Table~\ref{tab:tier3-competitors} carries the same five rows without footnotes.}
  \label{tab:tier3-competitors-full}
  \small
  \setlength{\tabcolsep}{3pt}
  \begin{tabular}{@{}lp{2.1cm}p{2.4cm}p{2.6cm}l@{}}
    \toprule
    Block & Competitor & Matched (task) & Competitor (task) & Verdict \\
    \midrule
    T7B & PGD-AT & clean $81.0$; $\sigma{=}0.1$ $73.8$ &
      clean $64.7$; $\sigma{=}0.1$ $62.2$ &
      \textbf{win} (joint clean$+$generic; PGD pays ${-}13$\,pp clean) \\
    T7B & PGD-AT & PGD@4 $16.6$; drift $0.93$ &
      PGD@4 $44.6$; drift $0.18$ &
      lose (registered PGD; ${\approx}3.4\times$ wall-clock) \\
    T2B & VAT & worst-shift $69.2$; clean $89.1$ &
      worst-shift $81.7$; clean $92.6$ &
      lose (non-matched head-to-head) \\
    T7A & iso & syc.\ $13.5$; MC1 $0.55$ &
      syc.\ $5.8$; MC1 $0.65$ &
      \textbf{lose} (pilot) \\
    T6A & CE re-run & WER $14.63$ & WER $13.91$ &
      lose (task, WER); win (geom.\ \TDI{}) \\
    \bottomrule
  \end{tabular}\\[0.35em]
  \raggedright\scriptsize
  T2B protocol: same ResNet-18 / 30 epochs / Adam $10^{-4}$ / batch~$32$ / seed~$42$;
  VAT Miyato-style ($\varepsilon{=}2.0$, $\xi{=}10^{-6}$, weight~$1.0$, $1$ power iter) vs.\
  E1 even-spread multi-scale feature-MSE $+$ pixel Gaussian/intensity.
  E1 still leads L4 drift ($9.7$ vs.\ VAT $13.4$).
  T6A: matched alone cuts temporal \TDI{} ($0.38$ vs.\ accent-adapted $1.10$).
\end{table}

\subsection{Geometry probe (TDI) detail}
\label{app:tdi-detail}
\label{def:tdi}

When the \emph{deployment} law is roughly isotropic ($\Sigmatask\approx\sigma^2 I$) we use trajectory
\TDI{} (lower is better).
When an estimated subspace $\hat W$ is available we use $D_N/D_S$.
For $\delta\sim\mathcal{N}(0,\sigma^2 I)$ and layer maps $\phi^{(1:\ell)}$,
\begin{equation}
\mathrm{TDI}(\phi,\sigma)
\;:=\;
\frac{1}{L}\sum_{\ell=1}^{L}
\frac{\E_{x,\delta}\!\left[\big\|\phi^{(1:\ell)}(x+\delta)-\phi^{(1:\ell)}(x)\big\|_2^2\right]}
     {\E_{x}\!\left[\big\|\phi^{(1:\ell)}(x)\big\|_2^2\right]}.
\label{eq:tdi-trajectory}
\end{equation}
On T2A, \TDI{}@$\sigma{=}0.10$ falls $58\%$ under even-spread \PMH{} while $\|J\|_F$ falls
only $9.1\%$.

\subsection{Why methods work (visual)}
\label{app:why-methods-fig}

\begin{figure}[htbp]
\centering
\SubmissionFig{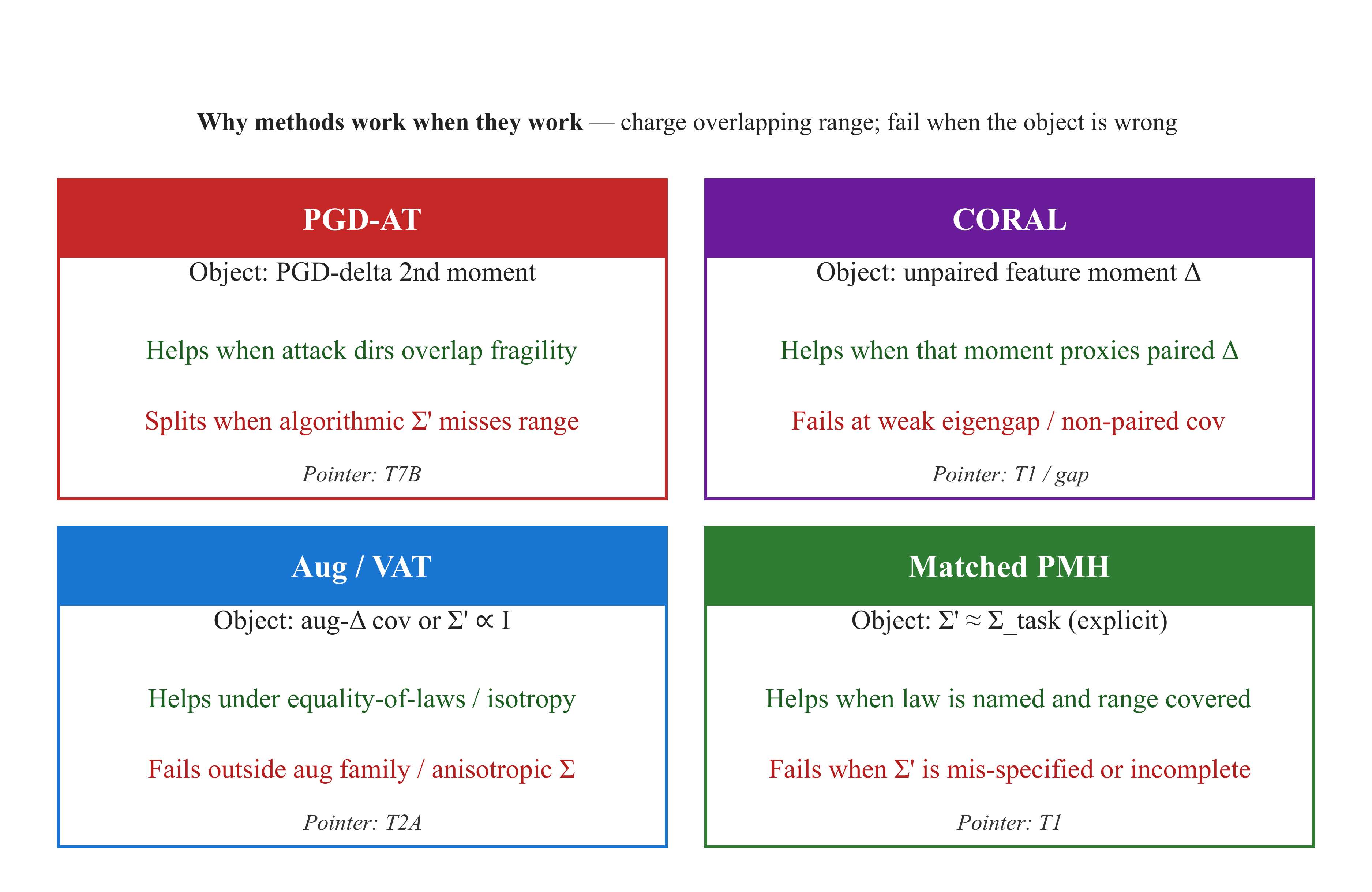}
\caption{\textbf{Why methods work when they work.}
Familiar methods penalise related second-moment objects; lasting help requires covering
the directions that actually move at deployment.}
\label{fig:why-methods-work}
\end{figure}

%% file: appendix/proofs.tex
\section{Supporting proofs}
\label{app:proofs-online}

Skip this section unless you need the foundation calculus.
Core theorems and license statements live in the main manuscript appendix
(Appendix~\ref{app:proofs}; Appendix~\ref{app:conditional-licenses}).
Protocols and frozen numbers are earlier: \S\ref{app:tasks-full}--\S\ref{app:moved-from-main}.
What follows is supporting material only: foundation lemmas, secondary calculus, and
the D1/D2/D7 coupling catalogue---not required for the paper's empirical claims.

\paragraph{Contents.}
\begin{enumerate}[leftmargin=*, itemsep=2pt]
\item Foundation (\ref{app:foundation}): blind spot, even-spread, PGD limits, Theorem~\ref{thm:p1-1}.
\item Secondary calculus (\ref{app:secondary-calculus}): allocation and mismatch asymptotics
  behind Corollary~\ref{cor:rate-separation} (skippable for the main spine).
\item D1/D2/D7 and constructions D3--D6 (\ref{app:D-lemmas}): candidate couplings under named
  assumptions.
\end{enumerate}

\input{appendix/foundation_proofs}
\input{appendix/secondary_calculus}

\subsection{D1, D2, D7 and constructions D3--D6}
\label{app:D-lemmas}

The rate lemmas below supply candidate displacement laws for the experimental families.
Families $A_3$--$A_6$ are summarised as constructions in Table~\ref{tab:D3-D6}.
\input{appendix/d_lemmas}

%% file: appendix/foundation_proofs.tex
\subsection{Geometric foundation: blind spot, isotropy, PGD, and training cap}
\label{app:foundation}

Supporting geometry for the task-only floor and related limits.
Theorem~\ref{thm:p1-1} is proved here; Proposition~\ref{prop:pmh-cap} is in
Appendix~\ref{app:proofs}.
The remarks that follow are useful context, not required for the Core Theory claims.

\begin{remark}[Gaussian predictive coordinate with label-preserving deployment corruption]
\label{rem:gaussian-blindspot}
$S\sim\mathcal{N}(0,I_{d_s})$, $N\sim\mathcal{N}(0,I_{d_n})$, $S\perp N$,
$Y=\langle w_s,S\rangle+\rho\langle w_n,N\rangle+\varepsilon$ with
$\varepsilon\sim\mathcal{N}(0,\sigma_\varepsilon^2)$,
$\|w_s\|_2=\|w_n\|_2=1$, $\rho>0$, and training input $X=(S,N)$.
Deployment applies an independent post-measurement corruption
$X^+=(S,N+\Delta)$ with $\Delta\sim\mathcal{N}(0,\sigma^2I_{d_n})$ while holding
$Y^+=Y$.  Thus $N$ is a predictive training coordinate and $\Delta$ is a distinct,
structurally label-preserving deployment displacement.  Write
$J_\phi=[J_{\phi,S}\mid J_{\phi,N}]$.
\end{remark}

\begin{remark}[Factorisation scale]
\label{rem:factorisation-scale}
See Proposition~\ref{prop:gauge} in the main text: the predictor $f=h\circ\phi$ alone does not
identify encoder Jacobian energy under the rescaling $\phi\mapsto c\phi$, $h\mapsto h/c$.
\end{remark}

\begin{lemma}[Sub-block inequality]
\label{lem:subblock}
For $A\in\mathbb{R}^{m\times d}$ and unit $v\in\mathbb{R}^d$: $\|Av\|_2^2\leq\|A\|_F^2$.
For $A=[A_1\mid A_2]$: $\|A_2 v\|_2^2\leq\|A_2\|_F^2\leq\|A\|_F^2$.
\end{lemma}
\begin{proof}
Cauchy--Schwarz on $Av=\sum_j v_j A e_j$ with $\sum_j v_j^2=1$.
\end{proof}

\paragraph{Linearised vs.\ exact drift.}
Identical to Lemma~\ref{lem:lindrift} in the main manuscript appendix
(Appendix~\ref{app:proofs}); not restated here.

\begin{lemma}[The Bayes regressor encodes the predictive direction]
\label{lem:erm-encoding}
Under Remark~\ref{rem:gaussian-blindspot}, the population-MSE Bayes regressor
$f^*=h\circ\phi$ satisfies
$\partial_{w_n} f^*(x)=\rho$ almost everywhere for
$\partial_{w_n}:=\langle w_n,\nabla_N\rangle$.
Hence $J_{\phi,N}$ is non-zero on a set of positive measure.
\end{lemma}
\begin{proof}
Bayes predictor $f^*(x)=\langle w_s,s\rangle+\rho\langle w_n,n\rangle$.
Direct differentiation gives $\partial_{w_n}f^*=\rho$.
\end{proof}

\paragraph{Proof of Theorem~\ref{thm:p1-1}.}
By Lemma~\ref{lem:erm-encoding} and the chain rule,
$\rho=\E[\nabla_\phi h^\top J_{\phi,N} w_n]\leq L\,\E[\|J_{\phi,N} w_n\|_2]$.
Jensen gives $\E[\|J_{\phi,N} w_n\|_2^2]\geq\rho^2/L^2$.
Lemma~\ref{lem:subblock} yields
$\E[\|J_{\phi,N}\|_F^2]\geq\E[\|J_{\phi,N} w_n\|_2^2]\geq\rho^2/L^2$, hence
$\tilde D_Q(\phi,\sigma)=\sigma^2\E[\|J_{\phi,N}\|_F^2]\geq\sigma^2\rho^2/L^2$.
$\square$

\paragraph{Proof of the directional-encoding remark.}
By the chain rule,
$|\E[\partial_u f^*(X)]|
\leq L\,\E\|J_{\phi^*}(X)u\|_2$.
The assumed lower bound and Jensen's inequality give
$\E\|J_{\phi^*}(X)u\|_2^2\geq c_u^2/L^2$.
$\square$

\paragraph{Proof of the isotropy-uniqueness remark.}
For $\Cov(\delta)=\Sigma_\delta$,
$\E_\delta[\|J_\phi\delta\|_2^2]=\Tr(J_\phi^\top J_\phi\Sigma_\delta)$.
Equality with $\sigma^2\Tr(J_\phi^\top J_\phi)$ for all $J_\phi$ holds iff $\Sigma_\delta=\sigma^2 I$
(test all one-row matrices $J=x^\top$, which gives
$x^\top\Sigma_\delta x=\sigma^2\|x\|_2^2$ for every $x$).
$\square$

\paragraph{Proof of PGD-subspace incompleteness (Corollary~\ref{cor:nfl-noise} specialisation).}
Let $S$ be the strict controlled subspace and choose a unit $q\in S^\perp$.  For any scalar $a$, the
one-row Jacobian $J=a q^\top$ satisfies $Js=0$ for every $s\in S$, while
$\|J\|_F=|a|$ can be arbitrarily large.  Thus control on $S$ alone cannot imply isotropic Frobenius
control.  PGD directions may span such a strict subspace; whether they do is algorithm- and
data-dependent.
$\square$

\paragraph{Proof of Proposition~\ref{prop:pmh-cap}.}
See the main manuscript appendix (Appendix~\ref{app:proofs}).
$\square$

%% file: appendix/secondary_calculus.tex
\subsection{Secondary calculus: allocation, projectors, local costs}
\label{app:secondary-calculus}

\paragraph{Supplementary formalism.}
The remarks below are bookkeeping inside an already covered range, plus the algebra behind
main-text Corollary~\ref{cor:rate-separation}.
If you only need the full-support vs.\ missing-support split, take that corollary as stated
and skip this subsection.

\begin{remark}[Trace-constrained allocation in a shared eigenbasis]
\label{rem:allocation}
\label{prop:allocation}
Suppose $A$ and $\Sigmatask$ commute, with shared-basis eigenvalues $\mu_i$ and $\tau_i$,
and write $a_i=\tau_i\tilde v_i^2$.
Let $\lambda>0$ and $B>0$.
If at least one $a_i>0$, then under $\mu_i\geq0$ and $\sum_i\mu_i=B$, the minimiser of
$\sum_i a_i/(1+2\lambda\mu_i)^2$ is
\[
\mu_i^*=\frac{(c\,a_i^{1/3}-1)_+}{2\lambda},
\qquad
\sum_i(c\,a_i^{1/3}-1)_+=2\lambda B,
\]
with the unique $c>0$ satisfying the budget equation.
If all $a_i=0$, every feasible allocation is optimal.
KKT stationarity on $f(\mu)=\sum_i a_i(1+2\lambda\mu_i)^{-2}$ with trace multiplier $\nu$ yields
$1+2\lambda\mu_i=(4\lambda a_i/\nu)^{1/3}$ on active coordinates; inactive coordinates satisfy
$c a_i^{1/3}\leq 1$ with $c=(4\lambda/\nu)^{1/3}$.
\end{remark}

\begin{figure}[!htbp]
\centering
\SubmissionFig[width=0.82\linewidth]{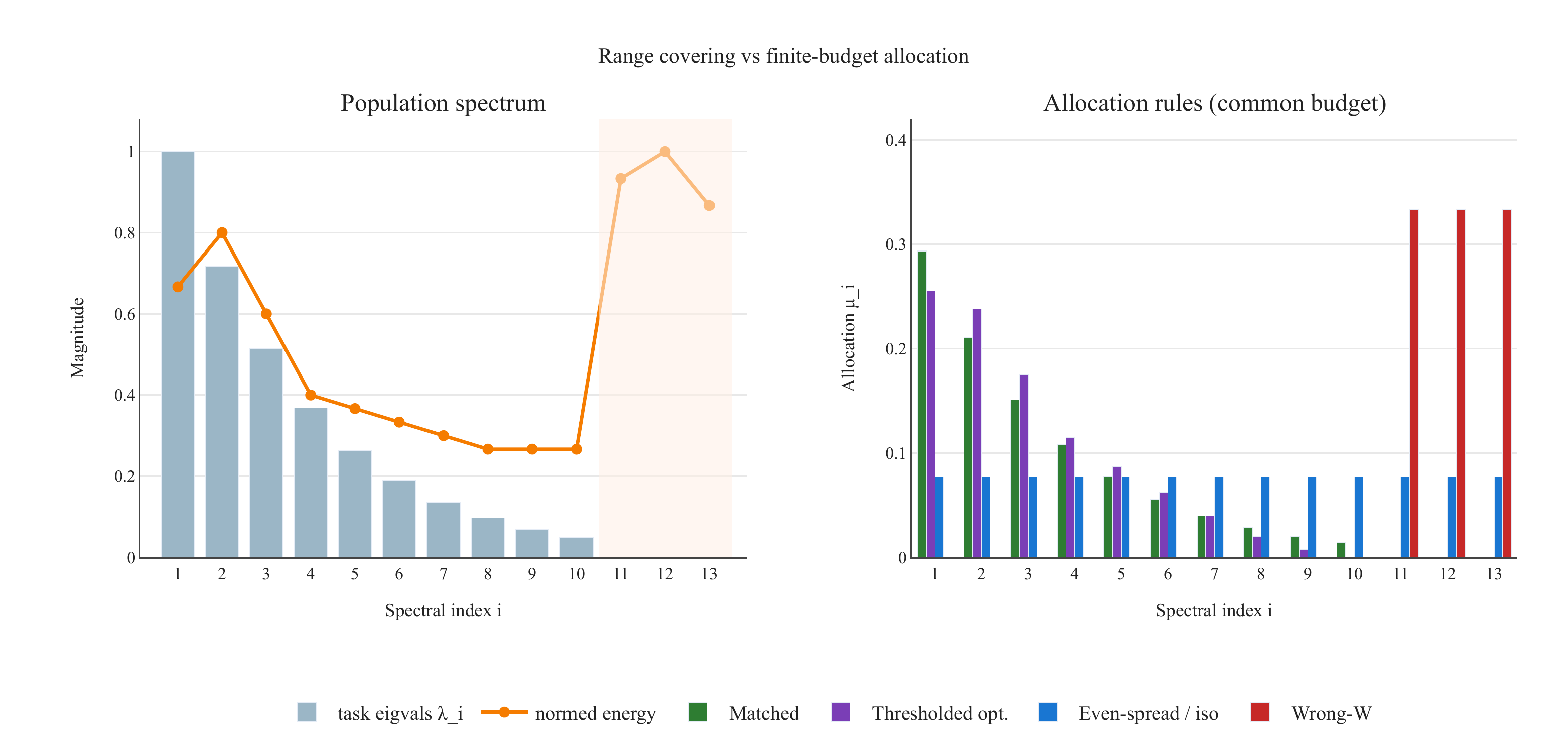}
\caption{\textbf{Remark~\ref{rem:allocation} (secondary).}
Proportional $\Sigma'\propto\Sigmatask$ covers the range but need not be trace-optimal when
task weights are uneven across eigenmodes.}
\label{fig:waterfilling}
\end{figure}

\begin{remark}[Mismatch asymptotics in a shared eigenbasis]
\label{rem:mismatch-asymp}
\label{thm:B}
Under Remark~\ref{rem:allocation} with $B>0$ and at least one $a_i>0$, write
$S=\sum_{i:a_i>0}a_i^{1/3}$ and $m=|\{i:a_i>0\}|$.
As $\lambda\to\infty$, for all large $\lambda$,
$f_\lambda(\mu^*)=S^3/(2\lambda B+m)^2$.
For any fixed feasible allocation with $\mu_i>0$ whenever $a_i>0$,
\[
f_\lambda(\mu)-f_\lambda(\mu^*)
=\frac{1}{4\lambda^2}
\Bigl[\sum_{i:a_i>0}\frac{a_i}{\mu_i^2}-\frac{S^3}{B^2}\Bigr]
+O(\lambda^{-3}),
\]
hence $f_\lambda(\mu)=\frac{1}{4\lambda^2}\sum_{a_i>0}a_i/\mu_i^2+O(\lambda^{-3})$
(Type~A / full support: every such penalty is $\Theta(\lambda^{-2})$, and ratios of two
Type~A penalties converge to a finite positive constant).
If $\mu_j=0$ for some $a_j>0$, then $f_\lambda(\mu)\geq a_j$ (Type~B / missing support:
Theorem~\ref{thm:game}'s $O(1)$ floor).
Large-$\lambda$ KKT activates all $a_i>0$, giving $c=(2\lambda B+m)/S$; expand each summand as
$a_i/(4\lambda^2\mu_i^2)+O(\lambda^{-3})$.
Main-text Corollary~\ref{cor:rate-separation} records the Type~A / Type~B split.
\end{remark}

\begin{lemma}[Mean of a random Stiefel projector]
\label{lem:C}
If $U\in\mathbb{R}^{d_x\times r}$ is uniform on the Stiefel manifold, then $\E_U[UU^\top]=(r/d_x)\,I$.
For $0<r<d_x$,
$\|UU^\top-(r/d_x)I\|_{\mathrm{op}}=\max\{1-r/d_x,r/d_x\}$;
a single random projector does not concentrate in operator norm.
\end{lemma}
\paragraph{Proof sketch.}
Orthogonal equivariance forces $\E[UU^\top]=cI$; tracing gives $c=r/d_x$.  Eigenvalues of
$UU^\top-(r/d_x)I$ are $1-r/d_x$ ($r$ times) and $-r/d_x$ ($d_x-r$ times). \(\square\)

\begin{remark}[Local smooth task-risk cost]
\label{rem:local-smooth-cost}
\label{cor:E}
Let $R,P$ be $C^3$ near $\theta_0$, with $\nabla R(\theta_0)=0$ and $H=\nabla^2R(\theta_0)\succ0$.
If $g=\nabla P(\theta_0)$ and the local minimiser $\theta_\lambda$ of $R+\lambda P$ converges to
$\theta_0$ as $\lambda\downarrow0$, then
$R(\theta_\lambda)-R(\theta_0)=\tfrac12\lambda^2 g^\top H^{-1}g+O(\lambda^3)$.
Implicit-function theorem: $\theta_\lambda=\theta_0-\lambda H^{-1}g+O(\lambda^2)$; Taylor-expand $R$.
\end{remark}

\begin{remark}[Conditional directional cost]
\label{rem:directional-cost}
\label{cor:Estar}
As $\lambda\downarrow0$, if $\theta_\lambda=\theta_0+\lambda d+o(\lambda)$ and $R$ is Hadamard
directionally differentiable with $R'(\theta_0;d)>0$, then
$R(\theta_\lambda)-R(\theta_0)=\lambda R'(\theta_0;d)+o(\lambda)$.
\end{remark}

\subsection{Near-affine envelope}
\label{app:local-soundness}

\begin{remark}[Near-affine pointer]
The even-spread positive-$\varepsilon_{\mathrm{jac}}$ envelope is
Proposition~\ref{prop:lazy-even-spread}; see also Remark~\ref{rem:near-affine}
(Appendix~\ref{app:conditional-licenses}). No additional expansion is given here.
\end{remark}

%% file: appendix/d_lemmas.tex
\paragraph{Coupling catalogue.}
The lemmas below give candidate recovery rates for subspace, isotropic, and
adversarial-delta estimands; Table~\ref{tab:D3-D6} summarises constructions for the
remaining families.

\paragraph{D1: spiked-subspace recovery.}
\begin{lemma}[Recovery of a spiked nuisance subspace]
\label{lem:D1}
Let $1\leq r<d_x$ and let $Z_1,\ldots,Z_N$ be i.i.d.\ centred vectors with covariance
$C=W\Lambda W^\top+\sigma^2I$, where $W\in\mathbb{R}^{d_x\times r}$ has orthonormal columns and
$\Lambda\succ0$.  Assume the covariance-relative sub-Gaussian bound
$\|\langle Z_i,u\rangle\|_{\psi_2}\leq K(u^\top Cu)^{1/2}$ for every unit $u$.
If $\gamma=\lambda_r(C)-\lambda_{r+1}(C)>0$, $\widehat C$ is the centred sample
covariance, and $\widehat W$ contains its top $r$ eigenvectors, then on
$\|\widehat C-C\|_{\mathrm{op}}\leq\gamma/2$,
\[
\|\Pi_{\widehat W}-\Pi_W\|_{\mathrm{op}}
\leq\frac{2\|\widehat C-C\|_{\mathrm{op}}}{\gamma},
\qquad
\|\Pi_{\widehat W}-\Pi_W\|_F
\leq\frac{2\sqrt{2r}\|\widehat C-C\|_{\mathrm{op}}}{\gamma}.
\]
Writing $r_{\mathrm{eff}}(C):=\Tr(C)/\|C\|_{\mathrm{op}}$,
\[
\|\widehat C-C\|_{\mathrm{op}}
=O_P\!\left(K^2\|C\|_{\mathrm{op}}
\left[\sqrt{\frac{r_{\mathrm{eff}}(C)}{N}}
+\frac{r_{\mathrm{eff}}(C)}{N}\right]\right).
\]
\end{lemma}
The singular-vector implementation must use the left singular vectors of the
$d_x\times N$ centred data matrix.  The limiting projector is $WW^\top$, not $WW^\top/r$.

\newtheorem{Dlem}{Lemma}
\renewcommand{\theDlem}{D\arabic{Dlem}}

\paragraph{D2: isotropic covariance.}
\setcounter{Dlem}{1}
\begin{Dlem}[Isotropic covariance estimation]
\label{lem:D2}
If $\sigma>0$, $d\geq1$, and
$n_i\stackrel{\mathrm{iid}}{\sim}\mathcal N(0,\sigma^2I_d)$, define
$\widehat\sigma^2=(Nd)^{-1}\sum_{i=1}^N\|n_i\|_2^2$.  Then
$Nd\,\widehat\sigma^2/\sigma^2\sim\chi^2_{Nd}$ and
$|\widehat\sigma^2-\sigma^2|=O_P(\sigma^2/\sqrt{Nd})$.
\end{Dlem}

\paragraph{D3--D6: coupling constructions (unnumbered).}
\label{lem:D3}\label{lem:D4}\label{lem:D5}\label{lem:D6}
Families $A_3$--$A_6$ supply \emph{named} candidate second moments under the couplings below.
They are constructions / modelling assumptions for the supporting screens---not load-bearing
rate propositions for estimators that were not run as stated.
Empirical blocks record the proxy actually used (Table~\ref{tab:app-block-index}).

\begin{table}[!htbp]
\centering
\small
\caption{Assumptions and candidate constructions for $A_3$--$A_6$
(not numbered rate claims).}
\label{tab:D3-D6}
\begin{tabular}{@{}clp{0.72\linewidth}@{}}
\toprule
ID & $A_k$ & Coupling / candidate $\Sigma_{\mathrm{task}}$ \\
\midrule
D3 & $A_3$ &
Centred augmentation displacement $\beta$ with $M=\E[\beta\beta^\top]$;
sample (or exact finite-list) second moment.
Rank-$r$ PSD truncation and Davis--Kahan control need an eigengap; weak gaps
(e.g.\ T3A spectrum ratio ${\approx}1.08$) predict unreliable subspace match. \\
D4 & $A_4$ &
Specified label-equivalent pairing $\pi$ with $\Delta=X_T-X_S$ and
$M_\pi=\E_\pi[\Delta\Delta^\top]\neq0$.
Equality of marginal conditionals does not supply $\pi$.
Flagship T4 uses a representation-Gram \emph{proxy}, not this paired-moment estimator. \\
D5 & $A_5$ &
Coordinate support $n=J_nu$ a.s.\ yields $\Sigma_{\mathrm{task}}=J_n\Cov(u)J_n^\top$;
block-restricted sample covariance under finite moments.
``Nuisance'' still requires a separate label-preserving coupling. \\
D6 & $A_6$ &
If deployment perturbations are \emph{assumed} to share the law of centred
within-segment increments $u_t$, then $\Sigma_{\mathrm{task}}=\E[u_tu_t^\top]$.
Label constancy makes increments candidates only; T6 uses residual/scatter proxies. \\
\bottomrule
\end{tabular}
\end{table}

\paragraph{Sketch (D3 truncation; optional).}
If $\|\widehat M-M\|_{\mathrm{op}}=\epsilon_N$, the best rank-$r$ PSD truncation obeys
$\|\mathcal T_r(\widehat M)-M\|_{\mathrm{op}}\leq\lambda_{r+1}(M)+2\epsilon_N$;
an eigengap larger than $2\epsilon_N$ additionally yields Davis--Kahan subspace control.
This is bookkeeping for the construction, not a claim that the T3 arms ran a proved-rate
estimator.

\paragraph{D7: learned perturbation Gram.}
\setcounter{Dlem}{6}
\begin{Dlem}[Estimation of an algorithm-induced moment]
\label{lem:D7}
Let $d_x\geq2$ and let $\delta_i^\star$ be i.i.d.\ draws from the law defining
$M=\E[\delta^\star\delta^{\star\top}]$, with $\|\delta_i^\star\|\leq R$.  Suppose
$\|\widehat\delta_i-\delta_i^\star\|\leq\eta$.  With
$\widehat M=N^{-1}\sum_i\widehat\delta_i\widehat\delta_i^\top$,
\[
\|\widehat M-M\|_{\mathrm{op}}
\leq
\left\|N^{-1}\sum_i\delta_i^\star\delta_i^{\star\top}-M\right\|_{\mathrm{op}}
+2R\eta+\eta^2.
\]
The first term is controlled by a standard matrix-Bernstein bound; its exact rate depends on the
variance proxy as well as $R$, $N$, and $d_x$.
\end{Dlem}
This estimates the perturbation law induced by the algorithm.  It estimates deployment displacement
only if equality of those laws is separately assumed; PGD deltas are not automatically
label-preserving.